\documentclass[12pt]{article}
\usepackage[T1]{fontenc}
\usepackage[margin=1in]{geometry}
\usepackage{amsmath,amssymb,amsthm,mathtools,bm}
\usepackage{booktabs,array,graphicx,xcolor,enumitem}
\usepackage[round,authoryear]{natbib}
\usepackage[hidelinks]{hyperref}

\graphicspath{{figures/}}
\hypersetup{
  pdftitle={Minimax Gaussian Mechanisms for Continual Machine Unlearning},
  pdfauthor={Qi Kuang and Yin Xia},
  pdfsubject={},
  pdfkeywords={adaptive deletion, differential privacy, empirical risk minimization, Newton approximation, uniform consistency}
}

\theoremstyle{plain}
\newtheorem{theorem}{Theorem}[section]
\newtheorem{proposition}[theorem]{Proposition}
\newtheorem{corollary}[theorem]{Corollary}
\newtheorem{lemma}[theorem]{Lemma}
\theoremstyle{definition}
\newtheorem{assumption}[theorem]{Assumption}
\newtheorem{remark}[theorem]{Remark}

\newcommand{\R}{\mathbb R}
\newcommand{\E}{\mathbb E}
\newcommand{\Pp}{\mathbb P}
\newcommand{\cN}{\mathcal N}
\newcommand{\cH}{\mathcal H}
\newcommand{\tr}{\operatorname{tr}}
\newcommand{\GED}{\operatorname{GED}}
\newcommand{\rw}{\mathrm{rw}}
\newcommand{\ind}{\mathrm{ind}}
\newcommand{\polylog}{\operatorname{polylog}}
\newcommand{\norm}[1]{\left\|#1\right\|_2}
\newcommand{\Pred}{\operatorname{Pred}}

\title{Minimax Gaussian Mechanisms for\\Continual Machine Unlearning}
\author{Qi Kuang \quad Yin Xia\\[0.5em]
\small Department of Statistics and Data Science, Fudan University}
\date{}

\begin{document}
\maketitle

\begin{abstract}
Machine unlearning updates a trained model after specified training records are deleted. The goal is to match the result of exact retraining on the remaining data without repeating the full training procedure. This article develops Gaussian mechanisms for efficient Newton updates under sequential deletion requests. Using Gaussian differential privacy (GDP) and its adaptive composition rule, we show that the full sequence of released models is statistically difficult to distinguish from matched exact retraining. To calibrate these mechanisms for empirical risk minimization, we derive upper bounds on the error of the Newton approximation relative to exact retraining and on how this error changes after each deletion batch. Independent Gaussian noise is calibrated using bounds on the full residual at each release, whereas Gaussian random walk noise uses smaller bounds on residual increments. These bounds yield allocations minimizing the worst-case maximum noise variance across releases under the resulting GDP certification constraints. With count-based bounds, the random walk asymptotically matches the worst-case variance of a single release at deletion cap $M$, while independent noise incurs an additional factor of order $M$. Set-based bounds can reduce the noise variances by using gradients and Hessians of the deleted records. For singleton deletion, we further show that count-based independent noise, count-based random walk noise, and set-based independent noise are minimax among fixed Gaussian covariances under their respective residual or increment bounds. With set-based bounds, allowing variances to adapt to deleted records can improve on every fixed covariance by a factor of order $(\log M)^2$ on some data sequences. The residual and noise bounds also yield parameter and predictive consistency relative to exact retraining, uniformly over deletion policies. Simulations and a credit default data analysis evaluate bounds, noise variances, and estimation errors.
\end{abstract}


\section{Introduction}
\label{sec:related}

After a model is deployed, some of its training records may need to be deleted.
For example, individuals may request deletion of personal information, and data
providers may withdraw copyrighted material. Records may also be deleted because they are poisoned, are associated
with harmful model behavior, or are no longer relevant to the prediction task. Removing records from storage does not remove their
influence from a trained model. Machine unlearning seeks to update the
model as if those records had been omitted from training
\citep{cao2015towards,bourtoule2021machine}.

Exact retraining on the remaining data provides the benchmark, but
repeating the optimization after each request can be expensive.
We study continual unlearning, where deletion requests arrive over time
and updated models must remain available throughout the deletion process.
Each request may remove a batch of records, and the system releases an
updated model after each batch.
To reduce the cost of these repeated optimizations, approximate
unlearning updates reuse information from the original fit.

However, an approximate update may be easy to distinguish
from the retrained estimator even when its parameter error is small. \citet{guo2019certified} illustrate this
distinction with a linear regression example: a coefficient is exactly
zero when an observation is absent, whereas an approximate update can
leave it small but nonzero. Its nonzero value can reveal that the observation was used in training
despite a small parameter error. {\it Certified unlearning} therefore requires
the unlearning and retraining output distributions to be difficult to
distinguish. Adding noise to both outputs can help meet this
requirement, at the cost of statistical accuracy.

For continual release, certification must cover the full sequence of
released models, which we call the \emph{transcript}. A guarantee for each
model separately need not give the same guarantee for their joint
distribution. We therefore compare the joint distributions of the
unlearning and exact retraining sequences under the same deletion
policy and noise rule. When requests depend on earlier releases, applying
this policy to each procedure's own history releases can produce different deletion paths.
Gaussian differential privacy (GDP)
quantifies the distinguishability of these sequences and accounts
for composition under such adaptive requests. Maintaining a fixed
guarantee for the full sequence can therefore require more noise than
certifying a single release.

We therefore study how dependence between the noise added at different
releases can reduce the Gaussian noise needed to certify continual Newton updates and when the resulting
mechanisms are optimal. For a cap of $M$ deleted records, we compare
their worst-case maximum noise variances across releases with the variance needed for a
single release after $M$ deletions.

To study this question, we develop two Gaussian noise mechanisms for the same
Newton update. For each cumulative deletion set, this update takes one Newton
step from the original full-data fit. Write $e_t$ for its residual
relative to the exact retrained estimator at release $t$.
Independent noise, denoted IND, is calibrated to $e_t$ at each release.
A Gaussian random walk, denoted RW, retains the preceding noise and adds
a fresh increment. At a common release history, the conditional mean
gap between approximate unlearning and exact retraining is then $e_t-e_{t-1}$.
Subtracting the residuals cancels shared terms in the Newton remainders
and gives smaller bounds. With bounds based on deletion counts, RW
has the same asymptotic worst-case variance as a single release,
whereas IND requires a variance larger by a factor of order $M$.

\subsection{Related work}
\label{sec:related-work}

\leavevmode Computational approaches to unlearning reduce the work required after
deletion by changing how the original model is trained or stored.
\citet{cao2015towards} express learning procedures through additive
summaries that can be updated when records are removed.
\citet{ginart2019making} study deletion algorithms
that avoid retraining from scratch for selected learning problems.
\citet{bourtoule2021machine} partition training data and retain
checkpoints so that deletion requires retraining only affected
components.
Stability under changes to training data
\citep{chakraborty2026stability}
provides another way to make
deletion efficient.
\citet{ullah2021machine} use total variation stability and
couplings to construct exact unlearning procedures for convex risk
minimization. These methods make the original training procedure part
of the design of an efficient deletion algorithm.

Local updates offer another way to reduce computation after deletion.
Methods based on gradients or Newton updates refine existing estimates in both
statistical estimation \citep{han2022optimal,luo2023low}
and unlearning \citep{guo2019certified,pandey2026gaussian}.
For empirical risk minimization, such unlearning updates avoid solving the
retained optimization problem from scratch.
\citet{guo2019certified} develop certified removal using a Newton update
and control of the remaining optimization residual. Their
$(\epsilon,\delta)$ criterion compares unlearning with retraining
without requiring the original learner to be differentially private.
\citet{izzo2021approximate} study efficient approximate deletion for
linear and logistic models. \citet{suriyakumar2022algorithms} use the infinitesimal
jackknife to process sequential deletion requests without recomputing
and inverting a Hessian at each step.
Approximate cross-validation also avoids repeated empirical risk
minimization refits \citep{luo2023iterative}.

{Recent analyses establish accuracy and certification guarantees for unlearning using Newton updates with different numbers of steps and different noise distributions.}
\citet{zou2025certified} analyze repeated Newton
updates with Laplacian perturbations in high dimensions and identify
settings where a single Newton step is insufficient.
\citet{pandey2026gaussian} introduce Gaussian certifiability and
establish guarantees for one Newton step with Gaussian noise in a
high-dimensional regime.

Other analyses address statistical accuracy, generalization, deletion capacity,
and resource tradeoffs. \citet{sekhari2021remember} analyze
unlearning with generalization guarantees and computational and storage
constraints.
\citet{huang2023tight}
give tight bounds on deletion capacity for unlearning based on differential
privacy. \citet{allouah2025utility} study utility, time, and storage
tradeoffs, while \citet{koloskova2025certified} give certified unlearning
guarantees for neural networks through noisy fine-tuning. For squared loss, \citet{xie2026efficient} establish minimax
estimation rates for the model parameter after deletion under restricted
access to the remaining data.

For continual release, guarantees must also account for the number
of deletion requests and how they are chosen.
\citet{neel2021descent} study gradient methods for convex
models under sequences of updates, with cost per deletion
that does not grow with the sequence length.
When deletion requests depend on previously released models,
guarantees must also account for this adaptivity.
\citet{gupta2021adaptive} give reductions from adaptive to nonadaptive
unlearning guarantees, while \citet{ullah2023adaptive} relate unlearning to
adaptive query release. More recently, \citet{qiao2025hessian} develop a Hessian-free method for certified online unlearning that uses stored statistics for each record to perform deletion updates through vector addition.

Continued certification also requires a criterion for comparing distributions
and a rule for composing guarantees across releases. Differential privacy
compares output laws for data sets that differ in one record
\citep{dwork2006our,dwork2006calibrating,mironov2017renyi}, with applications including
private covariance estimation and regression
\citep{li2024federated,cai2026differentially}.
The $f$-differential privacy framework of \citet{dong2022gaussian} describes
distinguishability through the full tradeoff between testing errors.
Its Gaussian form, GDP, gives an explicit composition rule.
\citet{smith2022fully} extend Gaussian composition to parameters chosen
from preceding outputs, allowing per-release budgets to adapt
to the release history.

We use this testing framework to compare the full sequences released
by one-step Newton unlearning and exact retraining under the same deletion
policy and noise rule. Working in fixed dimension, we bound the residuals
of updates from the original fit and their increments uniformly over deletion
sets below the cap. These bounds support noise calibration for arbitrary
deletion batches, including those chosen using earlier releases.

\subsection{Contributions}
\label{sec:contributions}

Within this framework, we study noise allocation,
covariance optimality, and statistical consistency for Gaussian perturbations
of Newton updates. The main contributions are as follows.

\begin{enumerate}[label=(\arabic*)]
\item We construct count-based and set-based bounds on the Newton
approximation residual and its increments, valid simultaneously over
all deletion sets below the cap with high probability. Conditional on
the training data, count-based bounds depend on deletions only through
their counts, while set-based bounds also use the deleted records'
gradients and Hessians. These bounds yield GDP certificates for
the full sequences of models released by IND and RW.

\item We derive variance allocations that minimize the worst-case maximum
noise variance across releases under the respective certification
constraints. RW's worst-case variance has the same order as
the corresponding single-release benchmark, whereas IND incurs an
additional factor of order $M$. A computable RW
allocation avoids combinatorial optimization over deletion sets, preserves the
variance order of set-based RW, and requires no more noise than
count-based RW along the same deletion path.

\item For singleton deletion, we characterize Gaussian covariances that
certify the residual classes defined by our bounds. Among mechanisms
whose covariance is fixed before release, count-based and set-based IND
are minimax under their respective residual bounds. Count-based RW is minimax using only
the increment bounds and asymptotically minimax using both bounds.
Set-based RW adapts its increment variances to the deleted sets
and may lie outside the fixed covariance class. Under increment bounds, alone or together with
residual bounds, its worst-case variance is at most a constant multiple of
the fixed covariance optimum and can be smaller by a factor of order
$(\log M)^2$.

\item We establish parameter, prediction, and loss consistency relative to
exact retraining, uniformly over deletion policies. The weaker
sufficient conditions for RW allow consistency to be established for larger
deletion caps than those covered by the IND conditions. Simulations and
a credit default data analysis show lower errors for RW than for
IND under the same certification guarantee.
\end{enumerate}

\subsection{Organization}
\label{sec:organization}

{\color{black}Section~\ref{sec:mechanisms}} introduces the certification criterion and
mechanisms. {\color{black}Section~\ref{sec:certification}} develops residual bounds and
variance allocations. {\color{black}Section~\ref{sec:variance}} studies variance
comparisons, covariance optimality, and consistency.
{\color{black}Section~\ref{sec:experiments}} presents numerical experiments, and
{\color{black}Section~\ref{sec:discussion}} discusses limitations and further questions.
The supplementary material contains the detailed bounds and proofs
and further experimental and computational details.

\section{Continual unlearning}
\label{sec:setup}
\label{sec:mechanisms}

\subsection{Problem setup}

We consider deletion requests that arrive sequentially, possibly in batches,
and release an updated model after each batch. Let
$D=\{(x_i,y_i)\}_{i=1}^n$, where $x_i\in\R^p$. Let $\ell$ be a loss function,
and write $\ell_i(\beta)=\ell(y_i,x_i^\top\beta)$ for observation $i$,
where $x_i^\top\beta$ is the linear predictor. Let $\mathcal P$ be a regularizer and
$\lambda>0$ its regularization parameter. For a set $A\subseteq[n]$ of
deleted indices, define the objective and estimator on the remaining data by
\[
L_A(\beta)
=
\lambda \mathcal P(\beta)+\sum_{i\notin A}\ell_i(\beta),
\qquad
\widehat\beta_A=\arg\min_{\beta\in\R^p}L_A(\beta),
\qquad
\widehat\beta=\widehat\beta_\varnothing.
\]
The original full-data estimator is $\widehat\beta$.
\emph{Exact retraining} recomputes
$\widehat\beta_A$ on the remaining data, whereas \emph{approximate unlearning}
uses the one-step Newton update defined in
{\color{black}Section~\ref{sec:gaussian-releases}}. When comparing these
procedures, \emph{unlearning} refers to the approximate update.

Let $S_t$ be the cumulative set of indices deleted at the $t$th unlearning
release, and let $T$ be the total number of releases under a deletion
cap $1\le M\le n$, which may grow with $n$. Then
$\varnothing=S_0\subset S_1\subset\cdots\subset S_T$, $m_t\coloneqq|S_t|$,
with $m_0=0$ and $m_T\le M$. After release $t-1$, a deletion
policy $\pi$ chooses the next deletion batch based on previous deletions
and model releases.
If the next batch would make the cumulative deletion count exceed $M$, the
algorithm stops issuing unlearning releases and retrains exactly on the
remaining data. Thus $T\le M$, but $T$ need not be known in advance.

Our goal is to release accurate models using efficient Newton updates
while controlling Gaussian noise variance under a fixed certification guarantee
for the full sequence. The following subsections define this guarantee relative
to exact retraining, the accuracy measures, and the release mechanisms.

\subsection{Transcript certification}

Exact retraining on the remaining data gives $\widehat\beta_{S_t}$,
the model obtained by training without the records in $S_t$.
We use exact retraining as the reference for certification and require
the joint distributions of the unlearning and retraining releases to be
statistically difficult to distinguish.

To quantify this indistinguishability, let $P$ and $Q$ denote the unlearning
and retraining output laws, respectively. The
$(\epsilon,\delta)$-certified removal criterion of \citet{guo2019certified}
requires
\[
P(E)\le e^\epsilon Q(E)+\delta,
\qquad
Q(E)\le e^\epsilon P(E)+\delta
\]
for every measurable event $E$. These inequalities limit how much the
probability of an observable event can differ between the two procedures.
The same inequalities define $(\epsilon,\delta)$-differential
privacy (DP) when required for output distributions on every pair of
neighboring data sets \citep{dwork2006our,dwork2006calibrating}.
By contrast, certified removal uses them to compare
unlearning with exact retraining on the remaining data. This borrows a criterion
for indistinguishability without requiring the original training algorithm
to satisfy DP.

The $f$-DP framework generalizes $(\epsilon,\delta)$-DP by describing
indistinguishability through the full tradeoff between testing errors
\citep{dong2022gaussian}. Within this framework, GDP gives a single-parameter criterion with a direct statistical
interpretation and a sharp composition rule. \citet{pandey2026gaussian} use this
criterion to define Gaussian certifiability for unlearning. We adopt it to
measure indistinguishability from exact retraining for the full
sequence of released models.

To express this criterion in terms of testing errors, consider $H_0$,
under which an observation is drawn from $P$, versus $H_1$, under which
it is drawn from $Q$. A measurable test $0\le\phi\le1$ gives the probability
of rejecting $H_0$ given the observation. Its type I and type II errors
are $\E_P\phi$ and $\E_Q(1-\phi)$, respectively. The tradeoff function is
\[
\mathcal T(P,Q)(\alpha)
=
\inf_{0\le\phi\le1:\,\E_P\phi\le\alpha}\E_Q(1-\phi),
\qquad 0\le\alpha\le1.
\]
It gives the smallest type II error when the type I error is at most
$\alpha$. A higher curve means that the two laws are harder to distinguish.

To turn this curve into a certification criterion, we use a prescribed
tradeoff function $f$ as a benchmark and require
$\mathcal T(P,Q)(\alpha)\ge f(\alpha)$, $0\le\alpha\le1$.
Thus $\mathcal T(P,Q)$ describes the actual testing difficulty, while $f$
specifies its required lower bound. This is the testing criterion used
in $f$-DP, applied here to unlearning and exact retraining.

Gaussian certifiability uses the GDP benchmark $f=G_\mu$, where
\[
G_\mu(\alpha)
=
\mathcal T\!\left(N(0,1),N(\mu,1)\right)(\alpha)
=
\Phi\{\Phi^{-1}(1-\alpha)-\mu\},
\]
and $\Phi$ is the standard normal distribution function. This is the
tradeoff for deciding, from one observation, whether it comes from
$N(0,1)$ or $N(\mu,1)$. Requiring $\mathcal T(P,Q)\ge G_\mu$ makes $P$
and $Q$ at least as hard to distinguish as this Gaussian pair. A smaller
$\mu$ requires $P$ and $Q$ to be harder to distinguish.

We apply this Gaussian benchmark to the full sequences $Y^{\mathrm{un}}$
and $Y^{\mathrm{re}}$ of models released by unlearning and matched exact
retraining, respectively. In the reference procedure, exact retrained
estimators replace the Newton approximations, while the deletion policy and
noise rule remain the same. Each procedure applies the policy to its own
history, so their deletion paths may differ. The test uses the entire sequence of released models to
distinguish unlearning from matched exact retraining. Conditional on the training data $D$, we require the
Gaussian bound in both directions:
\begin{equation}
\min\!\left\{
\mathcal T\!\left(\mathcal L(Y^{\mathrm{re}}\mid D),\mathcal L(Y^{\mathrm{un}}\mid D)\right),
\mathcal T\!\left(\mathcal L(Y^{\mathrm{un}}\mid D),\mathcal L(Y^{\mathrm{re}}\mid D)\right)
\right\}\ge G_\mu.
\label{eq:target-one}
\end{equation}
With high probability over the training data, this guarantee must hold simultaneously for all allowed deletion policies. We extend the Gaussian certifiability criterion of \citet{pandey2026gaussian} to the full sequence of released models. An unlearning method is $(\delta_n,\mu)$-Gaussian certifiable for this sequence if
 some set
$\cH_n$ of training data sets satisfies
$\Pp\{D\in\cH_n\}\ge1-\delta_n$,
and \eqref{eq:target-one} holds for every $D\in\cH_n$ and every allowed
deletion policy $\pi$.

This guarantee for the full sequence can be established by composing
conditional guarantees at successive releases. At each common history
of deletion sets and model releases, suppose the conditional output laws of unlearning and retraining
satisfy the GDP criterion with parameter $\mu_t$ in both directions.
Adaptive composition then gives \eqref{eq:target-one} whenever
$\sum_t\mu_t^2\le\mu^2$ along every allowed history, even when the $\mu_t$
are chosen based on preceding releases
\citep{dong2022gaussian,smith2022fully}.

\subsection{Utility}

\leavevmode The certification guarantee above compares the full sequences of unlearning
and exact retraining releases under the same noise rule. Even when this
guarantee holds, approximation error and added noise may reduce model
accuracy. We therefore compare each noisy released model
with the unperturbed exact retrained estimator on the same retained data.

Let $Y_t$ denote the parameter vector released at step $t$, and let
$\widehat\beta_{S_t}$ be the exact retrained estimator on the same retained
data. For each release $t\le T$, set
\[
\mathcal F_t = \sigma(D,S_1,\ldots,S_t,Y_1,\ldots,Y_{t-1}).
\]
Conditioning on $\mathcal F_t$ fixes the training data, the deletion sets
through step $t$, and all preceding releases. The fresh noise at step $t$
remains random. Let $(x_0,y_0)$ be an independent test observation from the
same population as the training data. The conditional expectations below
average over both this test observation and the fresh noise.

The prediction error measures the expected squared difference between
the two models' linear predictors on a new test observation:
\[
\operatorname{Pred}_t
=
\E\!\left[
\{x_0^\top(Y_t-\widehat\beta_{S_t})\}^2
\,\middle|\,
\mathcal F_t
\right].
\]
Following \citet{pandey2026gaussian}, define the generalization error
divergence (GED) as the conditional expected absolute difference between
their losses on the same test observation:
\[
\GED_t
=
\E\!\left[
\left|
\ell(y_0,x_0^\top Y_t)
-
\ell(y_0,x_0^\top\widehat\beta_{S_t})
\right|
\,\middle|\,
\mathcal F_t
\right].
\]

In addition to prediction and loss errors, we measure parameter error by
$\|Y_t-\widehat\beta_{S_t}\|_2$, using the realized release $Y_t$.

Consistency requires these errors to vanish as the training sample size
$n$ grows. Write $\Pp_\pi$ for probability under deletion policy $\pi$,
including randomness in the training data and the release process.
Prediction consistency means
\[
\sup_\pi\Pp_\pi\!\left(
\max_{t\le T}\operatorname{Pred}_t>\epsilon
\right)\longrightarrow0
\qquad
\text{for every }\epsilon>0.
\]
GED consistency means
\[
\sup_\pi\Pp_\pi\!\left(
\max_{t\le T}\GED_t>\epsilon
\right)\longrightarrow0
\qquad
\text{for every }\epsilon>0.
\]
Parameter consistency means
\[
\sup_\pi
\Pp_\pi\!\left(
\max_{t\le T}
\|Y_t-\widehat\beta_{S_t}\|_2>\epsilon
\right)
\longrightarrow0
\qquad
\text{for every }\epsilon>0.
\]

The maximum over $t$ requires accuracy throughout the deletion process.
The supremum over $\pi$ makes each guarantee uniform over all allowed
deletion policies, including those that choose requests using previously
released models.

\subsection{Gaussian release mechanisms}
\label{sec:gaussian-releases}

To construct the release mechanisms, we first approximate the retrained
estimator by one Newton step from the original fit
$\widehat\beta$. For a cumulative deletion set $A$, define
\begin{equation}
\widetilde\beta_A
=
\widehat\beta
+
\{\nabla^2L_A(\widehat\beta)\}^{-1}
\sum_{i\in A}\nabla\ell_i(\widehat\beta),
\qquad
e_A=\widetilde\beta_A-\widehat\beta_A.
\label{eq:one-step}
\end{equation}
The residual $e_A$ measures the parameter error relative to exact
retraining. Along a deletion sequence, write $e_t\coloneqq e_{S_t}$,
with $e_0=0$.

Starting each update from $\widehat\beta$ allows us to reuse derivatives
evaluated at the original fit. A step from the preceding internal approximation
may be more accurate, but requires recomputing derivatives on the retained data.
For $M$ singleton deletions, this increases Hessian evaluation and summation
costs by a factor of order $M$
({\color{black}Supplementary Section~\ref{app:computational-comparison}}).

Additional Newton steps offer another way to improve the approximation,
at greater computational cost. {\color{black}Supplementary Section~\ref{app:multi-newton}} shows that their
residuals can be made arbitrarily small.
The same certification framework applies to these updates once suitable residual
bounds and, for RW, increment bounds are available. We focus on one Newton step
per release to limit computation, as it already suffices for the certification
and consistency results established below.

We next add Gaussian noise to this approximation in two ways: independently
at each release, or by accumulating Gaussian noise increments. Write $r_t$
for a bound on $\|e_t\|_2$ and $d_t$ for a bound on
$\|e_t-e_{t-1}\|_2$, with $r$ standing for residual and $d$ for difference.
IND uses $r_t$ to choose its noise variance, whereas RW uses $d_t$ to
choose its noise increment variance.
The conditional comparisons below fix the same deletion sets and
preceding releases in both procedures.

Specifically, IND releases
\[
Y_t^{\ind}
=
\widetilde\beta_{S_t}+\varepsilon_t,
\qquad
\varepsilon_t\mid\mathcal F_t
\sim N(0,\tau_t^2I_p).
\]
The variance $\tau_t^2$ is measurable with respect to $\mathcal F_t$.
Conditional on $\mathcal F_t$, the approximate and matched exact releases
have Gaussian distributions whose tradeoff function is
\[
\mathcal T\!\left(
N(\widetilde\beta_{S_t},\tau_t^2I_p),
N(\widehat\beta_{S_t},\tau_t^2I_p)
\right)
=
G_{\|e_t\|_2/\tau_t}.
\]
The release is $\mu_t$-GDP if and only if
$\|e_t\|_2/\tau_t\le\mu_t$. We therefore seek an
$\mathcal F_t$-measurable bound $\|e_t\|_2\le r_t$.
Once such a bound is available, setting $\tau_t=r_t/\mu_t$
ensures that the release is $\mu_t$-GDP. 

RW, by contrast, retains the preceding noise and releases
\[
Y_t^{\rw}
=
\widetilde\beta_{S_t}+Z_t,
\qquad
Z_t=Z_{t-1}+\xi_t,
\qquad
\xi_t\mid\mathcal F_t
\sim N(0,\sigma_t^2I_p),
\qquad
Z_0=0.
\]
The variance $\sigma_t^2$ is measurable with respect to $\mathcal F_t$.
Equivalently,
$Y_t^{\rw} = Y_{t-1}^{\rw} +\widetilde\beta_{S_t}-\widetilde\beta_{S_{t-1}} +\xi_t$, $Y_0^{\rw}=\widehat\beta$.
Write $y_{t-1}$ for the preceding release value, with
$y_0=\widehat\beta$. Conditional on $\mathcal F_t$, the approximate and
matched exact release distributions are, respectively,
\[
N\!\left(
y_{t-1}+\widetilde\beta_{S_t}-\widetilde\beta_{S_{t-1}},
\sigma_t^2I_p
\right),
\qquad
N\!\left(
y_{t-1}+\widehat\beta_{S_t}-\widehat\beta_{S_{t-1}},
\sigma_t^2I_p
\right).
\]
The common term $y_{t-1}$ cancels, so their mean difference is
$e_t-e_{t-1}$ and their tradeoff function is
$G_{\|e_t-e_{t-1}\|_2/\sigma_t}$.
The release is $\mu_t$-GDP if and only if
$\|e_t-e_{t-1}\|_2/\sigma_t\le\mu_t$. We therefore seek an
$\mathcal F_t$-measurable bound $\|e_t-e_{t-1}\|_2\le d_t$.
Setting $\sigma_t=d_t/\mu_t$ then ensures that the release is
$\mu_t$-GDP. 

These conditional guarantees yield certification for the full sequence
whenever $\sum_{t=1}^T\mu_t^2\le\mu^2$ along every allowed history.
In terms of the noise variances, sufficient conditions are
\begin{equation}
\sum_{t=1}^T\frac{r_t^2}{\tau_t^2}\le\mu^2,
\qquad
\sum_{t=1}^T\frac{d_t^2}{\sigma_t^2}\le\mu^2.
\label{eq:certificate-budgets}
\end{equation}
A zero bound can be assigned zero variance and contributes zero to
the corresponding sum.
These inequalities give two steps for choosing the noise: construct valid bounds $r_t$ and $d_t$, then
choose variances that satisfy \eqref{eq:certificate-budgets}. Tighter bounds
relax these constraints and can therefore reduce the required noise.

\section{Residual bounds and transcript certification}
\label{sec:certification}
\leavevmode The constraints in \eqref{eq:certificate-budgets} require bounds on
$\|e_t\|_2$ for IND and $\|e_t-e_{t-1}\|_2$ for RW. We first construct two pairs of such
bounds. Count-based bounds use only the number of deleted records,
whereas set-based bounds also use information about the deleted records.
For each mechanism and each pair of bounds, we then choose a variance
allocation that minimizes the worst-case maximum noise variance across
releases subject to these constraints. Together, the bounds and
allocations certify the full sequence of released models.

\subsection{Construction of residual bounds}

\leavevmode To allow requests to depend on earlier releases, the bounds must hold
simultaneously for every deletion set and transition allowed by the cap
$M$. We establish these bounds under the assumptions below. The
convexity, smoothness, polynomial growth, and tail conditions parallel
those used by \citet{pandey2026gaussian}. We work with fixed $p$ rather
than their proportional high-dimensional scaling. Throughout,
$\polylog(n)$ denotes a polynomial in $\log n$.

\begin{assumption}
\label{ass:loss}
For every $y$, the nonnegative loss $\ell(y,\cdot)$ is proper,
convex, and three times differentiable. The regularizer $\mathcal P$ is
proper, strongly convex, and twice differentiable.
There exist constants $C,s>0$ such that, for all $y,z\in\R$,
\[
\max\{\ell(y,z),|\partial_z\ell(y,z)|,
|\partial_z^3\ell(y,z)|\}
\le C(1+|y|^s+|z|^s).
\]
The regularization multiplier satisfies $\lambda=\Theta(1)$.
\end{assumption}

Strong convexity of $\mathcal P$ gives unique retained estimators and
invertible retained Hessians. The polynomial growth condition allows unbounded losses and
derivatives, as in squared loss, but excludes exponential growth.

\begin{assumption}
\label{ass:sampling}
The observations $(x_i,y_i)_{i=1}^n$ are independent and
identically distributed. Let $(x,y)$ have their common distribution.
The feature vector $x$ has mean zero and covariance $\Gamma$.
There exist constants $C_X,K_{\mathrm{sg}}<\infty$ independent of $n$
such that
\[
\|\Gamma\|_2\le C_X,
\qquad
\|a^\top x\|_{\psi_2}
\le K_{\mathrm{sg}}(a^\top\Gamma a)^{1/2}
\quad\text{for every }a\in\R^p.
\]
Here
$\|Z\|_{\psi_2}:=\inf\{t>0:\E\exp(Z^2/t^2)\le2\}$ is the sub-Gaussian
Orlicz norm. For a deterministic threshold
$B_y(n)=O(\polylog(n))$, the response satisfies
$n\Pp\{|y|>B_y(n)\}\to0$ and $\sup_n\E|y|^{2s}<\infty$.
The exponent $s$ is from Assumption~\ref{ass:loss}.
\end{assumption}

The sub-Gaussian feature condition excludes heavy-tailed
feature distributions, while the response conditions allow, for example,
Gaussian and exponential tails. Together with the growth condition,
these assumptions provide uniform derivative bounds that control the
Newton residuals and their increments.

\begin{assumption}
\label{ass:identification}
 The population risk
$\mathcal R(\beta)=\E\{\ell(y,x^\top\beta)\}$ has a minimizer $\beta_0$.
There is a fixed compact convex neighborhood $\cN\subset\R^p$ of $\beta_0$
such that, for some $c_0>0$,
\[
\inf_{\beta\in\cN}
\lambda_{\min}
\E\!\left[
\partial_z^2\ell(y,x^\top\beta)xx^\top
\right]
\ge c_0.
\]
On the same neighborhood, the Hessian of $\mathcal P$ satisfies
\[
\|\nabla^2\mathcal P(\beta)-\nabla^2\mathcal P(\beta')\|_2
\le C_{\mathrm{reg}}(n)\|\beta-\beta'\|_2,
\qquad \beta,\beta'\in\cN,
\]
where $C_{\mathrm{reg}}(n)=O(\polylog(n))$.
\end{assumption}

The population curvature condition prevents the population loss from becoming
nearly flat near $\beta_0$, which is needed for stable Newton updates after
deletions. The condition on $\nabla^2\mathcal P$ limits the regularizer's
contribution to the Newton approximation error. For ridge regularization,
$C_{\mathrm{reg}}(n)=0$.

Under these assumptions, we obtain two types of bounds on the Newton
residual and its increments. The count-based bound $r^{\mathrm{cnt}}(m)$
covers every deletion set of size $m$, and $d^{\mathrm{cnt}}(k,m)$ covers
every transition $A\subsetneq B$ with $|A|=k$ and $|B|=m$. The set-based
bounds $r^{\mathrm{set}}(S)$ and $d^{\mathrm{set}}(A,B)$ use information
about the particular deletion set or transition. The superscripts
\emph{cnt} and \emph{set} distinguish these two choices for both the
residual bound $r_t$ and the increment bound $d_t$. The following theorem
establishes their validity on a common event whose probability tends to
one.

\begin{theorem}
\label{thm:bounds}
Under Assumptions~\ref{ass:loss}--\ref{ass:identification}, suppose
$M=o\{n/\polylog(n)\}$. There exists an event $\cH_n$ with
$\Pp(\cH_n)\ge1-\delta_n$ and $\delta_n=o(1)$ on which the following bounds
hold simultaneously for every $S\subseteq[n]$ with $|S|\le M$ and every
$A\subsetneq B\subseteq[n]$ with $|B|\le M$:
\begin{equation}
\begin{aligned}
\|e_S\|_2&\le r^{\mathrm{set}}(S)\le r^{\mathrm{cnt}}(|S|),&
\|e_B-e_A\|_2&\le d^{\mathrm{set}}(A,B)\le d^{\mathrm{cnt}}(|A|,|B|).
\end{aligned}
\label{eq:cert}
\end{equation}
Here $r^{\mathrm{cnt}},d^{\mathrm{cnt}},r^{\mathrm{set}}(S)$, and $d^{\mathrm{set}}(A,B)$ are defined in
Supplementary Section~\ref{app:residual-increment-bounds}.
The functions $r^{\mathrm{cnt}}$ and $d^{\mathrm{cnt}}$ satisfy
\[
r^{\mathrm{cnt}}(m)=O\!\left(\frac{m^2\polylog(n)}{n^2}\right),\qquad
d^{\mathrm{cnt}}(k,m)=O\!\left(\frac{(m-k)(k+m)\polylog(n)}{n^2}\right).
\]
\end{theorem}
Because the bounds hold simultaneously over deletion sets and
transitions, they remain valid when requests are chosen adaptively from
earlier releases. The inequalities in \eqref{eq:cert} also show that the
set-based bounds are no larger than the count-based bounds. Using
information about the deleted records can therefore reduce the noise
required for the same certification budget.

For a single-record deletion taking the cumulative count from $m-1$
to $m$, the increment bound is at most of order
$m\polylog(n)/n^2$, whereas the residual bound is at most of order
$m^2\polylog(n)/n^2$. Since RW is calibrated to increments rather than residuals,
this difference allows it to use less noise than IND.

To apply these bounds in \eqref{eq:certificate-budgets}, the count-based choice is
$r_t=r^{\mathrm{cnt}}(m_t)$ and $d_t=d^{\mathrm{cnt}}(m_{t-1},m_t)$, whereas the set-based choice is $r_t=r^{\mathrm{set}}(S_t)$ and $d_t=d^{\mathrm{set}}(S_{t-1},S_t)$.

\subsection{GDP budget allocation}
\leavevmode Using the bounds \(r_t\) and \(d_t\) above, we now allocate the GDP
budget across releases. Setting \(\tau_t=r_t/\mu_t\) for IND and
\(\sigma_t=d_t/\mu_t\) for RW ensures certification whenever
\(\sum_{t=1}^T\mu_t^2\le\mu^2\). To choose the budgets
$\mu_t$ for each release, we account for how noise accumulates in each mechanism.

IND and RW release $\widetilde\beta_{S_t}+\varepsilon_t$ and
$\widetilde\beta_{S_t}+Z_t$, respectively, where
$Z_t=\sum_{s=1}^t\xi_s$.
Here $\tau_t^2$ is the variance per coordinate of the fresh IND noise,
whereas $\sigma_t^2$ is that of the fresh RW increment $\xi_t$.
Conditional on $D$, their noise covariances satisfy
$\operatorname{Cov}(\varepsilon_t\mid D) =\E[\tau_t^2\mid D]I_p$, $\operatorname{Cov}(Z_t\mid D) =\E\!\left[\sum_{s=1}^t\sigma_s^2\,\middle|\,D\right]I_p$.
We therefore control the allocated variance for IND and the accumulated
increment variance for RW along every allowed history. Their respective maxima
across releases are $\max_{t\le T}\tau_t^2$ for IND and
$\max_{t\le T}\sum_{s=1}^t\sigma_s^2=\sum_{s=1}^T\sigma_s^2$ for RW.

Let $\mathcal C$ be the class of allowed deletion paths
$P=(S_0,\ldots,S_T)$ with $S_0=\varnothing$,
$S_{t-1}\subsetneq S_t$, and $|S_T|\le M$. For each mechanism, we minimize
the maximum noise variance across releases in the worst case over
$\mathcal C$. The respective objectives are
\begin{equation}
\mathrm{IND}:\quad
\inf_{\{\tau_t^2\in\mathcal F_t\}}\sup_{P\in\mathcal C}
\max_{1\le t\le T}\tau_t^2,
\qquad
\mathrm{RW}:\quad
\inf_{\{\sigma_t^2\in\mathcal F_t\}}\sup_{P\in\mathcal C}
\max_{1\le t\le T}\sum_{s=1}^t\sigma_s^2,
\label{eq:minimax-objectives}
\end{equation}
subject to the corresponding certification constraint in
\eqref{eq:certificate-budgets} on every path.
For allocations that depend on earlier releases, we take the worst
case over those releases as well.
Fix either choice $a\in\{\mathrm{cnt},\mathrm{set}\}$, and let
$r_t,d_t$ denote the corresponding residual and increment bounds.

Any allocation satisfying \eqref{eq:certificate-budgets} obeys, on every
allowed deletion path,
\[
\begin{aligned}
\sum_{t=1}^T r_t^2
&\le \left(\max_{t\le T}\tau_t^2\right)
\sum_{t=1}^T\frac{r_t^2}{\tau_t^2}
\le \mu^2\max_{t\le T}\tau_t^2,\\
\left(\sum_{t=1}^T d_t\right)^2
&\le \left(\sum_{t=1}^T\frac{d_t^2}{\sigma_t^2}\right)
\left(\sum_{t=1}^T\sigma_t^2\right)
\le \mu^2\sum_{t=1}^T\sigma_t^2.
\end{aligned}
\]

The first line uses the largest IND release variance, and the second
follows from Cauchy--Schwarz. Hence the two variance objectives have lower
bounds involving $\sum_t r_t^2$ and $(\sum_t d_t)^2$, respectively.
Taking the maximum over allowed deletion paths leads to
\[
R^a:=\max_{P\in\mathcal C}\sum_{t=1}^T r_t^2,
\qquad
D^a:=\max_{P\in\mathcal C}\left(\sum_{t=1}^T d_t\right)^2.
\]
These quantities normalize the budgets below. The following
proposition shows that the resulting allocations attain the corresponding
lower bounds and therefore solve the two optimization problems.

\begin{proposition}
\label{prop:allocation-minimax}
For either $a\in\{\mathrm{cnt},\mathrm{set}\}$, the budget allocations
\[
\mathrm{IND}:\quad \mu_t^2=\mu^2\frac{r_t^2}{R^a},
\qquad
\mathrm{RW}:\quad \mu_t^2=\mu^2\frac{d_t}{\sqrt{D^a}}
\]
satisfy $\sum_{t=1}^T\mu_t^2\le\mu^2$ on every allowed deletion path.
They attain the respective minimax values $R^a/\mu^2$ and $D^a/\mu^2$
for the objectives in \eqref{eq:minimax-objectives} under the corresponding certification constraints
in \eqref{eq:certificate-budgets}.
If a normalizing constant is zero, the corresponding bounds vanish
on every path, and that mechanism uses zero budgets and zero noise variances.
\end{proposition}

These budgets yield the variance choices $\tau_t^2=R^a/\mu^2$ and
$\sigma_t^2=d_t\sqrt{D^a}/\mu^2$, respectively. The resulting variances for the two choices of bounds are summarized below.
\begin{equation}
\begin{array}{c|c@{\qquad}c}
\text{Bounds used}&\text{IND variance}&\text{RW noise increment variance}\\ \hline
\text{Count-based}&(\tau_t^{\mathrm{cnt}})^2=R^{\mathrm{cnt}}/\mu^2
 &(\sigma_t^{\mathrm{cnt}})^2=d^{\mathrm{cnt}}(m_{t-1},m_t)\sqrt{D^{\mathrm{cnt}}}/\mu^2\\[2pt]
\text{Set-based}&(\tau_t^{\mathrm{set}})^2=R^{\mathrm{set}}/\mu^2
 &(\sigma_t^{\mathrm{set}})^2=d^{\mathrm{set}}(S_{t-1},S_t)\sqrt{D^{\mathrm{set}}}/\mu^2.
\end{array}
\label{eq:variance-allocation}
\end{equation}
Combining the simultaneous bounds in Theorem~\ref{thm:bounds} with the
GDP budget constraint satisfied on every allowed path gives the following
guarantee for the full sequence of released models.

\begin{theorem}
\label{thm:transcript}
Under Assumptions~\ref{ass:loss}--\ref{ass:identification}, suppose
$M=o\{n/\polylog(n)\}$. Both IND and RW with the corresponding variance
allocations in \eqref{eq:variance-allocation} are $(\delta_n,\mu)$-Gaussian
certifiable for the transcript, where $\delta_n=o(1)$ is the
failure probability bound in Theorem~\ref{thm:bounds}.
\end{theorem}

Implementing these allocations requires computing or bounding the path
maxima. For count-based bounds, splitting batches into single-record requests
adds nonnegative terms to the sum of squared residual bounds. It also cannot
decrease the sum of increment bounds, because
$d^{\mathrm{cnt}}(k,m)\le\sum_{j=k+1}^m d^{\mathrm{cnt}}(j-1,j)$
(Supplementary Section~\ref{app:count-path-maximum}). Hence both maxima are attained by the path
$m_t=t$, $T=M$, giving
\[
R^{\mathrm{cnt}}=\sum_{m=1}^M r^{\mathrm{cnt}}(m)^2,
\qquad
D^{\mathrm{cnt}}=\left(\sum_{m=1}^M d^{\mathrm{cnt}}(m-1,m)\right)^2.
\]

For set-based bounds, computing $R^{\mathrm{set}}$ and $D^{\mathrm{set}}$
requires combinatorial optimization, which can be costly in practice.
A natural alternative is to use the upper bounds
$R^{\mathrm{set}}\le R^{\mathrm{cnt}}$ and
$D^{\mathrm{set}}\le D^{\mathrm{cnt}}$ from \eqref{eq:cert}.
Using $R^{\mathrm{cnt}}$ for IND recovers its count-based
allocation. For RW, replacing $D^{\mathrm{set}}$ with $D^{\mathrm{cnt}}$
still makes each increment variance proportional to
$d^{\mathrm{set}}(S_{t-1},S_t)$, so the allocation remains set-dependent.
The following corollary gives a computable
bound $D^{\mathrm{com}}$ between $D^{\mathrm{set}}$ and $D^{\mathrm{cnt}}$,
which yields a practical allocation for RW.

\begin{corollary}
\label{cor:computable-calibration}
On $\cH_n$, the computable bound $D^{\mathrm{com}}$ defined in
Supplementary Section~\ref{app:computable-caps} satisfies
\[
D^{\mathrm{set}}\le D^{\mathrm{com}}\le D^{\mathrm{cnt}}.
\]
The variance allocation
\[
(\sigma_t^{\mathrm{com}})^2
=\frac{d^{\mathrm{set}}(S_{t-1},S_t)\sqrt{D^{\mathrm{com}}}}{\mu^2}
\]
makes RW $(\delta_n,\mu)$-Gaussian certifiable for the transcript.
For the same deletions, each noise increment variance is no larger than
with count-based bounds.
\end{corollary}

\section{Efficiency and utility}
\label{sec:variance}
The allocations in Proposition~\ref{prop:allocation-minimax} minimize the
worst-case maximum variance for IND and RW under their respective
certification constraints. We first compare the resulting noise variances
and then study optimality among Gaussian mechanisms with fixed covariance.
Finally, we establish consistency relative to exact retraining.

\subsection{Variance efficiency}
For the allocations in {\color{black}Section~\ref{sec:certification}}, write the
worst-case maximum noise variances across releases as
\[
V_{\ind}:=\sup_{P\in\mathcal C}\max_{1\le t\le T}\tau_t^2,
\qquad
V_{\rw}:=\sup_{P\in\mathcal C}\max_{1\le t\le T}
\sum_{s=1}^t\sigma_s^2
=\sup_{P\in\mathcal C}\sum_{t=1}^T\sigma_t^2.
\]
Substituting the count-based and set-based allocations gives, respectively,
$V_{\rw}^{\mathrm{cnt}}=D^{\mathrm{cnt}}/\mu^2$,
$V_{\ind}^{\mathrm{cnt}}=R^{\mathrm{cnt}}/\mu^2$,
$V_{\rw}^{\mathrm{set}}=D^{\mathrm{set}}/\mu^2$, and
$V_{\ind}^{\mathrm{set}}=R^{\mathrm{set}}/\mu^2$.
The computable RW allocation gives
$V_{\rw}^{\mathrm{com}}=\sqrt{D^{\mathrm{com}}D^{\mathrm{set}}}/\mu^2$.
\begin{theorem}
\label{thm:variance}
Under Assumptions~\ref{ass:loss}--\ref{ass:identification}, suppose
$M=o\{n/\polylog(n)\}$.
The following ratios hold whenever their denominators are positive.
\[
\frac{V_{\rw}^{\mathrm{com}}}{V_{\ind}^{\mathrm{cnt}}}
\le
\frac{V_{\rw}^{\mathrm{cnt}}}{V_{\ind}^{\mathrm{cnt}}}
=\Theta(M^{-1}).
\]
The set-based allocations and the computable RW satisfy
\[
\frac{V_{\rw}^{\mathrm{set}}}{V_{\ind}^{\mathrm{set}}}
=\Theta(M^{-1}),
\qquad
\frac{V_{\rw}^{\mathrm{com}}}{V_{\rw}^{\mathrm{set}}}=O(1).
\]
All bounds hold uniformly over $D\in\cH_n$.
\end{theorem}
The theorem shows a variance reduction of order $M$ from IND to RW
under either pair of bounds. Computable RW preserves the
variance order of the theoretical set-based RW allocation and requires no more
variance than count-based RW.

A single release after $M$ deletions provides a benchmark for the cost of
intermediate releases under the same GDP budget. Consider a deletion set
$S$ of size $M$ chosen before release. If we delete all its records at once and release only
$\widetilde\beta_S+\varepsilon$, the count-based residual bound prescribes
$V_{\mathrm{single}}:=r^{\mathrm{cnt}}(M)^2/\mu^2$, $\varepsilon\sim N(0,V_{\mathrm{single}}I_p)$.
On $\cH_n$, this release is $\mu$-GDP relative to the matched exact
release $\widehat\beta_S+\varepsilon$. We compare this with the count-based continual mechanisms, where the
same budget certifies the full sequence of released models.

\begin{corollary}
\label{cor:single-release}
Under Assumptions~\ref{ass:loss}--\ref{ass:identification}, suppose
$M=o\{n/\polylog(n)\}$. Whenever $V_{\mathrm{single}}>0$, as $n\to\infty$,
\[
\frac{V_{\rw}^{\mathrm{cnt}}}{V_{\mathrm{single}}}=1+o(1).
\]
If also $M\to\infty$, then
\[
\frac{V_{\ind}^{\mathrm{cnt}}}{V_{\mathrm{single}}}
=\{1+o(1)\}\frac{M}{5}.
\]
\end{corollary}

Repeated releases can make approximate unlearning and exact
retraining easier to distinguish. Under GDP composition, $M$ releases
each conditionally certified at $\mu$-GDP give a $\sqrt{M}\mu$-GDP guarantee for the
full sequence \citep{dong2022gaussian}. For example, assigning $\mu/\sqrt{M}$ to each IND
release keeps the sequence at $\mu$-GDP but multiplies each noise
variance by $M$ relative to certifying that release alone. For the
optimized count-based allocations, the corollary shows that, as $M\to\infty$,
IND has worst-case variance
asymptotically $M/5$ times the single-release variance, whereas RW has
the same worst-case variance asymptotically. Thus RW avoids an additional
factor of order $M$ in this comparison.

\begin{remark}
\label{rem:set-single-release}
Set-based and computable RW have worst-case variance of the same order
as the largest set-based single-release variance below the cap.
The theoretical set-based IND allocation requires a variance larger by
a factor of order $M$. See {\color{black}Supplementary Section~\ref{app:set-single-release}} for details
and proof.
\end{remark}

\subsection{Covariance optimality}
\label{sec:covariance-adaptive-singleton}
The preceding results compare IND with RW after optimizing each
allocation under its certification constraint. We next examine whether a
different Gaussian covariance can use less noise under the same residual
or increment bounds. For this comparison, we allow an arbitrary joint
Gaussian covariance fixed before releases.

For the covariance comparison, each deletion request removes one
observation. Let
\[
\mathcal C_1
=\{P=(S_0,\ldots,S_M)\in\mathcal C:|S_t|=t,\ t=0,\ldots,M\}
\]
be the set of full singleton deletion paths. For the analysis, any
sequence that stops after $T<M$ singleton deletions can be extended to
$M$ deletions without changing the first $T$ releases. Because discarding
later releases cannot increase distinguishability, a certificate for the
full sequence also covers these first $T$ releases. It therefore suffices
to establish certification on $\mathcal C_1$. Throughout this subsection, path maxima and variance
allocations are taken over this class.

Since every path in $\mathcal C_1$ contains $M$ releases, we represent the
joint noise across releases by a positive definite covariance matrix
$\Sigma_D\in\R^{pM\times pM}$. This matrix may depend on $D$ but is fixed
before any releases. We generate noise
$(Z_1^\top,\ldots,Z_M^\top)^\top\mid D\sim N(0,\Sigma_D)$.
Release $Y_t=\widetilde\beta_{S_t}+Z_t$. The matched retraining experiment uses
the same Gaussian law and deletion policy.

In this comparison, a candidate covariance must certify the full sequence
for all possible residuals satisfying the chosen bounds. Each choice
specifies a residual vector for every allowed deletion set. We consider
three cases: bounding each residual, bounding the increment in residual, and imposing both bounds together.

To define these classes, let ${\zeta}$ assign a possible residual vector to
each allowed deletion set:
${\zeta}:\{S\subseteq[n]:|S|\le M\}\longrightarrow\R^p$, ${\zeta}(\varnothing)=0$.
All sets below have size at most $M$. For such functions, given residual
bounds $r(S)$ and increment bounds $d(A,B)$, define
\[
\begin{aligned}
\mathcal U_0&=\{{\zeta}:\|{\zeta}(S)\|_2\le r(S)\text{ for every }S\},\\
\mathcal U_\Delta&=\{{\zeta}:\|{\zeta}(B)-{\zeta}(A)\|_2\le d(A,B)
 \text{ for every }A\subset B,\ |B\setminus A|=1\}.
\end{aligned}
\]
The class $\mathcal U_0\cap\mathcal U_\Delta$ imposes both constraints.
For count-based bounds, take $r(S)=r^{\mathrm{cnt}}(|S|)$ and
$d(A,B)=d^{\mathrm{cnt}}(|A|,|B|)$. For set-based bounds, take
$r(S)=r^{\mathrm{set}}(S)$ and $d(A,B)=d^{\mathrm{set}}(A,B)$.

For a given ${\zeta}$ and $P=(S_0,\ldots,S_M)\in\mathcal C_1$, write
$u_{\zeta}(P)=({\zeta}(S_1)^\top,\ldots,{\zeta}(S_M)^\top)^\top\in\R^{pM}$ for the vector
of residuals along that path. The following theorem expresses
certification of the full sequence as a quadratic constraint on this
vector.
\begin{theorem}
\label{thm:covariance-adaptive-singleton}
For either pair of bounds and
$\mathcal U\in\{\mathcal U_0,\mathcal U_\Delta,\mathcal U_0\cap\mathcal U_\Delta\}$,
the fixed-covariance mechanism satisfies \eqref{eq:target-one} for every
${\zeta}\in\mathcal U$ and every allowed deletion policy if and only if
\[
u_{\zeta}(P)^\top\Sigma_D^{-1}u_{\zeta}(P)\le\mu^2
\qquad\text{for all }{\zeta}\in\mathcal U,\ P\in\mathcal C_1.
\]
\end{theorem}
To optimize noise subject to this criterion, let
$\Sigma_{D,tt}\in\R^{p\times p}$ denote the $t$th diagonal block of
$\Sigma_D$, so $\Sigma_{D,tt}=\operatorname{Cov}(Z_t\mid D)$.
The variance of the projected noise $v^\top Z_t$ can differ across
unit vectors $v$. We therefore measure noise at each release by the largest
of these projection variances,
$
\max_{\|v\|_2=1}\operatorname{Var}(v^\top Z_t\mid D)
=\lambda_{\max}(\Sigma_{D,tt})
$
\citep{cai2010optimal,chen2021asymmetry,depersin2022robust}.
For IND and count-based RW, this criterion equals $\tau_t^2$ and
$\sum_{s=1}^t\sigma_s^2$, respectively, recovering the variance measures used above.
We minimize the
maximum variance across releases over covariances satisfying the
certificate:
\[
\mathfrak V_D(\mathcal U)=\inf_{\Sigma_D\succ0}
\left\{\max_{t\le M}\lambda_{\max}(\Sigma_{D,tt}):
 u_{\zeta}(P)^\top\Sigma_D^{-1}u_{\zeta}(P)\le\mu^2
 \text{ for all }{\zeta}\in\mathcal U,\ P\in\mathcal C_1\right\}.
\]
Under singleton deletion, count-based IND and RW, together with set-based
IND, belong to the fixed covariance class: their covariance matrices can
be specified before any releases. By contrast, set-based RW adjusts its noise increment
variances to the deleted sets and therefore need not belong to this class.
The following proposition gives the minimax values within the fixed
covariance class and compares them with the worst-case variance of
set-based RW.
\begin{proposition}
\label{prop:covariance-adaptive-singleton-minimax}
Under Assumptions~\ref{ass:loss}--\ref{ass:identification}, suppose
$M=o\{n/\polylog(n)\}$. For every $D\in\cH_n$, the following hold.

\textnormal{(i) Count-based bounds.} The minimax values are
\[
\mathfrak V_D(\mathcal U_0)=R^{\mathrm{cnt}}/\mu^2,\qquad
\mathfrak V_D(\mathcal U_\Delta)=D^{\mathrm{cnt}}/\mu^2{.}
\]
When positive, these values are attained by IND and RW, respectively.
For the intersection, RW is asymptotically minimax:
\[
V_{\rw}^{\mathrm{cnt}}=\{1+o(1)\}\,\mathfrak V_D(\mathcal U_0\cap\mathcal U_\Delta).
\]

\textnormal{(ii) Set-based bounds.} The residual constraint gives
\[
\mathfrak V_D(\mathcal U_0)=R^{\mathrm{set}}/\mu^2,
\]
attained by IND when positive. For either
$\mathcal U\in\{\mathcal U_\Delta,\mathcal U_0\cap\mathcal U_\Delta\}$,
there exists a constant $C_*<\infty$ such that, for all sufficiently large $n$,
\[
\frac{V_{\rw}^{\mathrm{set}}}{C_*}
\le\mathfrak V_D(\mathcal U)
\le C_*{(\log(M+1))^2}\,V_{\rw}^{\mathrm{set}}.
\]
There exist even $n\to\infty$, $M=M_n\to\infty$ satisfying the stated
cap, and data sets $D_n\in\cH_n$ such that
\[
\mathfrak V_{D_n}(\mathcal U)=\Theta({(\log M)^2})\,V_{\rw}^{\mathrm{set}}
\]
simultaneously for both choices of $\mathcal U$.
\end{proposition}
Within the fixed covariance class, count-based and set-based IND
are minimax under their respective residual bounds. Count-based RW is
minimax under increment bounds and asymptotically minimax when both bounds
are imposed. Set-based RW need not belong to this class because its
increment variances can depend on the realized deleted sets. Under increment
bounds, alone or with residual bounds, its worst-case variance is
at most a constant multiple of the fixed covariance optimum. {\color{black}Supplementary Section~\ref{app:matching-separation}}
constructs data sequences on which RW uses a factor of order ${(\log M)^2}$
less variance than every certifying fixed covariance.

\subsection{Utility}
\label{sec:utility}
\leavevmode We now use the Newton approximation error bounds and the calibrated noise variances to establish consistency relative to the unperturbed exact retrained estimator.
The next result gives parameter,
prediction, and GED consistency for IND and RW, uniformly over deletion
policies.
\begin{theorem}
\label{thm:utility-consistency}
Under Assumptions~\ref{ass:loss}--\ref{ass:identification}, suppose
$M=o\{n/\polylog(n)\}$. For a fixed certification budget $\mu>0$,
consider IND and RW with the variance allocations in
\eqref{eq:variance-allocation}.

The following conditions suffice for consistency, uniformly over deletion
policies:
\[
\begingroup
\setlength{\arraycolsep}{5pt}
\renewcommand{\arraystretch}{1.5}
\begin{array}{@{}llcc@{}}
\text{Bounds used}&\text{Consistency}&\text{RW}&\text{IND}\\ \hline
\text{(i) Count-based}&\text{Prediction, GED}
&r^{\mathrm{cnt}}(M)\to0
&r^{\mathrm{cnt}}(M)\sqrt M\to0\\
&\text{Parameter}
&r^{\mathrm{cnt}}(M)\to0
&r^{\mathrm{cnt}}(M)\sqrt{M\log(M+1)}\to0\\[3pt] \hline
\text{(ii) Set-based}&\text{Prediction, GED}
&D^{\mathrm{set}}\xrightarrow{p}0
&D^{\mathrm{set}}M\xrightarrow{p}0\\
&\text{Parameter}
&D^{\mathrm{set}}\xrightarrow{p}0
&D^{\mathrm{set}}M\log(M+1)\xrightarrow{p}0
\end{array}
\endgroup
\]
For computable RW, the RW conditions in (ii) apply with
$D^{\mathrm{com}}$ in place of $D^{\mathrm{set}}$.
\end{theorem}
The sufficient conditions are weaker for RW than for IND,
reflecting its lower calibrated noise variance. For IND, controlling the
largest realized noise norm across releases introduces the additional
logarithmic factor in the parameter consistency condition. A martingale
maximal bound avoids this factor for RW, so the same condition suffices
for parameter, prediction, and GED consistency.

The count-based conditions also suffice for the corresponding set-based
allocations and for computable RW. Substituting
$r^{\mathrm{cnt}}(M)=O\{M^2\polylog(n)/n^2\}$ gives a simple sufficient
range for polynomial deletion caps $M=\Theta(n^\gamma)$: all three
consistency conclusions hold for RW when $0<\gamma<1$ and for IND when
$0<\gamma<4/5$.
{\color{black}Section~\ref{sec:separation-experiment}} examines a deletion cap
within the sufficient range for RW but outside that for IND. In that
experiment, the mean maximum errors decrease for RW and increase overall
for IND over the sample sizes considered.

\section{Numerical experiments}
\label{sec:experiments}
We now evaluate the methods numerically. We start with synthetic logistic
regression experiments comparing count-based IND, count-based RW,
and computable RW under fixed and adaptive deletion rules.
We then apply the same methods to credit default data.
{\color{black}Supplementary Section~\ref{app:running-times}} compares the running times of
these methods with exact retraining.

\subsection{Simulation setup}
\leavevmode The synthetic experiments use a logistic regression model with $p=4$.
A common coefficient vector $\beta^\star$ is drawn once, uniformly from
the unit sphere, and held fixed across all experiments. Each configuration
uses $100$ independent training samples generated with this coefficient
vector. Within each sample, $x_1,\ldots,x_n$ are i.i.d. uniform
on the unit sphere. The responses are then generated independently as
$y_i\mid x_i,\beta^\star \sim \operatorname{Bernoulli}\{1/[1+\exp(-x_i^\top\beta^\star)]\}$.
We fit ridge logistic regression to the generated data. After deleting the
observations indexed by $A$, the exact retrained estimator minimizes
\[
L_A(\beta)
=
\sum_{i\notin A}
\left[
\log\{1+\exp(x_i^\top\beta)\}-y_ix_i^\top\beta
\right]
+\lambda\|\beta\|_2^2.
\]
We use $\lambda=1$ and $\mu=1$ throughout.

All three methods use the Newton approximation in \eqref{eq:one-step}.
We compare the following combinations of noise mechanism and bounds:
\begin{itemize}
\item \textbf{IND with count-based bounds.} At each unlearning step, add
noise with variance
$\tau_t^2=R^{\mathrm{cnt}}/\mu^2$.
\item \textbf{RW with count-based bounds.} At each unlearning step, add
noise with increment variance
$\sigma_t^2=d^{\mathrm{cnt}}(m_{t-1},m_t)\sqrt{D^{\mathrm{cnt}}}/\mu^2$.
\item \textbf{RW with computable bounds.} At each unlearning step, add
noise with increment variance
$\sigma_t^2=d^{\mathrm{set}}(S_{t-1},S_t)\sqrt{D^{\mathrm{com}}}/\mu^2$.
\end{itemize}

To implement these allocations, we need numerical bounds on the derivative
and curvature quantities in the assumptions. Their values need not be
known in advance. For logistic regression, we construct valid bounds
from the observed data and the known form of the loss.
Supplementary Section~\ref{app:logistic-recipe} lists the formulas used to
compute these constants.

All experiments delete one observation at a time, so $m_t=t$ and $T=M$.
In Sections~\ref{sec:residual-bounds-experiment}
and~\ref{sec:variance-utility-experiment}, deletion orders are sampled
uniformly without replacement and fixed before any releases. The three
methods share the training data and deletion order. In
{\color{black}Section~\ref{sec:adaptive-deletion-experiment}}, deletion requests are adaptive.
Errors are evaluated
relative to exact retraining on the remaining observations.

\subsection{Residual bounds}\label{sec:residual-bounds-experiment}
We first compare the residuals and their increments with the bounds used
for certification. This experiment uses $100$ independent repetitions
with $n=10{,}000$, $p=4$, and $M=32$, under the data generation and random
deletion rule above. At each deletion step, we compute
$e_{S_t}=\widetilde\beta_{S_t}-\widehat\beta_{S_t}$.
Figure~\ref{fig:residual-bounds} compares $\|e_{S_t}\|_2$
and $\|e_{S_t}-e_{S_{t-1}}\|_2$ with their count-based and set-based bounds.
\begin{figure}[!htbp]
\centering
\includegraphics[width=\linewidth]{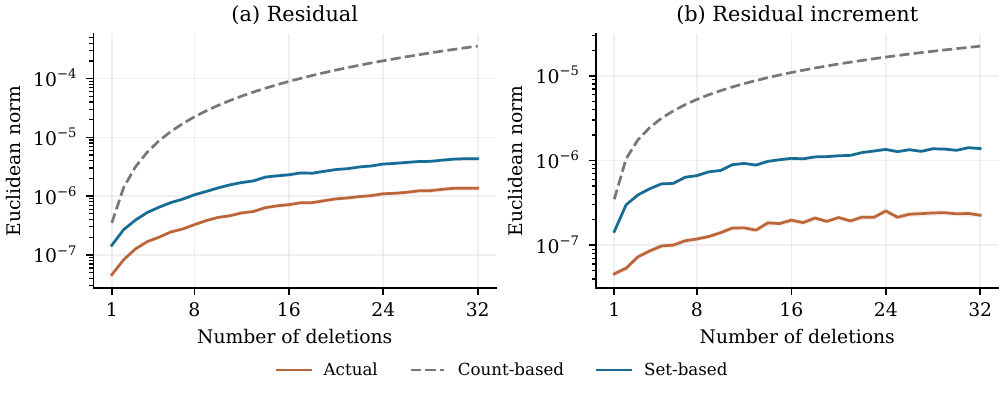}
\caption{Residuals, residual increments, and their bounds along the same deletion
paths. Panel (a) compares $\|e_{S_t}\|_2$,
$r^{\mathrm{set}}(S_t)$, and $r^{\mathrm{cnt}}(m_t)$.
Panel (b) compares $\|e_{S_t}-e_{S_{t-1}}\|_2$,
$d^{\mathrm{set}}(S_{t-1},S_t)$, and $d^{\mathrm{cnt}}(m_{t-1},m_t)$.
Orange curves show observed norms, blue curves show set-based
bounds, and dashed gray curves show count-based bounds.
Curves are means over
$100$ independent repetitions. Bands show one standard error of the mean.
Both vertical axes use a logarithmic scale. Here $n=10{,}000$ and $M=32$.}
\label{fig:residual-bounds}
\end{figure}

In both panels, the bounds lie above the observed errors, and
the set-based bound is tighter than the count-based bound.

\subsection{Variance and utility comparison}\label{sec:variance-utility-experiment}
\leavevmode We next compare the noise variances obtained from the count-based bounds.
The same $100$ training samples as in Figure~\ref{fig:residual-bounds},
with $n=10{,}000$ and $p=4$, are reused across deletion caps
$M\in\{2,4,8,16,32,64\}$. For each sample and cap, we compute the ratio
of the worst-case noise variances
$V_{\rw}^{\mathrm{cnt}}/V_{\ind}^{\mathrm{cnt}}
=D^{\mathrm{cnt}}/R^{\mathrm{cnt}}$.
\begin{figure}[!htbp]
\centering
\includegraphics[width=\linewidth]{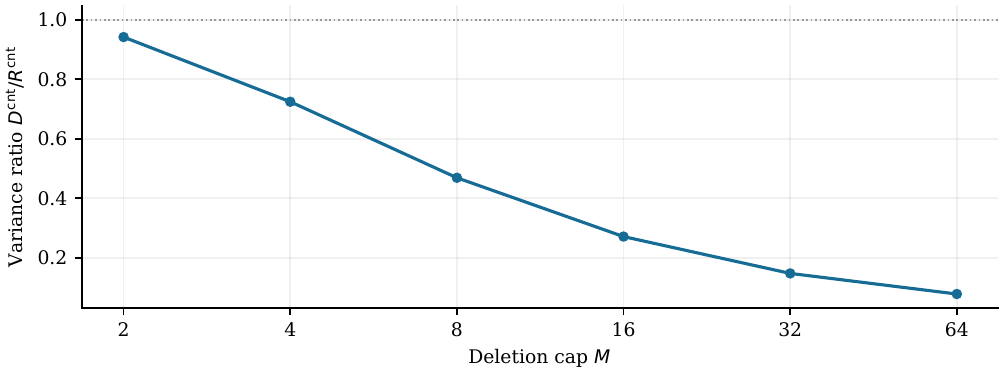}
\caption{Worst-case variance of count-based RW relative to count-based
IND. The curve plots $D^{\mathrm{cnt}}/R^{\mathrm{cnt}}$, averaged over
$100$ independent training samples with a common $\beta^\star$ and $n=10{,}000$.
Error bars show one standard error of the mean
but are largely hidden by the circular markers. The horizontal axis is
logarithmic with base $2$, and the vertical axis is linear.}
\label{fig:calibration-variance}
\end{figure}

In Figure~\ref{fig:calibration-variance}, the variance ratio remains below
one and decreases as the deletion cap grows. Thus count-based RW uses
less worst-case variance than count-based IND throughout this range,
with a larger relative reduction at higher caps.

To examine how these variance differences affect accuracy as the training
sample grows, we take
$n\in\{1000,\allowbreak 2000,\allowbreak 4000,\allowbreak 8000,\allowbreak 16000,\allowbreak 32000,\allowbreak 64000,\allowbreak 128000\}$ and
$M_n=\lfloor\sqrt n/2\rfloor$. At each sample size, we generate $100$
independent training samples with the same fixed $\beta^\star$.
For each sample, we draw a new deletion order uniformly without replacement
and use it for all three methods. Figure~\ref{fig:utility} reports the means of the three
maximum errors
$\max_{t\le M_n}\|Y_t-\widehat\beta_{S_t}\|_2$, $\max_{t\le M_n}\Pred_t$, and $\max_{t\le M_n}\GED_t$.
Each maximum is taken along a trajectory before averaging over
repetitions. {\color{black}Supplementary Section~\ref{app:consistency-experiments}} gives the
evaluation procedure.

\begin{figure}[!htbp]
\centering
\includegraphics[width=\linewidth]{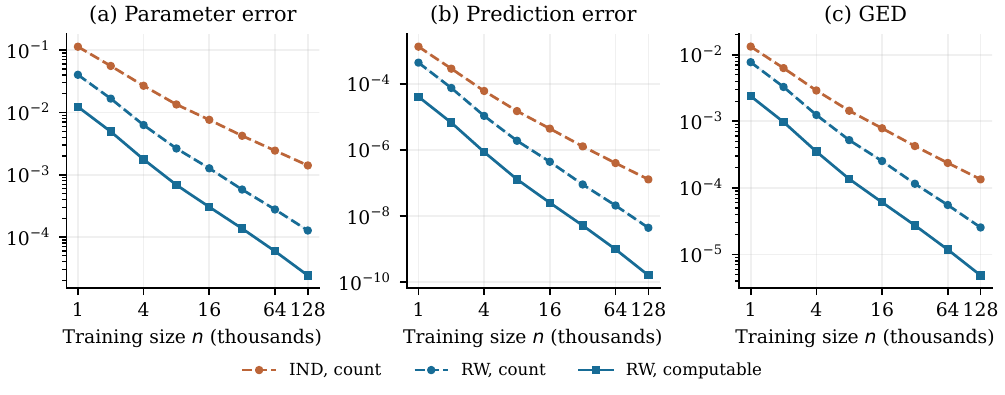}
\caption{Maximum parameter error, prediction error, and GED under random
deletion, with $M_n=\lfloor\sqrt n/2\rfloor$ and $\mu=1$.
Colors distinguish IND and RW, and line styles distinguish the bounds
used. Curves show means over $100$ repetitions at each sample size.
Bands show one standard error of the mean. Both axes are logarithmic.}
\label{fig:utility}
\end{figure}

All three mean maximum errors decrease with sample size for every method,
consistent with the theoretical guarantees in this regime. Both RW
allocations have lower errors than count-based IND, and computable RW
has lower errors than count-based RW. The variance reductions are
therefore accompanied by improved accuracy relative to exact retraining.

\subsection{Adaptive deletion}\label{sec:adaptive-deletion-experiment}
\leavevmode The preceding comparison fixes the deletion order before any releases.
We now allow requests to depend on earlier released models, using the
same sample sizes, deletion caps, and training samples as in
Figure~\ref{fig:utility}. Before any releases, we sample $2M_n-1$
distinct training indices uniformly without replacement. The first is
deleted at release $1$. The others form $M_n-1$ ordered candidate pairs,
one for each subsequent release. Independently, we draw a unit vector
$v$ uniformly from the sphere.

At release $t\ge2$, we delete the first member of the current pair if
$v^\top(Y_{t-1}-\widehat\beta)\ge0$, and the second otherwise, where
$\widehat\beta$ is the full-data fit. All three methods share the training
data, candidate pairs, and projection vector $v$. Each method applies the
rule to its own releases, so the realized deletion paths can differ.
Each method runs one trajectory per repetition.
Figure~\ref{fig:calibration-adaptive} reports the same three maximum
errors, with each release compared against the exact retrained
estimator for that method's deletion set.

\begin{figure}[!htbp]
\centering
\includegraphics[width=\linewidth]{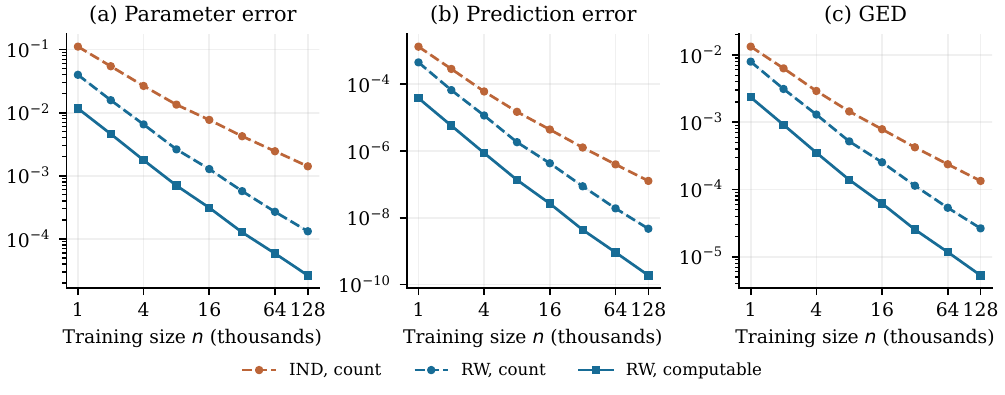}
\caption{Maximum parameter error, prediction error, and GED under
adaptive deletion, with $M_n=\lfloor\sqrt n/2\rfloor$ and $\mu=1$.
Each method chooses deletions using its own preceding releases. Curves
show means over $100$ repetitions at each sample size. Bands show one
standard error of the mean. Both axes are logarithmic.}
\label{fig:calibration-adaptive}
\end{figure}

Under this adaptive rule, the three mean maximum errors again decrease
with sample size. The ordering of the methods is unchanged: both RW
allocations have lower errors than count-based IND, and computable RW
has lower errors than count-based RW.

\subsection{An IND--RW separation}
\label{sec:separation-experiment}
\leavevmode The preceding utility experiments use a deletion cap for which
Theorem~\ref{thm:utility-consistency} gives consistency for both IND and RW.
To examine the difference between their sufficient conditions, we now
compare count-based IND and RW with the cap
$M_n=\lfloor0.05n^{9/10}\rfloor$, which grows more quickly with sample size.
This rate lies within the sufficient range established for RW but outside
the range established for IND.
We retain the same logistic model, fixed $\beta^\star$, $p=4$, and
$\lambda=\mu=1$. The sample sizes are
$n\in\{2000,\allowbreak 4000,\allowbreak 8000,\allowbreak 16000,\allowbreak 32000,\allowbreak 64000,\allowbreak 128000\}$, with respective
caps $46,87,162,303,567,1058,1974$.
Each configuration uses $100$ independent training data sets and the fixed
deletion order $1,2,\ldots,M_n$.
For each training sample, each method generates one release
sequence along this common deletion order.
As before, we report
the three maximum errors, evaluated as described in
{\color{black}Supplementary Section~\ref{app:consistency-experiments}}.

\begin{figure}[!htbp]
\centering
\includegraphics[width=\linewidth]{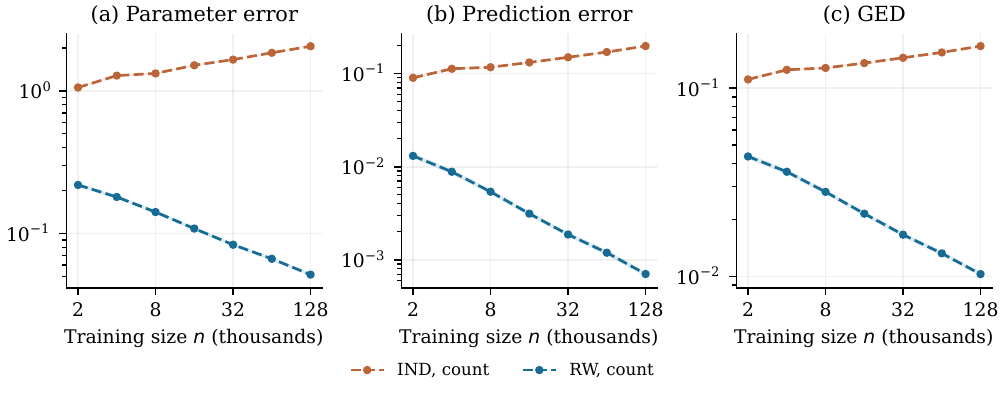}
\caption{Count-based IND and RW with $M_n=\lfloor0.05n^{9/10}\rfloor$
and fixed deletion order. Panels show means of the three maximum errors
over $100$ repetitions. Bands show one standard error of the mean. Both
axes are logarithmic.}
\label{fig:consistency-separation}
\end{figure}

Figure~\ref{fig:consistency-separation} shows different trends under this
larger deletion cap. As $n$ grows, all three mean maximum errors increase
overall for count-based IND and decrease for count-based RW. The maximum
noise variances follow the same respective trends, as shown in
Figure~\ref{fig:consistency-separation-variance} in
{\color{black}Supplementary Section~\ref{app:consistency-experiments}}.

\subsection{Credit default data}
\label{sec:real-data}
We finally evaluate the methods on the Default of Credit Card Clients
data set from the UCI Machine Learning Repository {\citep{default_of_credit_card_clients_350}}, which contains $30{,}000$ records with observed binary
default outcomes.
We select five predictors: credit limit, age, most recent repayment status,
most recent bill amount, and most recent payment amount. Including an
intercept gives $p=6$. We randomly divide the data into $n=24{,}000$
training records and $6000$ test records.

Before fitting the model, we transform the monetary predictors logarithmically, cap repayment
status at $2$, and center and standardize the five predictors using the
training means and standard deviations. After adding an intercept, we
divide all six coordinates by the largest row norm in the training
sample. The same transformation is applied to the test data and held
fixed during deletion. {\color{black}Supplementary Section~\ref{app:credit-details}} specifies the
transforms.

We fit ridge logistic regression with $\lambda=1$ and evaluate accuracy
relative to the same learning algorithm fitted to the retained data.
The working logistic model may be misspecified. We compare count-based
IND, count-based RW, and computable RW, all with certification parameter
$\mu=1$ and deletion cap $M=16$.

The data and training--test split remain fixed over $100$ repetitions.
In each repetition, we independently resample a uniform random deletion
order and release noise. All methods share the deletion order. Exact
retrained estimators are computed only for evaluation.

\begin{figure}[!htbp]
\centering
\includegraphics[width=\linewidth]{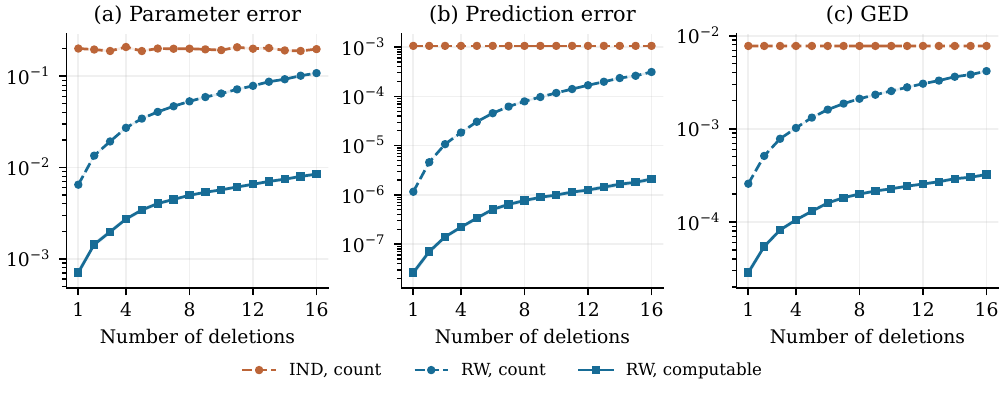}
\caption{Parameter error, prediction error, and GED for IND and RW on
credit default data, with $n=24{,}000$, $p=6$, $M=16$, and
$\lambda=\mu=1$. All errors are relative to exact retraining on the
retained data. Prediction error and GED average over the test sample
and fresh release noise, conditional on the current deletion set and
preceding releases.}
\label{fig:real-credit}
\end{figure}

Figure~\ref{fig:real-credit} compares the three allocations along the
same deletion sequences. As in the simulations, both RW allocations
have lower parameter error, prediction error, and GED than count-based
IND. Computable RW has the smallest errors throughout the evaluated
sequence.

\section{Discussion}
\label{sec:discussion}
We develop Gaussian mechanisms for continual certified unlearning
under adaptive deletions. Calibrating RW to residual increments reduces
the noise cost relative to calibrating IND to residual levels. With
count-based bounds, RW asymptotically matches the worst-case variance of
a single release at the deletion cap. The resulting bounds also give
consistency relative to exact retraining, uniformly over deletion policies,
under weaker sufficient conditions for RW than for IND.

Our analysis assumes fixed dimension, smooth convex losses, and local
population curvature. Extending the residual and increment bounds to
high dimensions is one direction for future work. Our covariance minimax results concern mechanisms with covariance
fixed before release. Characterizing the optimal variance when noise
covariances can adapt to deleted records and preceding releases remains open.

\bibliographystyle{plainnat}
\bibliography{reference}

\clearpage
\section*{Supplementary material}
\pdfbookmark[0]{Supplementary material}{supplement-title}

This supplement gives the residual and increment bounds used in
the main article and the formulas used in the logistic regression experiments.
It derives bounds for ridge logistic regression and provides
details of the numerical experiments and comparisons of running times.
It then examines how the starting point and number of Newton steps affect
computation and approximation error. The final section provides proofs
of the theoretical results in the main article and this supplement.

\setcounter{section}{0}
\renewcommand{\thesection}{S.\arabic{section}}
\renewcommand{\theHsection}{supp.\arabic{section}}

\section{Residual and variance bounds}\label{app:certificates}
We give the bounds on the Newton residual $e_S=\widetilde\beta_S-\widehat\beta_S$
and its increments used to calibrate IND and RW, respectively.
Because deletion sets may be chosen adaptively, these bounds must hold
simultaneously over all sets below the cap. We first specify a common event
on which the exact retrained estimators lie in a common neighborhood
and the loss derivatives are controlled,
then bound the residual levels and the changes along a deletion path.
Finally, we bound the path maxima needed to compute the noise variances.

\subsection{Uniform event and localization}
\leavevmode The following lemma provides the uniform bounds needed
to control the Newton residuals and their increments. It places all exact
retrained estimators in the common neighborhood $\cN$ from
Assumption~\ref{ass:identification}, so the relevant Taylor expansions stay
within that region. The Hessian lower bound controls the inverse Hessians,
and the derivative bounds control the Taylor remainders. These properties
hold on one high-probability event $\cH_n$ for all deletion sets with
$|A|\le M$, so the same bounds apply to adaptive deletion requests. Here $G_n$ and $B_n$ bound
individual gradients and Hessians. $C_{\ell,n}$ and $C_{L,n}$ control Hessian
variation for an individual loss and a retained objective, respectively.

\begin{lemma}
\label{lem:uniform-event}
Under Assumptions~\ref{ass:loss}--\ref{ass:identification}, there are deterministic sequences
$G_n,\allowbreak B_n,\allowbreak C_{\ell,n},\allowbreak C_{L,n}$ such that
\[
G_n,B_n,C_{\ell,n}=O(\polylog(n)),
\qquad
C_{L,n}=O(n\polylog(n)),
\qquad
{\color{black}C_{L,n}\ge C_{\ell,n}.}
\]
For any cap satisfying
$M=o({n/\polylog(n)})$,
{\color{black}the constants satisfy $MC_{\ell,n}/C_{L,n}\to0$, with
$0/0$ interpreted as zero, and}
there are an event $\cH_n$ and a sequence $\kappa_n\ge c_0n/4$ such that
$\Pp(\cH_n)\ge1-\delta_n$, $\delta_n=o(1)$.
On $\cH_n$, for every $|A|\le M$, every $i\le n$, and every
$\beta\in\cN$,
\[
\nabla^2L_A(\beta)\succeq\kappa_nI_p,
\qquad
\|\nabla\ell_i(\beta)\|_2\le G_n,
\qquad
\|\nabla^2\ell_i(\beta)\|_2\le B_n.
\]
Moreover, for every $i\le n$, every $|A|\le M$, and all
$\beta,\beta'\in\cN$,
\[
\|\nabla^2\ell_i(\beta)-\nabla^2\ell_i(\beta')\|_2
\le C_{\ell,n}\|\beta-\beta'\|_2,
\]
and
\[
\|\nabla^2L_A(\beta)-\nabla^2L_A(\beta')\|_2
\le C_{L,n}\|\beta-\beta'\|_2.
\]
The constants $C_{\ell,n}$ and $C_{L,n}$ are uniform over $i$ and $A$,
respectively, and $\widehat\beta_A\in\cN$ for every $|A|\le M$.
Here $\ell$ and $L$ refer to the individual loss $\ell_i$ and the retained
objective $L_A$, respectively.
\end{lemma}

\subsection{Count-based and set-based bounds}
\label{app:residual-increment-bounds}
On $\cH_n$, the following bounds control the residuals and their
increments for all deletion sets below the cap.

\textbf{Count-based residuals.}
For $0\le m\le M$, define
\[
r^{\mathrm{cnt}}(m):=\frac{C_{L,n}G_n^2m^2}{2\kappa_n^3}
=O\!\left(\frac{m^2\polylog(n)}{n^2}\right).
\]
Then $\|e_S\|_2\le r^{\mathrm{cnt}}(|S|)$ for every $|S|\le M$.

\textbf{Set-based residuals.}
Let $g_S=\sum_{i\in S}\nabla\ell_i(\widehat\beta)$,
$H_S=\nabla^2L_S(\widehat\beta)$, and
$\rho_S=\|g_S\|_2/\kappa_n$. Define
\[
r^{\mathrm{set}}(S):=\frac{C_{L,n}}2\|H_S^{-1}\|_2\rho_S^2.
\]
Then $\|e_S\|_2\le r^{\mathrm{set}}(S)\le r^{\mathrm{cnt}}(|S|)$.

\textbf{Count-based increments.}
For $0\le k<m\le M$, define
\[
\begin{aligned}
d^{\mathrm{cnt}}(k,m)
&:=\frac{G_n^2(m-k)}{2\kappa_n^3}
\left[C_{L,n}(k+m)+k^2\left(C_{\ell,n}
+\frac{B_nC_{L,n}}{\kappa_n+(m-k)B_n}\right)\right]\\
&=O\!\left(\frac{(m-k)(k+m)\polylog(n)}{n^2}\right).
\end{aligned}
\]
Then $\|e_B-e_A\|_2\le d^{\mathrm{cnt}}(k,m)$ for every
$A\subsetneq B$ with $|A|=k$ and $|B|=m$.

\textbf{Set-based increments.}
For $A\subsetneq B$ and $q=|B\setminus A|$, define
\[
\begin{aligned}
\delta_{A,B}:={}&\frac{1}{\kappa_n}
\min\left\{qG_n,\;\|g_B-g_A\|_2+qB_n\rho_A\right\},\\
U_{A,B}:={}&\frac{\|H_B^{-1}\|_2}{2}
\left[C_{L,n}(\rho_A+\rho_B)\delta_{A,B}
+qC_{\ell,n}\rho_A^2\right]
+\frac{C_{L,n}\rho_A^2}{2}
\|H_B^{-1}(H_A-H_B)H_A^{-1}\|_2.
\end{aligned}
\]
The increment bound is
\[
d^{\mathrm{set}}(A,B)
:=\min\{U_{A,B},r^{\mathrm{set}}(A)+r^{\mathrm{set}}(B)\},
\]
which satisfies
$\|e_B-e_A\|_2\le d^{\mathrm{set}}(A,B)
\le d^{\mathrm{cnt}}(|A|,|B|)$.

For $A\subsetneq B$ with $|B|\le M$, the set-based bounds satisfy
\begin{equation}
d^{\mathrm{set}}(\varnothing,B)=r^{\mathrm{set}}(B),
\qquad
|r^{\mathrm{set}}(B)-r^{\mathrm{set}}(A)|\le d^{\mathrm{set}}(A,B).
\label{eq:bound-relations}
\end{equation}
These relations connect residual levels to increments in the variance
and covariance comparisons. Write
$r_{*}^{\mathrm{set}}=\max_{|S|\le M}r^{\mathrm{set}}(S)$ for the largest
residual bound. When $r_*^{\mathrm{set}}>0$, a maximizing set $S_*$ is
nonempty, so the one-step path $(\varnothing,S_*)$ is allowed and
$d^{\mathrm{set}}(\varnothing,S_*)=r_*^{\mathrm{set}}$ by
\eqref{eq:bound-relations}. Thus
$r_*^{\mathrm{set}}\le\sqrt{D^{\mathrm{set}}}$ and
$(r_*^{\mathrm{set}})^2\le R^{\mathrm{set}}$.
These inequalities also hold when $r_*^{\mathrm{set}}=0$.
The maximum defining $R^{\mathrm{set}}$ is attained by a full singleton
deletion path, since splitting batches into single deletions and extending
the path to $M$ deletions only adds nonnegative terms to the sum.
\subsection{Computable variance bounds}\label{app:computable-caps}

The set-based variance allocations depend on the pathwise quantities
\[
\begin{aligned}
R^{\mathrm{set}}
&=
\max_{P=(S_0,\ldots,S_T)\in\mathcal C}
\sum_{t=1}^T
\left\{r^{\mathrm{set}}(S_t)\right\}^2,\\
D^{\mathrm{set}}
&=
\max_{P=(S_0,\ldots,S_T)\in\mathcal C}
\left\{
\sum_{t=1}^T
d^{\mathrm{set}}(S_{t-1},S_t)
\right\}^2
\end{aligned}
\]
for IND and RW, respectively. Since the set-based bounds depend on the identities of the deleted
observations, evaluating either maximum directly requires optimization over
deletion paths.

For IND, the count-based residual bound immediately gives
\[
R^{\mathrm{set}}
\le
R^{\mathrm{cnt}}
=
\sum_{m=1}^M
\left\{r^{\mathrm{cnt}}(m)\right\}^2.
\]
Using this upper bound in the set-based IND allocation gives exactly the
count-based IND variance. We therefore do not introduce a
computable set-based IND allocation.

For RW, however, replacing $D^{\mathrm{set}}$ with
$D^{\mathrm{cnt}}$ still retains the set-dependent factor
$d^{\mathrm{set}}(S_{t-1},S_t)$ in each increment variance.
We also derive a computable upper bound on $D^{\mathrm{set}}$ from the
gradients at the original fit. Taking the minimum of these two upper bounds
gives $D^{\mathrm{com}}$ and can further reduce the noise variances. All results below hold on $\cH_n$.

\subsubsection{Count-based path maximum}\label{app:count-path-maximum}

We first compute the count-based benchmark $D^{\mathrm{cnt}}$.
For $0\le k<m\le M$, the count-based increment bound satisfies
\[
d^{\mathrm{cnt}}(k,m)
\le
\sum_{t=k+1}^m d^{\mathrm{cnt}}(t-1,t).
\]
Indeed,
$
(m-k)k^2
\le
\sum_{t=k+1}^m(t-1)^2,
{1/(\kappa_n+(m-k)B_n)}
\le
{1/(\kappa_n+B_n)},
$
while the leading term involving $C_{L,n}(m^2-k^2)$ is additive over the
intermediate deletion counts. Hence splitting a batch into singleton
deletions cannot decrease the total count-based increment bound.

It follows that the largest path sum is attained by the full singleton
path, and therefore
\[
\sqrt{D^{\mathrm{cnt}}}
=
\sum_{m=1}^M d^{\mathrm{cnt}}(m-1,m).
\]

\subsubsection{Computable bound for set-based RW}
\label{app:final-computable-bounds}

Direct evaluation of $D^{\mathrm{set}}$ requires maximizing the sum of
set-based increment bounds over all deletion paths. We avoid this
combinatorial optimization by using the fact that each observation can be
deleted at most once.

Let $W_M$ be the sum of the $M$ largest values among
$\{\|\nabla\ell_i(\widehat\beta)\|_2:i=1,\ldots,n\}$. Thus, for every deletion set $S$ with
$|S|\le M$,
$
\sum_{i\in S}\|\nabla\ell_i(\widehat\beta)\|_2\le W_M.
$

If $C_{L,n}=0$, then $C_{\ell,n}=0$ and all residual and increment bounds vanish, so no
noise is required. We may therefore assume $C_{L,n}>0$. Define $\alpha_n$ and $\eta_n$ by
$
\alpha_n={C_{L,n}/(2\kappa_n^3)},
\eta_n={MC_{\ell,n}/C_{L,n}+3MB_n/\kappa_n}.
$
Summing the set-based increment bound along an arbitrary deletion path and
using that each deleted observation contributes at most once gives
\[
\sqrt{D^{\mathrm{set}}}
\le
\alpha_n W_M^2(1+\eta_n).
\]
The proof is given in {\color{black}Supplementary Section~\ref{app:computable-proof}}.

Combining this bound with
$D^{\mathrm{set}}\le D^{\mathrm{cnt}}$, define
\[
D^{\mathrm{com}}
=
\min\left\{
D^{\mathrm{cnt}},
\alpha_n^2W_M^4(1+\eta_n)^2
\right\}.
\]
Then
$
D^{\mathrm{set}}
\le
D^{\mathrm{com}}
\le
D^{\mathrm{cnt}}.
$ The resulting computable RW allocation uses the realized set-based
increment bound at each release,
\[
(\sigma_t^{\mathrm{com}})^2
=
\frac{
d^{\mathrm{set}}(S_{t-1},S_t)
\sqrt{D^{\mathrm{com}}}
}{\mu^2}.
\]

Given the gradients at the original fit, $W_M$ is obtained by sorting their norms,
so $D^{\mathrm{com}}$ can be computed in $O(n\log n)$ operations.

\section{Logistic regression formulas used in the experiments}
\label{app:logistic-recipe}

\leavevmode To implement the allocations in
{\color{black}Section~\ref{sec:experiments}}, we specialize the preceding bounds to ridge
logistic regression with binary responses, $\|x_i\|_2\le1$, and $\lambda>0$.
We first locate a ball containing all exact retrained estimators, then
bound the derivatives on that ball.
{\color{black}Supplementary Section~\ref{app:implementation}} gives the full derivation and guarantees.
At the full-data fit $\widehat\beta$, compute
\[
\begin{aligned}
\pi_i&=\{1+\exp(-x_i^\top\widehat\beta)\}^{-1},
&g_i&=(\pi_i-y_i)x_i,\\
H_i&=\pi_i(1-\pi_i)x_ix_i^\top,
&H_0&=2\lambda I_p+\sum_{i=1}^n H_i.
\end{aligned}
\]
Let $\|H\|_{(j)}$ be the $j$th largest value of $\|H_i\|_2$.
Set $h_i=\pi_i(1-\pi_i)x_i^\top H_0^{-1}x_i$, and let $h_{(j)}$
be the $j$th largest $h_i$.

\begin{samepage}
The bound $W_M$ from Supplementary Section~\ref{app:computable-caps} satisfies
$\|\nabla L_A(\widehat\beta)\|_2\le W_M$ for every $|A|\le M$.
For the Hessian bound
$\nabla^2L_A(\widehat\beta)\succeq\underline\alpha I_p$, compute
\[
\underline\alpha
=\max\Bigl\{2\lambda,\lambda_{\min}(H_0)-\sum_{j=1}^M\|H\|_{(j)},
\Bigl(1-\sum_{j=1}^M h_{(j)}\Bigr)\lambda_{\min}(H_0)\Bigr\}.
\]
\end{samepage}
Omit the last candidate for $\underline\alpha$ if $\sum_{j=1}^M h_{(j)}\ge1$.
Let $\rho$ be the unique nonnegative solution of
$2\lambda\rho+(\underline\alpha-2\lambda)(1-e^{-\rho})=W_M$.
The ball of radius $\rho$ around $\widehat\beta$ then contains
$\widehat\beta_A$ for every $|A|\le M$.

On this ball, each linear predictor lies in an interval that we
use to bound the loss derivatives. For each $i$, set $a_i=x_i^\top\widehat\beta-\rho\|x_i\|_2$ and
$b_i=x_i^\top\widehat\beta+\rho\|x_i\|_2$.
Write $\pi(z)=\{1+e^{-z}\}^{-1}$,
$w(z)=\pi(z)\{1-\pi(z)\}$, and
$m_i=\max_{z\in[a_i,b_i]}|w'(z)|$.
The maximum of $w$ is attained at the point of $[a_i,b_i]$ closest to
zero. To compute $m_i$, check the interval endpoints and any of
$\pm\log(2+\sqrt3)$ inside it.
To account for feature directions in the retained Hessian, form
$C_L^{\mathrm{mat}}=\|\sum_{i=1}^n m_i\|x_i\|_2x_ix_i^\top\|_2$.
The constants used in the residual and increment bounds are
\[
\begin{aligned}
\underline\kappa&=2\lambda+e^{-\rho}(\underline\alpha-2\lambda),\\
G&=\max_i\|x_i\|_2
    \max\{|\pi(a_i)-y_i|,|\pi(b_i)-y_i|\},\\
B&=\max_i\|x_i\|_2^2\max_{z\in[a_i,b_i]}w(z),
&C_\ell&=\max_i m_i\|x_i\|_2^3,\\
C_L&=\min\{nC_\ell,\ \max[C_L^{\mathrm{mat}},\,M C_\ell\log(n+1)]\}.
\end{aligned}
\]
This choice also ensures $MC_\ell/C_L\to0$ when $M=o(n)$,
with the ratio interpreted as zero if $C_\ell=0$.
Substitute $(\underline\kappa,G,B,C_\ell,C_L)$ for
$(\kappa_n,G_n,B_n,C_{\ell,n},C_{L,n})$ in the residual and increment
bounds of Supplementary Section~\ref{app:residual-increment-bounds} and the computable
variance bound of Supplementary Section~\ref{app:computable-caps}.
These bounds and the resulting $D^{\mathrm{com}}$ determine the
noise variances used in the experiments.

\section{Derivation of the logistic regression bounds}\label{app:implementation}
For ridge logistic regression with binary responses,
$\|x_i\|_2\le1$, and $\lambda>0$, evaluating the bounds in
Supplementary Section~\ref{app:certificates} requires
$(\underline\kappa,G,B,C_\ell,C_L)$ satisfying the following inequalities
for every $|A|\le M$, every observation $i$, and every $\beta,\beta'$ in a
common ball around $\widehat\beta$ containing all retained fits:
\begin{equation}
\begin{aligned}
\nabla^2L_A(\beta)&\succeq\underline\kappa I_p,\\
\|\nabla\ell_i(\beta)\|_2&\le G,
&\|\nabla^2\ell_i(\beta)\|_2&\le B,\\
\|\nabla^2\ell_i(\beta)-\nabla^2\ell_i(\beta')\|_2
&\le C_\ell\|\beta-\beta'\|_2,\\
\|\nabla^2L_A(\beta)-\nabla^2L_A(\beta')\|_2
&\le C_L\|\beta-\beta'\|_2.
\end{aligned}
\label{eq:logistic-targets}
\end{equation}
The general theory establishes the existence of these constants but does
not give computable values for a given dataset. For ridge logistic
regression, we first find the common ball from the observed data and the
full-data fit. We then compute the five constants on that ball.

\subsection{A common ball and a Hessian lower bound}
We seek a radius $\rho$ such that
$\|\widehat\beta_A-\widehat\beta\|_2\le\rho$ for every $|A|\le M$.
At the full-data fit, deletion changes both the gradient and the Hessian
of the objective. Bounds on these changes, combined with the way the
logistic Hessian varies away from $\widehat\beta$, determine $\rho$.

\noindent\hspace*{\parindent}Let $\pi_i=\{1+\exp(-x_i^\top\widehat\beta)\}^{-1}$.
The gradients and Hessians at the original fit are $g_i=(\pi_i-y_i)x_i$ and
$H_i=\pi_i(1-\pi_i)x_ix_i^\top$, with
$H_0=2\lambda I_p+\sum_{i=1}^nH_i$.
For a deletion set $A$, the full-data first-order condition
gives $\nabla L_A(\widehat\beta)=-\sum_{i\in A}g_i$, while
$\nabla^2L_A(\widehat\beta)=H_0-\sum_{i\in A}H_i$.
For every $|A|\le M$, we have
$\|\nabla L_A(\widehat\beta)\|_2\le W_M$.
We next compute $\underline\alpha$ such that
$\nabla^2L_A(\widehat\beta)\succeq\underline\alpha I_p$
for every $|A|\le M$.

\noindent\hspace*{\parindent}To compute $\underline\alpha$, we bound
$\|\sum_{i\in A}H_i\|_2$ and
$\sup_{v\ne0}\frac{v^\top(\sum_{i\in A}H_i)v}{v^\top H_0v}$.
Let $\|H\|_{(1)}\ge\cdots\ge\|H\|_{(n)}$ be the decreasing rearrangement
of $\{\|H_i\|_2\}_{i=1}^n$. For the relative bound, the maximum
fraction of the original curvature contributed by observation $i$
over all directions is
\[
h_i=\sup_{v\ne0}\frac{v^\top H_iv}{v^\top H_0v}
=\pi_i(1-\pi_i)x_i^\top H_0^{-1}x_i.
\]
Let $h_{(1)}\ge\cdots\ge h_{(n)}$ be the decreasing rearrangement of
$\{h_i\}_{i=1}^n$.
Thus $\nabla^2L_A(\widehat\beta)\succeq\underline\alpha I_p$
for every $|A|\le M$, where
\[
\underline\alpha
=
\max\left\{
2\lambda,\ 
\lambda_{\min}(H_0)-\sum_{j=1}^M\|H\|_{(j)},\
\Bigl(1-\sum_{j=1}^M h_{(j)}\Bigr)\lambda_{\min}(H_0)
\right\},
\]
where the last candidate is omitted if $\sum_{j=1}^M h_{(j)}\ge1$.

By the triangle inequality and the definition of $W_M$ in
Supplementary Section~\ref{app:computable-caps},
$\|\nabla L_A(\widehat\beta)\|_2
\le\sum_{i\in A}\|g_i\|_2\le W_M$
for every $|A|\le M$.

For $\|h\|_2=r$, the logistic Hessian satisfies
$\nabla^2L_A(\widehat\beta+h)\succeq
\{2\lambda+e^{-r}(\underline\alpha-2\lambda)\}I_p$.
Integrating this lower bound to $\widehat\beta_A$ gives
\[
2\lambda r_A+(\underline\alpha-2\lambda)(1-e^{-r_A})
\le W_M,
\qquad r_A=\|\widehat\beta_A-\widehat\beta\|_2.
\]
The equation $2\lambda\rho+(\underline\alpha-2\lambda)(1-e^{-\rho})
=W_M$ has a unique nonnegative root. The ball of radius
$\rho$ around $\widehat\beta$ contains all retained fits. On this ball,
the same Hessian bound gives
$$\underline\kappa=2\lambda+e^{-\rho}(\underline\alpha-2\lambda).$$
\begin{proposition}
\label{prop:logistic-certificate}
Suppose $\|x_i\|_2\le1$ and $\lambda>0$. The values $\rho$
and $\underline\kappa$ defined above satisfy, for every $|A|\le M$,
\[
\|\widehat\beta_A-\widehat\beta\|_2\le\rho,
\]
and, simultaneously for every $|A|\le M$ and
$\|\beta-\widehat\beta\|_2\le\rho$,
\[
\nabla^2L_A(\beta)\succeq\underline\kappa I_p.
\]
\end{proposition}

\subsection{Derivative bounds on the localization ball}
\begingroup
\color{black}
The common ball now gives a fixed interval for each predictor.
For $\|\beta-\widehat\beta\|_2\le\rho$, $x_i^\top\beta\in[a_i,b_i]$,
where $a_i=x_i^\top\widehat\beta-\rho\|x_i\|_2$ and
$b_i=x_i^\top\widehat\beta+\rho\|x_i\|_2$.
Write $\pi(z)=(1+e^{-z})^{-1}$ and $w(z)=\pi(z)\{1-\pi(z)\}$.

The gradient is
$\nabla\ell_i(\beta)=\{\pi(x_i^\top\beta)-y_i\}x_i$. Since $\pi$ is
increasing, its norm is largest at an interval endpoint. Thus the
gradient inequality in \eqref{eq:logistic-targets} holds with
\[
G=\max_i\|x_i\|_2
  \max\{|\pi(a_i)-y_i|,|\pi(b_i)-y_i|\}.
\]
The Hessian is
$\nabla^2\ell_i(\beta)=w(x_i^\top\beta)x_ix_i^\top$, so its norm is
bounded by
\[
B=\max_i\|x_i\|_2^2\max_{z\in[a_i,b_i]}w(z).
\]
The function $w(z)$ is maximized at $z=0$ and decreases as
$|z|$ increases. Thus its maximum on $[a_i,b_i]$ is $w(0)=1/4$ if the
interval contains zero, and its value at the endpoint closest to zero
otherwise.

To control the change in this Hessian, set
$m_i=\max_{z\in[a_i,b_i]}|w'(z)|$.
For $\beta,\beta'$ in the ball, the mean value theorem gives

\[
\begin{aligned}
\|\nabla^2\ell_i(\beta)-\nabla^2\ell_i(\beta')\|_2
&=|w(x_i^\top\beta)-w(x_i^\top\beta')|\|x_i\|_2^2\\
&\le m_i|x_i^\top(\beta-\beta')|\|x_i\|_2^2\\
&\le m_i\|x_i\|_2^3\|\beta-\beta'\|_2.
\end{aligned}
\]

Thus we can take
\[
C_\ell=\max_i m_i\|x_i\|_2^3.
\]
To compute $m_i$, check the two endpoints and any of
$\pm\log(2+\sqrt3)$ in the interval.

For the retained objective, summing the individual Hessian
changes in matrix form gives the bound
\[
C_L^{\mathrm{mat}}
=\left\|\sum_{i=1}^n m_i\|x_i\|_2x_ix_i^\top\right\|_2.
\]
The variance comparison also requires $MC_\ell/C_L\to0$
when $M=o(n)$. We therefore use the valid, possibly larger bound
\[
C_L=\min\left\{nC_\ell,
  \max\left[C_L^{\mathrm{mat}},\{\log(n+1)\}MC_\ell\right]\right\}.
\]

\begin{proposition}
\label{prop:logistic-constants}
For binary responses, $G$, $B$, $C_\ell$, and $C_L$ above
satisfy their respective inequalities in \eqref{eq:logistic-targets}. In
particular, for every deletion set $|A|\le M$ and
every $\beta,\beta'$ in the localization ball,
\[
\|\nabla^2L_A(\beta)-\nabla^2L_A(\beta')\|_2
\le C_L^{\mathrm{mat}}\|\beta-\beta'\|_2
\le C_L\|\beta-\beta'\|_2.
\]
Moreover,
$C_\ell\le C_L^{\mathrm{mat}}\le C_L\le nC_\ell$ and
\[
\frac{MC_\ell}{C_L}
\le\max\left\{\frac Mn,\frac1{\log(n+1)}\right\}\longrightarrow0
\quad\text{if }M=o(n),
\]
where the ratio is zero when $C_\ell=0$.
\end{proposition}

\endgroup

\section{Experiment details}
\label{app:consistency-experiments}

\subsection{Credit data scaling and calibration}
\label{app:credit-details}
We first scale the credit predictors to satisfy the norm bound
of Proposition~\ref{prop:logistic-certificate}, then evaluate the conditional
utility quantities using the held-out test sample.
A fixed random permutation independent of the labels selects $24{,}000$
training records and $6000$ test records.
{\color{black}The predictors are LIMIT\_BAL, AGE, PAY\_0, BILL\_AMT1,
and PAY\_AMT1. We apply $\log u$ to LIMIT\_BAL,
$\operatorname{sign}(u)\log(1+|u|)$ to BILL\_AMT1,
and $\log(1+u)$ to PAY\_AMT1, and replace PAY\_0 by
$\min(\mathrm{PAY\_0},2)$. AGE is unchanged at this stage.
Let $z_i\in\mathbb R^5$ contain the transformed predictors.
For coordinate $j$, let $\bar z_j$ and $s_j$ be its training mean and
sample standard deviation, with denominator $n-1$.
Set $\widetilde z_{ij}=(z_{ij}-\bar z_j)/s_j$,
$a_i=(1,\widetilde z_i^\top)^\top$, and
$K=\max_{i\in\mathrm{training}}\|a_i\|_2$.
Both training and test features are $x_i=a_i/K$, using the same training
means, standard deviations, and $K$. Every training row therefore satisfies
$\|x_i\|_2\le1$. All preprocessing quantities remain fixed during deletion.
The ridge penalty $\lambda\|\beta\|_2^2$ includes the intercept, and exact
retained retraining uses this same feature map and objective.}

{\color{black}Parameter error uses each actual release. Prediction error
and GED have the conditional definitions in {\color{black}Section~\ref{sec:mechanisms}},
with the independent test distribution approximated by the empirical
distribution of the $6000$ test pairs. To integrate over the fresh release noise, write its
conditional mean and scalar variance as follows.
Let $a_t=\widetilde\beta_{S_t}$ and $s_t^2=\tau_t^2$
for IND. For RW, let $a_t=\widetilde\beta_{S_t}+Z_{t-1}$ and
$s_t^2=\sigma_t^2$, where $Z_{t-1}$ is the realized preceding noise.
For each fixed test feature, Lemma~\ref{lem:adaptive-identities} gives
\[
\widehat{\Pred}_t
=\frac1{6000}\sum_{j=1}^{6000}
\left[\{x_j^\top(a_t-\widehat\beta_{S_t})\}^2+s_t^2\|x_j\|_2^2\right].
\]
For GED, we sample test pairs uniformly with replacement and draw independent
standard normal scalars $\eta_j$. We average
$|\ell(y_j,x_j^\top a_t+s_t\|x_j\|_2\eta_j)
-\ell(y_j,x_j^\top\widehat\beta_{S_t})|$.
These evaluation draws are separate from release noise and are shared across
methods and release times. In particular, the realized preceding RW noise
is held fixed. Only its fresh increment is averaged out.}

\subsection{Three error terms}
Figures~\ref{fig:utility} and~\ref{fig:calibration-adaptive} use $n\in\{1000,\allowbreak 2000,\allowbreak 4000,\allowbreak 8000,\allowbreak 16000,\allowbreak 32000,\allowbreak 64000,\allowbreak 128000\}$ and $M_n=\lfloor\sqrt n/2\rfloor$, giving caps $15,22,31,44,63,89,126,178$.

For each simulated deletion trajectory, we evaluate the three utility
quantities used in {\color{black}Section~\ref{sec:mechanisms}}.

The parameter error at release $t$ is computed directly as
$\|Y_t-\widehat\beta_{S_t}\|_2$.

For prediction error, conditional on the realized history before release
$t$, let $a_t=\widetilde\beta_{S_t}$ and $s_t^2=\tau_t^2$ for IND, and let
$a_t=\widetilde\beta_{S_t}+Z_{t-1}$ and $s_t^2=\sigma_t^2$ for RW.
The synthetic features are uniformly distributed on the unit sphere.
Their second moment matrix is $I_p/p$, so
Lemma~\ref{lem:adaptive-identities} gives
\[
\operatorname{Pred}_t
=
\frac{\|a_t-\widehat\beta_{S_t}\|_2^2}{p}+s_t^2.\]

For GED, we approximate the conditional expectation by Monte Carlo.
We generate independent test observations $(x_j,y_j)$ from the data-generating
model and independent $\eta_j\sim N(0,1)$, and compute
\[
\widehat{\GED}_t
=
\frac1B\sum_{j=1}^B
\left|
\ell\!\left(y_j,x_j^\top a_t+s_t\eta_j\right)
-
\ell\!\left(y_j,x_j^\top\widehat\beta_{S_t}\right)
\right|.
\]

For each trajectory, we take the maximum of the parameter error,
prediction error, and GED over all release times. The figures report the
mean of these maxima over the simulation repetitions, with
standard-error bands.
\subsection{Noise variances in the IND--RW example}
Figure~\ref{fig:consistency-separation-variance} gives the maximum noise
variances for the experiment in {\color{black}Section~\ref{sec:separation-experiment}}.
The mean maximum variance rises
from $0.09014$ to $0.1978$ for IND and falls from
$0.01095$ to $0.0005564$ for RW.

\begin{figure}[!htbp]
\centering
\includegraphics[width=\linewidth]{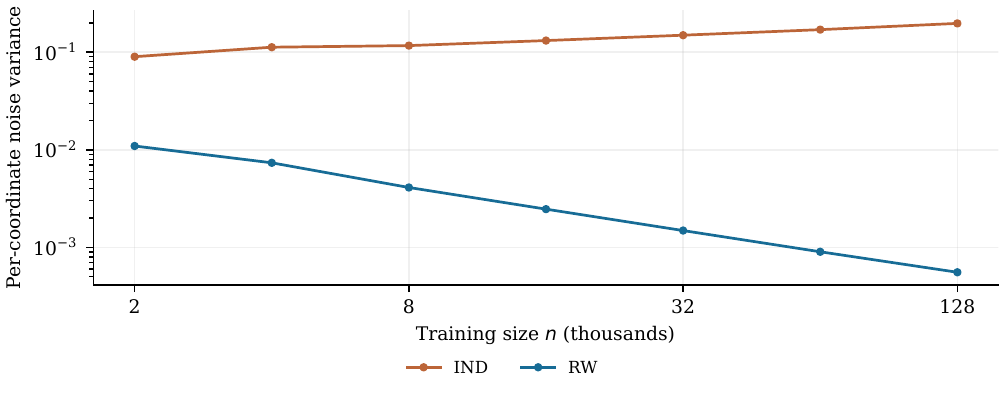}
\caption{Maximum noise variance for count-based IND and RW in the
separation experiment, with $\mu=1$. Curves show means over $100$
repetitions, and bands show one standard error.}
\label{fig:consistency-separation-variance}
\end{figure}
\subsection{Running times}\label{app:running-times}

Tables~\ref{tab:runtime-random}--\ref{tab:runtime-separation} report
median cumulative running times across $100$ repetitions for the complete release
sequences in the corresponding experiments. Times are in milliseconds.
For count-based IND, count-based RW, and computable RW, the times include
unlearning updates, computation of the bounds, variance calibration, and
Gaussian noise sampling and addition. Exact retraining is initialized
at the preceding retained fit. Its times include retraining only,
without certification calibration or added noise.

In the evaluated settings, longer release sequences give larger relative
time savings. For random deletion, the count-based methods are about
$5$ times faster than exact retraining at $M=15$ and $56$ times faster
at $M=178$. Computable RW is about $1.9$ and $38$ times faster,
respectively. In the IND--RW separation experiment, the count-based
methods are about $9$ times faster at $M=46$ and $360$ times faster
at $M=1974$.
\begin{table}[!htbp]
\color{black}
\centering
\caption{Median cumulative running times (ms) for the settings
of Figures~\ref{fig:residual-bounds}, \ref{fig:utility}, and
\ref{fig:real-credit}.}
\label{tab:runtime-random}
\small
\setlength{\tabcolsep}{3pt}
\begin{tabular}{lrrrrrr}
\toprule
Setting & $n$ & $M$ & Count IND & Count RW & Computable RW & Retraining\\
\midrule
Residual & 10000 & 32 & 2.412 & 2.417 & 4.065 & 20.489\\
Random & 1000 & 15 & 0.523 & 0.507 & 1.371 & 2.536\\
Random & 2000 & 22 & 0.797 & 0.784 & 2.014 & 4.473\\
Random & 4000 & 31 & 1.299 & 1.289 & 3.017 & 9.983\\
Random & 8000 & 44 & 2.382 & 2.371 & 4.851 & 25.056\\
Random & 16000 & 63 & 4.235 & 4.228 & 7.826 & 69.137\\
Random & 32000 & 89 & 8.306 & 8.303 & 13.858 & 209.203\\
Random & 64000 & 126 & 16.394 & 16.376 & 25.441 & 607.739\\
Random & 128000 & 178 & 33.364 & 33.343 & 49.071 & 1853.179\\
Credit & 24000 & 16 & 4.786 & 4.775 & 6.181 & 22.436\\
\bottomrule
\end{tabular}
\par\vspace{12pt}
\caption{Median cumulative running times (ms) for adaptive
deletion in Figure~\ref{fig:calibration-adaptive}.}
\label{tab:runtime-adaptive}
\small
\setlength{\tabcolsep}{4pt}
\begin{tabular}{rrrrrr}
\toprule
$n$ & $M$ & Count IND & Count RW & Computable RW & Retraining\\
\midrule
1000 & 15 & 0.517 & 0.507 & 1.346 & 2.506\\
2000 & 22 & 0.795 & 0.787 & 2.023 & 4.431\\
4000 & 31 & 1.300 & 1.284 & 2.990 & 9.820\\
8000 & 44 & 2.392 & 2.380 & 4.844 & 24.812\\
16000 & 63 & 4.250 & 4.236 & 7.891 & 68.018\\
32000 & 89 & 8.326 & 8.312 & 13.991 & 207.281\\
64000 & 126 & 16.402 & 16.390 & 25.472 & 605.442\\
128000 & 178 & 33.386 & 33.381 & 48.943 & 1837.089\\
\bottomrule
\end{tabular}
\par\vspace{12pt}
\caption{Median cumulative running times (ms) for the
IND--RW separation experiment in Figure~\ref{fig:consistency-separation}.}
\label{tab:runtime-separation}
\small
\begin{tabular}{rrrrr}
\toprule
$n$ & $M$ & Count IND & Count RW & Retraining\\
\midrule
2000 & 46 & 0.944 & 0.950 & 8.111\\
4000 & 87 & 1.781 & 1.796 & 24.252\\
8000 & 162 & 3.393 & 3.436 & 79.766\\
16000 & 303 & 6.143 & 6.229 & 273.181\\
32000 & 567 & 12.003 & 12.165 & 1149.982\\
64000 & 1058 & 23.969 & 24.178 & 4685.625\\
128000 & 1974 & 46.153 & 46.787 & 16958.835\\
\bottomrule
\end{tabular}
\end{table}

\section{Newton update computation and extensions}\label{app:newton-extensions}
We compare the computational costs of two starting points for one-step Newton updates and quantify the residual reduction from further Newton steps.

\subsection{Computational comparison}\label{app:computational-comparison}
\leavevmode To explain the computational choice in
{\color{black}Section~\ref{sec:gaussian-releases}}, we compare starting each Newton update
from $\widehat\beta$ (fixed start) with starting from the preceding internal
approximation (moving start). Both methods take one step per release and
start from $\widehat\beta$ for the first release.

Consider $M$ singleton deletions, with $M<n$. For the first
release, both updates form the retained Hessian at $\widehat\beta$ from
$n-1$ records. The fixed-start update can retain those contributions:
each later deletion subtracts one Hessian, and each release adds one
gradient of a deleted record to the cumulative gradient. The moving-start update instead evaluates and sums the
derivatives of the remaining $n-t$ records at its new starting point in
release $t\ge2$.
The resulting counts are
\begin{center}
\small
\begin{tabular}{@{}lcc@{}}
\toprule
& Fixed start & Moving start \\
\midrule
Hessian evaluations
& $n-1$ & $Mn-M(M+1)/2$ \\
Gradient evaluations
& $M$ & $(M-1)n-M(M+1)/2+2$ \\
Hessian additions
& $n+M-3$ & $Mn-M(M+1)/2-M$ \\
Gradient additions
& $M-1$ & $(M-1)n-M(M+1)/2-M+2$ \\
\bottomrule
\end{tabular}
\end{center}
For $p\times p$ Hessians, a matrix addition or subtraction costs
$O(p^2)$, while a gradient addition costs $O(p)$. For the single-index
loss in Assumption~\ref{ass:loss}, forming a record's Hessian and
gradient costs $O(p^2)$ and $O(p)$, respectively, if the scalar loss
derivatives have constant evaluation cost. Thus the Hessian computation
and summation in the moving-start update are of order $M$ times those
of the fixed-start update. This factor concerns Hessian work only. Both updates also
require $M$ linear solves, whose common cost is excluded from the table. Retaining all Hessian contributions
from the first release uses $O(np^2)$ storage. Computing a contribution
again when its record is deleted reduces this storage without changing
the order of the fixed-start work.

\subsection{Further Newton steps}\label{app:multi-newton}
\leavevmode Additional Newton steps can reduce the residual at
each release. To quantify this gain, fix a deletion set $A$ with $|A|\le M$ and
set $b_A^{(0)}=\widehat\beta$ and
$b_A^{(j+1)}=b_A^{(j)}-\{\nabla^2L_A(b_A^{(j)})\}^{-1}
\nabla L_A(b_A^{(j)})$ for $j\ge0$.
Thus $b_A^{(1)}=\widetilde\beta_A$ and, on $\cH_n$,
$\|b_A^{(0)}-\widehat\beta_A\|_2\le\rho_A$.
If the ball of radius $\rho_A$ around $\widehat\beta_A$ lies in $\cN$
and $C_{L,n}\rho_A/(2\kappa_n)<1$, then every iterate stays in the ball and
\[
\|b_A^{(K)}-\widehat\beta_A\|_2
\le
\left(\frac{C_{L,n}}{2\kappa_n}\right)^{2^K-1}
\rho_A^{2^K},\qquad K\ge0.
\]

Indeed, for $e_j=b_A^{(j)}-\widehat\beta_A$, Taylor's integral formula
and the bounds on curvature and its Lipschitz constant give
$\|e_{j+1}\|_2\le C_{L,n}\|e_j\|_2^2/(2\kappa_n)$ on this ball.
Induction proves the bound and keeps the iterates in the ball.
Since $\rho_A\le |A|G_n/\kappa_n$, the residual bound has order
$O\{(|A|/n)^{2^K}\polylog(n)\}$ for each fixed $K$ and tends to zero
as $K\to\infty$ for fixed data under the stated condition.
If $C_{L,n}=0$, one step is exact.

We focus on one Newton step per release because it already suffices
for the certification and consistency results established in this article.
Further steps can reduce the approximation error but require additional
derivative evaluations and linear solves.

\section{Technical proofs}\label{app:proofs}
This section provides proofs of the theoretical results in the main article
and this supplement.

\subsection{Proof of Theorem~\ref{thm:bounds}}\label{app:bounds-proof}
\begin{proof}
We first establish the common event in Lemma~\ref{lem:uniform-event}.
On this event, Taylor remainders give the residual bounds, and
subtracting the remainder equations gives the increment bounds.
Finally, we verify \eqref{eq:bound-relations}, which is used in the later
path comparisons.

\par\medskip\noindent\textbf{Uniform event and localization.}
To use the same Taylor bounds for all deletion sets, we need uniform
derivative bounds on $\cN$ and all retained estimators to lie in $\cN$.

\textit{Derivative bounds.}
The growth assumption controls the first and third predictor
derivatives. The second derivative has the same growth bound: Taylor's
theorem applied to $\partial_z\ell(y,\cdot)$ gives
\[
|\partial_z^2\ell(y,z)|
\le |\partial_z\ell(y,z+1)|+|\partial_z\ell(y,z)|
+\frac12\sup_{u\in[0,1]}|\partial_z^3\ell(y,z+u)|.
\]
Since $p$ is fixed, sub-Gaussian tails and the response tail condition
give a deterministic $u_n=O(\sqrt{\log n})$ such that the event
\[
\max_{i\le n}\|x_i\|_2\le u_n,
\qquad \max_{i\le n}|y_i|\le B_y(n)
\]
has probability $1-o(1)$ by a union bound.
On the event that both maxima are below these thresholds, compactness
of $\cN$ and the chain rule therefore give deterministic bounds
$G_n,B_n,C_{\ell,n}=O(\polylog(n))$ for the individual gradients,
Hessians, and Hessian Lipschitz constants, respectively.
Set
$C_{L,n}=nC_{\ell,n}+\lambda C_{\mathrm{reg}}(n)$.
This bounds the Hessian Lipschitz constant of every retained objective,
and gives $C_{L,n}=O(n\polylog(n))$, $C_{L,n}\ge C_{\ell,n}$, and
$MC_{\ell,n}/C_{L,n}\le M/n\to0$, with $0/0$ interpreted as zero.

\textit{Curvature and localization.}
We next prove that, on a common event of probability $1-o(1)$,
for every $|A|\le M$,
\[
\nabla^2L_A(\beta)\succeq\kappa_n I_p
\quad\text{for every }\beta\in\cN,
\qquad\text{and}\qquad
\widehat\beta_A\in\cN.
\]

\textit{1. Population curvature.}
Assumption~\ref{ass:identification} gives
\[
\E\nabla^2\ell_i(\beta)\succeq c_0I_p
\qquad\text{for every }\beta\in\cN.
\]

\textit{2. Full-data curvature.}
To transfer this lower bound to the sample, we first prove
\[
\sup_{\beta\in\cN}
\left\|\frac1n\sum_{i=1}^n\nabla^2\ell_i(\beta)
-\E\nabla^2\ell_i(\beta)\right\|_2\xrightarrow{p}0.
\]
Whenever the supremum is at most $c_0/4$, convexity of the regularizer
and the population lower bound give, for every $\beta\in\cN$,
\[
\begin{aligned}
\nabla^2L_\varnothing(\beta)
&=\sum_{i=1}^n\nabla^2\ell_i(\beta)
+\lambda\nabla^2\mathcal P(\beta)\\
&\succeq n\left(\E\nabla^2\ell_i(\beta)-\frac{c_0}{4}I_p\right)
\succeq\frac{c_0n}{2}I_p.
\end{aligned}
\]

To justify the uniform convergence, we use a finite net.
To control the supremum, we also bound the Hessian variation between
nearby points. By Assumption~\ref{ass:loss}, the individual gradient and Hessian
norms and Hessian Lipschitz constants on $\cN$ are bounded by
\[
K_i=C(1+|y_i|^s+\|x_i\|_2^s)(1+\|x_i\|_2^3).
\]
In particular,
$\|\nabla^2\ell_i(\beta)-\nabla^2\ell_i(\beta')\|_2
\le K_i\|\beta-\beta'\|_2$ for $\beta,\beta'\in\cN$.
The term involving both the response and the features satisfies
\[
\E\!\left[|y_i|^{4s/3}(1+\|x_i\|_2^3)^{4/3}\right]
\le(\E|y_i|^{2s})^{2/3}
\{\E(1+\|x_i\|_2^3)^4\}^{1/3}.
\]
The response moment assumption and sub-Gaussian feature moments therefore
give $\sup_n\E K_i^{4/3}<\infty$. The remaining terms involve only
feature moments.

Fix $\eta>0$ and take a finite $\eta$-net $\mathcal N_\eta$ of $\cN$.
The uniform $4/3$ moment bound gives the weak law at each net point.
To extend it to all of $\cN$, the Hessian variation bound gives
\[
\begin{aligned}
&\sup_{\beta\in\cN}
\left\|\frac1n\sum_{i=1}^n\nabla^2\ell_i(\beta)
-\E\nabla^2\ell_i(\beta)\right\|_2\\
&\quad\le
\max_{\beta'\in\mathcal N_\eta}
\left\|\frac1n\sum_{i=1}^n\nabla^2\ell_i(\beta')
-\E\nabla^2\ell_i(\beta')\right\|_2
+\eta\left(\frac1n\sum_{i=1}^nK_i+\E K_i\right).
\end{aligned}
\]
For fixed $\eta$, the maximum over the net converges to zero in
probability, while $n^{-1}\sum_iK_i=O_p(1)$ by Markov's inequality.
Letting $n\to\infty$ and then $\eta\downarrow0$ proves the claimed
uniform convergence.

\textit{3. Retained curvature.}
On the event that both observation maxima satisfy their bounds,
$\sum_{i\in A}\nabla^2\ell_i(\beta)\preceq MB_nI_p$ and $MB_n=o(n)$.
Thus, whenever the full-data lower bound holds, for all sufficiently
large $n$, every $|A|\le M$, and every $\beta\in\cN$,
\[
\nabla^2L_A(\beta)
\succeq(c_0n/2-MB_n)I_p\succeq\kappa_nI_p,
\qquad \kappa_n:=c_0n/4.
\]

To prove that every retained estimator lies in $\cN$, we compare
$L_A(\beta_0)$ with its values on the boundary of a ball contained in $\cN$.
Choose a fixed $\rho_0>0$ with
$\overline B(\beta_0,\rho_0)\subset\operatorname{int}(\cN)$.
The bound $K_i$ also justifies differentiation under the expectation and
the weak law for the gradient. Since $\beta_0$ is an interior minimizer,
$\E\nabla\ell_i(\beta_0)=\nabla\mathcal R(\beta_0)=0$. Hence
\[
\frac1n\nabla L_\varnothing(\beta_0)
=\frac1n\sum_{i=1}^n\nabla\ell_i(\beta_0)
+\frac{\lambda}{n}\nabla\mathcal P(\beta_0)
\xrightarrow{p}0.
\]
Let $\cH_n$ be the event on which the two bounds on the observation
maxima above and the following inequalities hold:
\[
\begin{aligned}
\sup_{\beta\in\cN}
\left\|\frac1n\sum_{i=1}^n\nabla^2\ell_i(\beta)
-\E\nabla^2\ell_i(\beta)\right\|_2&\le\frac{c_0}{4},\\
\left\|\frac1n\nabla L_\varnothing(\beta_0)\right\|_2
&\le\frac{c_0\rho_0}{32}.
\end{aligned}
\]
Then $\delta_n:=\Pp(\cH_n^c)=o(1)$, and the Hessian bounds above hold on
$\cH_n$. Since $MG_n=o(n)$, on this event, for all sufficiently large $n$
and every $|A|\le M$,
\[
\|\nabla L_A(\beta_0)\|_2
\le\|\nabla L_\varnothing(\beta_0)\|_2+MG_n
\le c_0n\rho_0/16.
\]
For $\|h\|_2=\rho_0$, Taylor's theorem therefore gives
\[
L_A(\beta_0+h)-L_A(\beta_0)
\ge-\frac{c_0n\rho_0^2}{16}
+\frac{\kappa_n\rho_0^2}{2}>0.
\]
If a retained minimizer were on or outside this ball, the segment from
$\beta_0$ to that minimizer would meet its boundary at a point with
objective value at most $L_A(\beta_0)$ by convexity, contradicting the
strict boundary inequality. Thus
$\widehat\beta_A\in B(\beta_0,\rho_0)\subset\cN$ for every $|A|\le M$,
including $A=\varnothing$.
This proves all conclusions of
Lemma~\ref{lem:uniform-event} on one event for all deletion sets.

\par\medskip\noindent\textbf{The four residual and increment bounds.}
Fix $D\in\cH_n$ from Lemma~\ref{lem:uniform-event}. If
$C_{L,n}=0$, all remainders and bounds vanish. Otherwise the following
argument applies.

\textit{Residual bounds.}
To bound $e_S$, write the retained first-order equation as a Taylor
remainder at $\widehat\beta$.
Let $v_S=\widehat\beta_S-\widehat\beta$.
The segment from $\widehat\beta$ to $\widehat\beta_S$ lies in $\cN$
by convexity.
Since $\nabla L_S(\widehat\beta)=-g_S$ and
$\nabla L_S(\widehat\beta_S)=0$, the Hessian lower bound gives
\[
\begin{aligned}
\kappa_n\|v_S\|_2^2
&\le\int_0^1 v_S^\top\nabla^2L_S(\widehat\beta+tv_S)v_S\,dt\\
&=v_S^\top\{\nabla L_S(\widehat\beta+v_S)
-\nabla L_S(\widehat\beta)\}
=v_S^\top g_S
\le\|v_S\|_2\|g_S\|_2.
\end{aligned}
\]
Thus
\begin{equation}
\|v_S\|_2\le\frac{\|g_S\|_2}{\kappa_n}=\rho_S.
\label{eq:fit-displacement}
\end{equation}
Define the gradient remainder
$F_S(v)=\nabla L_S(\widehat\beta+v)-\nabla L_S(\widehat\beta)-H_Sv$.
Taylor's integral formula and the Hessian Lipschitz bound give
\[
F_S(v)
=\int_0^1\{\nabla^2L_S(\widehat\beta+uv)-H_S\}v\,du,
\]
and
\begin{equation}
\|F_S(v_S)\|_2
\le C_{L,n}\|v_S\|_2^2\int_0^1u\,du
\le\frac{C_{L,n}}2\rho_S^2.
\label{eq:gradient-remainder-bound}
\end{equation}
The retained first-order equation is
$0=-g_S+H_Sv_S+F_S(v_S)$.
Since $\widetilde\beta_S=\widehat\beta+H_S^{-1}g_S$,
$e_S=H_S^{-1}g_S-v_S=H_S^{-1}F_S(v_S)$.
Thus
\[
\|e_S\|_2\le r^{\mathrm{set}}(S)
:=\frac{C_{L,n}}2\|H_S^{-1}\|_2\rho_S^2.
\]

Since $\rho_S\le |S|G_n/\kappa_n$ and $\|H_S^{-1}\|_2\le1/\kappa_n$,
\[
\|e_S\|_2\le r^{\mathrm{set}}(S)
\le\frac{C_{L,n}}2\frac1{\kappa_n}
\left(\frac{|S|G_n}{\kappa_n}\right)^2
=r^{\mathrm{cnt}}(|S|)
:=\frac{C_{L,n}G_n^2|S|^2}{2\kappa_n^3}.
\]

\textit{Set-based increments.}
Fix $A\subsetneq B$ and $q=|B\setminus A|$.
To bound $\|e_B-e_A\|_2$, we use the difference of the two
remainder equations and the residual bounds.
We first bound $\|v_B-v_A\|_2$ to control the difference between the
remainders at the two retained fits.
Since $\nabla L_B(\widehat\beta_B)=0$ and
$\nabla L_B(\widehat\beta_A)
=-\sum_{i\in B\setminus A}\nabla\ell_i(\widehat\beta_A)$,
the Hessian lower bound gives
\[
\begin{aligned}
\kappa_n\|v_B-v_A\|_2^2
&\le\int_0^1(v_B-v_A)^\top
\nabla^2L_B\bigl(\widehat\beta_A+t(v_B-v_A)\bigr)
(v_B-v_A)\,dt\\
&=(v_B-v_A)^\top
\{\nabla L_B(\widehat\beta_B)-\nabla L_B(\widehat\beta_A)\}\\
&=(v_B-v_A)^\top
\sum_{i\in B\setminus A}\nabla\ell_i(\widehat\beta_A)\\
&\le\|v_B-v_A\|_2
\left\|\sum_{i\in B\setminus A}\nabla\ell_i(\widehat\beta_A)\right\|_2.
\end{aligned}
\]
The individual gradient bounds give
$\|\sum_{i\in B\setminus A}\nabla\ell_i(\widehat\beta_A)\|_2\le qG_n$.
Expanding each gradient at
$\widehat\beta$ also gives
\[
\sum_{i\in B\setminus A}\nabla\ell_i(\widehat\beta_A)
=g_B-g_A+
\int_0^1\sum_{i\in B\setminus A}
\nabla^2\ell_i(\widehat\beta+uv_A)v_A\,du
\]
whose norm is at most $\|g_B-g_A\|_2+qB_n\rho_A$.
Taking the smaller of these two bounds yields
\[
\|v_B-v_A\|_2\le\delta_{A,B}
:=\frac1{\kappa_n}
\min\{qG_n,\|g_B-g_A\|_2+qB_n\rho_A\}.
\]

Subtracting $H_Ae_A=F_A(v_A)$ from $H_Be_B=F_B(v_B)$ and
adding and subtracting $F_B(v_A)$ gives
\begin{equation}
H_B(e_B-e_A)
=F_B(v_B)-F_B(v_A)
-\{F_A(v_A)-F_B(v_A)\}+(H_A-H_B)e_A.
\label{eq:error-increment-decomposition}
\end{equation}
We bound the three terms on the right separately.

To bound $\|F_B(v_B)-F_B(v_A)\|_2$, integrate
$\nabla F_B(v)=\nabla^2L_B(\widehat\beta+v)-H_B$ along
$v=(1-u)v_A+uv_B$, $0\le u\le1$, to obtain
\[
F_B(v_B)-F_B(v_A)
=\int_0^1\nabla F_B\bigl((1-u)v_A+uv_B\bigr)(v_B-v_A)\,du.
\]
On this segment, $\|v\|_2\le(1-u)\rho_A+u\rho_B$
by \eqref{eq:fit-displacement}.
Since $H_B=\nabla^2L_B(\widehat\beta)$, the Hessian Lipschitz bound gives
\[
\begin{aligned}
\|\nabla F_B(v)\|_2
&=\|\nabla^2L_B(\widehat\beta+v)-\nabla^2L_B(\widehat\beta)\|_2\\
&\le C_{L,n}\|(\widehat\beta+v)-\widehat\beta\|_2
=C_{L,n}\|v\|_2.
\end{aligned}
\]
Thus
\[
\begin{aligned}
\|F_B(v_B)-F_B(v_A)\|_2
&\le\|v_B-v_A\|_2
\int_0^1\|\nabla F_B((1-u)v_A+uv_B)\|_2\,du\\
&\le C_{L,n}\|v_B-v_A\|_2
\int_0^1\{(1-u)\rho_A+u\rho_B\}\,du\\
&\le\frac{C_{L,n}}2(\rho_A+\rho_B)\delta_{A,B}.
\end{aligned}
\]

To bound $\|F_A(v_A)-F_B(v_A)\|_2$, use
$L_A-L_B=\sum_{i\in B\setminus A}\ell_i$ to obtain
\[
\begin{aligned}
F_A(v_A)-F_B(v_A)
&=\sum_{i\in B\setminus A}\int_0^1
\{\nabla^2\ell_i(\widehat\beta+uv_A)
-\nabla^2\ell_i(\widehat\beta)\}v_A\,du,\\
\|F_A(v_A)-F_B(v_A)\|_2
&\le\frac{qC_{\ell,n}}2\rho_A^2.
\end{aligned}
\]

For the third term, $e_A=H_A^{-1}F_A(v_A)$ and
\eqref{eq:gradient-remainder-bound} with $S=A$ give
\[
\begin{aligned}
\|H_B^{-1}(H_A-H_B)e_A\|_2
&\le\|H_B^{-1}(H_A-H_B)H_A^{-1}\|_2\|F_A(v_A)\|_2\\
&\le\|H_B^{-1}(H_A-H_B)H_A^{-1}\|_2
\frac{C_{L,n}}2\rho_A^2.
\end{aligned}
\]

Multiplying \eqref{eq:error-increment-decomposition} by $H_B^{-1}$ and applying the
triangle inequality yields $\|e_B-e_A\|_2\le U_{A,B}$, where
\[
U_{A,B}:=\|H_B^{-1}\|_2\frac{C_{L,n}}2(\rho_A+\rho_B)\delta_{A,B}
+\|H_B^{-1}\|_2\frac{qC_{\ell,n}}2\rho_A^2
+\|H_B^{-1}(H_A-H_B)H_A^{-1}\|_2\frac{C_{L,n}}2\rho_A^2.
\]
The residual bounds also give
\[
\begin{aligned}
\|e_B-e_A\|_2
&\le\|e_A\|_2+\|e_B\|_2\\
&\le\frac{C_{L,n}}2
\left(\|H_A^{-1}\|_2\rho_A^2+\|H_B^{-1}\|_2\rho_B^2\right)
=r^{\mathrm{set}}(A)+r^{\mathrm{set}}(B).
\end{aligned}
\]
Taking the smaller of these two bounds gives
\[
\|e_B-e_A\|_2
\le d^{\mathrm{set}}(A,B)
:=\min\left\{U_{A,B},\frac{C_{L,n}}2
\left(\|H_A^{-1}\|_2\rho_A^2+\|H_B^{-1}\|_2\rho_B^2\right)\right\}.
\]

\textit{Count-based increments.}
To replace the set-dependent quantities in $U_{A,B}$ by deletion counts,
we first bound $\|H_B^{-1}-H_A^{-1}\|_2$.
Write $k=|A|$, $m=|B|$, and $q=m-k$. Convexity gives
$H_B\preceq H_A\preceq H_B+qB_nI_p$. Monotonicity of the matrix inverse
and $H_B\succeq\kappa_nI_p$ imply
\[
\begin{aligned}
0\preceq H_B^{-1}-H_A^{-1}
&\preceq H_B^{-1}-(H_B+qB_nI_p)^{-1},\\
\|H_B^{-1}(H_A-H_B)H_A^{-1}\|_2
&=\|H_B^{-1}-H_A^{-1}\|_2
\le\frac{qB_n}{\kappa_n(\kappa_n+qB_n)}.
\end{aligned}
\]
For the last inequality, diagonalizing $H_B$ gives eigenvalues
$qB_n/\{t(t+qB_n)\}$ for
$H_B^{-1}-(H_B+qB_nI_p)^{-1}$, with $t\ge\kappa_n$.
This expression decreases with $t$, so its maximum is at most the
displayed bound.
Substitute this bound, $\|H_B^{-1}\|_2\le1/\kappa_n$,
$\rho_A\le kG_n/\kappa_n$, $\rho_B\le mG_n/\kappa_n$, and
$\delta_{A,B}\le qG_n/\kappa_n$ into $U_{A,B}$ to obtain
\[
\begin{aligned}
\|e_B-e_A\|_2
&\le d^{\mathrm{set}}(A,B)\le U_{A,B}\\
&\le\frac{C_{L,n}G_n^2q(k+m)}{2\kappa_n^3}
+\frac{qC_{\ell,n}k^2G_n^2}{2\kappa_n^3}
+\frac{qB_nC_{L,n}k^2G_n^2}{2\kappa_n^3(\kappa_n+qB_n)}\\
&=d^{\mathrm{cnt}}(k,m),
\end{aligned}
\]
where
\[
d^{\mathrm{cnt}}(k,m):=\frac{G_n^2(m-k)}{2\kappa_n^3}
\left[C_{L,n}(k+m)+k^2\left(C_{\ell,n}
+\frac{B_nC_{L,n}}{\kappa_n+(m-k)B_n}\right)\right].
\]
Since $k^2\le M(k+m)$,
\[
\frac{k^2}{k+m}\left(\frac{C_{\ell,n}}{C_{L,n}}
+\frac{B_n}{\kappa_n+(m-k)B_n}\right)
\le\frac{MC_{\ell,n}}{C_{L,n}}+\frac{MB_n}{\kappa_n}=o(1)
\]
by Lemma~\ref{lem:uniform-event} and the cap condition. Hence, uniformly
over $0\le k<m\le M$,
\[
d^{\mathrm{cnt}}(k,m)
=\frac{C_{L,n}G_n^2(m-k)(k+m)}{2\kappa_n^3}\{1+o(1)\}
=O\!\left(\frac{(m-k)(k+m)\polylog(n)}{n^2}\right).
\]
The same constant bounds give
$r^{\mathrm{cnt}}(m)=O(m^2\polylog(n)/n^2)$.

\par\medskip\noindent\textbf{Relations between residual and increment bounds.}
For $A=\varnothing$, $\rho_A=0$ and $\delta_{A,B}=\rho_B$, giving
$d^{\mathrm{set}}(\varnothing,B)=r^{\mathrm{set}}(B)$.
To prove the second relation in \eqref{eq:bound-relations}, we bound
$|\rho_B^2-\rho_A^2|$ and $\|H_B^{-1}-H_A^{-1}\|_2$.
Since $\norm{g_B-g_A}\le qG_n$, the definition of $\delta_{A,B}$ gives
$|\rho_B-\rho_A|\le\norm{g_B-g_A}/\kappa_n\le \delta_{A,B}$.
Hence $|\rho_B^2-\rho_A^2|=(\rho_A+\rho_B)|\rho_B-\rho_A|
\le(\rho_A+\rho_B)\delta_{A,B}$.
Since $H_B^{-1}-H_A^{-1}=H_B^{-1}(H_A-H_B)H_A^{-1}$,
\begin{align*}
|r^{\mathrm{set}}(B)-r^{\mathrm{set}}(A)|
&\le\frac{C_{L,n}}2\left[
\|H_B^{-1}\|_2|\rho_B^2-\rho_A^2|
+\rho_A^2\big|\|H_B^{-1}\|_2-\|H_A^{-1}\|_2\big|\right]\\
&\le\frac{C_{L,n}}2\left[
\|H_B^{-1}\|_2(\rho_A+\rho_B)\delta_{A,B}
+\rho_A^2\|H_B^{-1}(H_A-H_B)H_A^{-1}\|_2\right]
\\&\le U_{A,B}.
\end{align*}
Also $|r^{\mathrm{set}}(B)-r^{\mathrm{set}}(A)|\le r^{\mathrm{set}}(A)+r^{\mathrm{set}}(B)$, proving the claimed inequality.
\end{proof}

\subsection{Proof of Proposition~\ref{prop:allocation-minimax}}
\begin{proof}
Fix $D$, $M$, and $a\in\{\mathrm{cnt},\mathrm{set}\}$. Use the corresponding
bounds $r_t,d_t$ and path maxima $R^a,D^a$ from the main text.
For each mechanism, we check that the proposed allocation attains its
claimed value, then bound every feasible allocation from below on each
path. Maximizing these lower bounds over paths proves optimality.
If the relevant path maximum is zero, all corresponding bounds vanish.
Zero budgets and zero noise attain value zero. We therefore assume the
relevant path maximum is positive in the calculations below.

\textit{IND.} For the upper bound, the proposed budgets satisfy
$\sum_{t=1}^T\mu_t^2 ={(\mu^2/R^a)}\sum_{t=1}^T r_t^2 \le\mu^2$
on every $P\in\mathcal C$.
The corresponding variance is $\tau_t^2=R^a/\mu^2$, so the largest
variance on every path is $R^a/\mu^2$.
For the lower bound, consider any allocation satisfying the IND constraint in
\eqref{eq:certificate-budgets}. On each path,
\[
\sum_{t=1}^T r_t^2
\le\left(\max_{1\le t\le T}\tau_t^2\right)
\sum_{t=1}^T\frac{r_t^2}{\tau_t^2}
\le\mu^2\max_{1\le t\le T}\tau_t^2.
\]
Taking the supremum over $P\in\mathcal C$ gives the lower bound
$R^a/\mu^2$, which the proposed budgets attain.

\textit{RW.} For the upper bound, the proposed budgets satisfy
$\sum_{t=1}^T\mu_t^2 ={(\mu^2/\sqrt{D^a})}\sum_{t=1}^T d_t \le\mu^2$
on every $P\in\mathcal C$.
They give $\sigma_t^2=d_t\sqrt{D^a}/\mu^2$. Since the increment
variances are nonnegative,
\[
\max_{1\le t\le T}\sum_{s=1}^t\sigma_s^2
=\sum_{t=1}^T\sigma_t^2
=\frac{\sqrt{D^a}}{\mu^2}\sum_{t=1}^T d_t
\le\frac{D^a}{\mu^2},
\]
with equality on a path maximizing $\sum_t d_t$.
For the lower bound, consider any allocation satisfying the RW constraint in
\eqref{eq:certificate-budgets}. Cauchy--Schwarz gives
\[
\left(\sum_{t=1}^T d_t\right)^2
\le\left(\sum_{t=1}^T\frac{d_t^2}{\sigma_t^2}\right)
\left(\sum_{t=1}^T\sigma_t^2\right)
\le\mu^2\sum_{t=1}^T\sigma_t^2.
\]
Thus its worst-case accumulated variance is at least $D^a/\mu^2$,
which the proposed budgets attain. Both arguments for the lower bounds hold for
every realized history, so they also apply to the stated
$\mathcal F_t$-measurable allocations. This proves the proposition for
both choices of bounds. Terms with zero bounds contribute zero cost.
\end{proof}

The count-based residual maximum has the explicit form
$R^{\mathrm{cnt}}=\sum_{m=1}^M r^{\mathrm{cnt}}(m)^2$:
the positive counts on a path are distinct, and the full singleton path
visits all of them. Supplementary Section~\ref{app:computable-caps} shows that splitting batches into
single-record deletions cannot decrease the sum of count-based increment bounds.
Thus the full singleton path also attains the count-based increment maximum.

\subsection{Proof of Theorem~\ref{thm:transcript}}\label{app:allocation-proof}
\begin{proof}
The result follows directly from Theorem~\ref{thm:bounds},
Proposition~\ref{prop:allocation-minimax}, and adaptive GDP composition.
\end{proof}

\subsection{Proof of Corollary~\ref{cor:computable-calibration}}\label{app:computable-proof}
\begin{proof}
Fix $D\in\cH_n$. If $C_{L,n}=0$, all increment bounds vanish and no
noise is needed. We therefore assume $C_{L,n}>0$. We use
$\|g_A\|_2\le\sum_{i\in A}\|\nabla\ell_i(\widehat\beta)\|_2$
to bound $\sum_{t=1}^{T}d^{\mathrm{set}}(S_{t-1},S_t)$
along every allowed deletion path. The resulting
leading term telescopes, while the remaining terms sum over at most $M$
deleted records.
Recall
$\alpha_n=C_{L,n}/(2\kappa_n^3)$, $\eta_n=MC_{\ell,n}/C_{L,n}+3MB_n/\kappa_n$,
and $W_M$ from Supplementary Section~\ref{app:computable-caps}.
For a deletion set $A$, put $w_A=\sum_{i\in A}\|\nabla\ell_i(\widehat\beta)\|_2$.
For an edge $A\subset B$, let $q=|B\setminus A|$.
The bounds $\|H_A^{-1}\|_2,\|H_B^{-1}\|_2\le1/\kappa_n$ and
$\|H_A-H_B\|_2\le qB_n$, together with
$\delta_{A,B}\le(\|g_B-g_A\|_2+qB_n\rho_A)/\kappa_n$, give
\[
\begin{aligned}
d^{\mathrm{set}}(A,B)\le\alpha_n\biggl[&
(\norm{g_A}+\norm{g_B})(w_B-w_A)\\
&+\frac{qB_n}{\kappa_n}(\norm{g_A}+\norm{g_B})\norm{g_A}
+q\left(\frac{C_{\ell,n}}{C_{L,n}}+\frac{B_n}{\kappa_n}\right)\norm{g_A}^2\biggr].
\end{aligned}
\]
Since $\norm{g_A}\le w_A\le W_M$, the first term is at most
$(w_A+w_B)(w_B-w_A)=w_B^2-w_A^2$.
Since $\norm{g_A},\norm{g_B}\le W_M$, the other two terms satisfy
\[
\begin{aligned}
\frac{qB_n}{\kappa_n}(\norm{g_A}+\norm{g_B})\norm{g_A}
&\le\frac{2qB_n}{\kappa_n}W_M^2,\\
q\left(\frac{C_{\ell,n}}{C_{L,n}}+\frac{B_n}{\kappa_n}\right)\norm{g_A}^2
&\le q\left(\frac{C_{\ell,n}}{C_{L,n}}+\frac{B_n}{\kappa_n}\right)W_M^2.
\end{aligned}
\]
Their sum is at most
$qW_M^2\{C_{\ell,n}/C_{L,n}+3B_n/\kappa_n\}$.
Along a path $P=(S_0,\ldots,S_T)$, the first bound telescopes because
$\sum_{(A,B)\in P}(w_B^2-w_A^2)=w_{S_T}^2$ and $w_{S_0}=0$.
Also $\sum_{(A,B)\in P}q=|S_T|\le M$. Hence
\[
\sum_{(A,B)\in P}d^{\mathrm{set}}(A,B)
\le\alpha_n\left[w_{S_T}^2+
MW_M^2\left(\frac{C_{\ell,n}}{C_{L,n}}+\frac{3B_n}{\kappa_n}\right)\right]
\le\alpha_n W_M^2(1+\eta_n).
\]
Taking the maximum over all allowed paths gives
\begin{equation}
\sqrt{D^{\mathrm{set}}}
=\max_{P\in\mathcal C}\sum_{(A,B)\in P}d^{\mathrm{set}}(A,B)
\le\alpha_n W_M^2(1+\eta_n).
\label{eq:set-path-bound}
\end{equation}
The transition comparison $d^{\mathrm{set}}(A,B)\le
d^{\mathrm{cnt}}(|A|,|B|)$ also gives
$D^{\mathrm{set}}\le D^{\mathrm{cnt}}$.
Taking the minimum of these two path bounds proves
\[
D^{\mathrm{set}}\le D^{\mathrm{com}}\le D^{\mathrm{cnt}},
\qquad
\sqrt{D^{\mathrm{com}}}\le\alpha_n W_M^2(1+\eta_n).
\]
If $D^{\mathrm{com}}=0$, all set-based increment bounds vanish on every
path, so no noise is needed for transcript certification. We therefore
assume $D^{\mathrm{com}}>0$. The proposed RW variances
satisfy, on every allowed path,
\[
\sum_{t=1}^T
\frac{d^{\mathrm{set}}(S_{t-1},S_t)^2}{(\sigma_t^{\mathrm{com}})^2}
=\frac{\mu^2}{\sqrt{D^{\mathrm{com}}}}
\sum_{t=1}^T d^{\mathrm{set}}(S_{t-1},S_t)
\le\mu^2\sqrt{\frac{D^{\mathrm{set}}}{D^{\mathrm{com}}}}
\le\mu^2.
\]
When $d^{\mathrm{set}}(S_{t-1},S_t)=0$, the corresponding summand
in the GDP budget is defined to be zero.
Adaptive GDP composition, as used in the proof of Theorem~\ref{thm:transcript},
gives the same $(\delta_n,\mu)$ transcript guarantee.

For the same deletions, \eqref{eq:cert} and
$D^{\mathrm{com}}\le D^{\mathrm{cnt}}$ give
\[
(\sigma_t^{\mathrm{com}})^2
=\frac{d^{\mathrm{set}}(S_{t-1},S_t)\sqrt{D^{\mathrm{com}}}}{\mu^2}
\le\frac{d^{\mathrm{cnt}}(m_{t-1},m_t)\sqrt{D^{\mathrm{cnt}}}}{\mu^2}
=(\sigma_t^{\mathrm{cnt}})^2.
\]
Consequently, for every $u\le T$,
$\sum_{t=1}^u(\sigma_t^{\mathrm{com}})^2
\le\sum_{t=1}^u(\sigma_t^{\mathrm{cnt}})^2$.
\end{proof}

\subsection{Proof of Theorem~\ref{thm:variance}}\label{app:ratio-proof}
\label{app:count-variance-proof}\label{app:set-variance-proof}
\begin{proof}
We first evaluate $D^{\mathrm{cnt}}/R^{\mathrm{cnt}}$ from the explicit
path maxima. We then bound $D^{\mathrm{set}}/R^{\mathrm{set}}$ by comparing
both quantities with $W_M^4$. These bounds also give the comparison
between computable and set-based RW.

\par\medskip\noindent\textbf{Count-based bounds.}
The exact path maximum and the sum of fourth powers determine the variance
ratio. If $R^{\mathrm{cnt}}=0$, all count-based bounds and variances
vanish. The stated ratios have positive denominators. We therefore
assume $R^{\mathrm{cnt}}>0$ and write
$a_n=C_{\ell,n}/C_{L,n}+B_n/(\kappa_n+B_n)$.
The count-based residual bound in Supplementary Section~\ref{app:certificates} gives
$r^{\mathrm{cnt}}(m)=r^{\mathrm{cnt}}(1)m^2$.
Substituting $k=m-1$ into the count-based increment bound gives
\[
\begin{aligned}
d^{\mathrm{cnt}}(m-1,m)
&=\frac{G_n^2}{2\kappa_n^3}
\left[C_{L,n}(2m-1)+(m-1)^2
\left(C_{\ell,n}+\frac{B_nC_{L,n}}{\kappa_n+B_n}\right)\right]\\
&=r^{\mathrm{cnt}}(1)\{2m-1+a_n(m-1)^2\}.
\end{aligned}
\]
The full singleton path attains both count-based maxima, as stated after
Theorem~\ref{thm:transcript}. Supplementary Section~\ref{app:count-path-maximum}
proves this for the increment sum. Thus
\[
\begin{aligned}
\sqrt{D^{\mathrm{cnt}}}
&=\sum_{m=1}^M d^{\mathrm{cnt}}(m-1,m)
=r^{\mathrm{cnt}}(1)\left\{M^2+\frac{a_nM(M-1)(2M-1)}6\right\},\\
R^{\mathrm{cnt}}
&=\sum_{m=1}^M r^{\mathrm{cnt}}(m)^2
=r^{\mathrm{cnt}}(1)^2\sum_{m=1}^M m^4
\end{aligned}
\]

Consequently,
\[
\frac{V_{\rw}^{\mathrm{cnt}}}{V_{\ind}^{\mathrm{cnt}}}
=\frac{D^{\mathrm{cnt}}}{R^{\mathrm{cnt}}}
=\frac{\left\{M^2+a_nM(M-1)(2M-1)/6\right\}^2}{\sum_{m=1}^M m^4}.
\]
Lemma~\ref{lem:uniform-event} gives
$MC_{\ell,n}/C_{L,n}\to0$, $B_n=O(\polylog(n))$, and
$\kappa_n\ge c_0n/4$. Since $M=o\{n/\polylog(n)\}$,
\[
0\le Ma_n
=\frac{MC_{\ell,n}}{C_{L,n}}
+\frac{MB_n}{\kappa_n+B_n}
\le\frac{MC_{\ell,n}}{C_{L,n}}+\frac{4MB_n}{c_0n}
\longrightarrow0.
\]
Since $(M-1)(2M-1)/(6M)\le M/3$ and
$M^5/5\le\sum_{m=1}^M m^4\le M^5$,
$D^{\mathrm{cnt}}/R^{\mathrm{cnt}}=\Theta(M^{-1})$.
If $M\to\infty$, substitute
$\sum_{m=1}^M m^4=M^5\{1+O(M^{-1})\}/5$ to obtain
\[
\frac{V_{\rw}^{\mathrm{com}}}{V_{\ind}^{\mathrm{cnt}}}
\le\frac{V_{\rw}^{\mathrm{cnt}}}{V_{\ind}^{\mathrm{cnt}}}
=\frac{5+o(1)}{M},
\]
where the inequality follows from
Corollary~\ref{cor:computable-calibration}.

\par\medskip\noindent\textbf{Set-based bounds.}
We bound $D^{\mathrm{set}}$ from above and $R^{\mathrm{set}}$ from below
to show that $D^{\mathrm{set}}/R^{\mathrm{set}}=O(M^{-1})$. The upper bound on $D^{\mathrm{set}}$ comes from summing score norms.
To obtain a lower bound on $R^{\mathrm{set}}$, we construct a deletion path
with at least $M/4$ sets satisfying
$\|g_{S_t}\|_2\ge\frac12\max_{|S|\le M}\|g_S\|_2$.

To bound $D^{\mathrm{set}}/R^{\mathrm{set}}$, assume $n\ge2M$ and $M\ge2$.
We also assume $C_{L,n}>0$ and $\nabla\ell_i(\widehat\beta)\ne0$ for at least one $i$.
Otherwise, $R^{\mathrm{set}}=D^{\mathrm{set}}=0$, which is the trivial case.
To compare residual bounds with squared score norms, we use the ratio
$\omega_n=\lambda_{\min}(H_\varnothing)/\kappa_n$.
Recall $\eta_n=MC_{\ell,n}/C_{L,n}+3MB_n/\kappa_n$.
Let $\gamma_p=\mathbb E|U_1|$, where $U$ is
uniform on the unit sphere in $\mathbb R^p$, with $\gamma_1=1$.

Since $p$ is fixed, $\gamma_p$ is a positive constant independent of $n$.
We first prove the bound
\[
\frac1M\le\frac{D^{\mathrm{set}}}{R^{\mathrm{set}}}
\le\frac{1024\omega_n^2(1+\eta_n)^2}{\gamma_p^4M}.
\]

Recall that $W_M$ is the sum of the $M$ largest values among
$\{\|\nabla\ell_i(\widehat\beta)\|_2:i=1,\ldots,n\}$,
and $\alpha_n=C_{L,n}/(2\kappa_n^3)$.
The Hessian lower bound and $\nabla^2\ell_i(\widehat\beta)\succeq0$ give
$\kappa_nI_p\preceq H_A\preceq H_\varnothing$, so
\[
\frac1{\lambda_{\min}(H_\varnothing)}
\le\|H_A^{-1}\|_2
=\frac1{\lambda_{\min}(H_A)}
\le\frac1{\kappa_n}.
\]
Using
\[
r^{\mathrm{set}}(A)
=\frac{C_{L,n}}{2\kappa_n^2}\|H_A^{-1}\|_2\|g_A\|_2^2
=\alpha_n\kappa_n\|H_A^{-1}\|_2\|g_A\|_2^2
\]
and $\omega_n=\lambda_{\min}(H_\varnothing)/\kappa_n$ gives
\[
\frac{\alpha_n}{\omega_n}\|g_A\|_2^2
\le r^{\mathrm{set}}(A)\le\alpha_n\|g_A\|_2^2.
\]

\textit{An upper bound for $D^{\mathrm{set}}$.}
By \eqref{eq:set-path-bound}, we have
\[
\sqrt{D^{\mathrm{set}}}\le\alpha_n W_M^2(1+\eta_n).
\]

\textit{A lower bound for $R^{\mathrm{set}}$.}
Let $s_*:=\max_{|S|\le M}\|g_S\|_2$, and let $J$ index the $M$ largest
values of $\|\nabla\ell_i(\widehat\beta)\|_2$.
By rotational invariance, $U^\top a$ has the same distribution as
$\|a\|_2U_1$ for every $a\in\mathbb R^p$, so
$\mathbb E_U|U^\top a|=\gamma_p\|a\|_2$. Hence
\[
\mathbb E_U\sum_{i\in J}|U^\top\nabla\ell_i(\widehat\beta)|
=\gamma_p\sum_{i\in J}\|\nabla\ell_i(\widehat\beta)\|_2
=\gamma_pW_M.
\]
There is therefore a unit direction $u$ satisfying
$\sum_{i\in J}|u^\top\nabla\ell_i(\widehat\beta)|\ge\gamma_pW_M$.
Partition $J$ by the sign of these projections and take $S$ to be the
group with the larger absolute projection sum. Since the projections in
$S$ have the same sign, $|u^\top g_S|\ge\gamma_pW_M/2$.
Using $|S|\le M$ gives
\[
s_*\ge\|g_S\|_2\ge|u^\top g_S|\ge\gamma_pW_M/2.
\]
Taking maxima in the residual lower bound gives
\[
\begin{aligned}
r_*^{\mathrm{set}}
&=\max_{|A|\le M}r^{\mathrm{set}}(A)
\ge\frac{\alpha_n}{\omega_n}\max_{|A|\le M}\|g_A\|_2^2\\
&=\frac{\alpha_n}{\omega_n}s_*^2
\ge\frac{\alpha_n\gamma_p^2W_M^2}{4\omega_n}.
\end{aligned}
\]

Choose a set $S_*$ satisfying $|S_*|\le M$ and $\|g_{S_*}\|_2=s_*$,
and write $k=|S_*|$. Since $s_*>0$, set $u=g_{S_*}/s_*$. Then
$\|u\|_2=\|g_{S_*}\|_2/s_*=1$ and $u^\top g_{S_*}=s_*$.
For every $i\in S_*$, $u^\top\nabla\ell_i(\widehat\beta)\ge0$:
otherwise, removing $i$ would give
$\|g_{S_*\setminus\{i\}}\|_2\ge u^\top g_{S_*\setminus\{i\}}
=s_*-u^\top\nabla\ell_i(\widehat\beta)>s_*$, contradicting the definition of $s_*$.
We construct a deletion path from $S_0=\varnothing$ with at least $M/4$
states satisfying $\|g_{S_t}\|_2\ge s_*/2$.

If $k\ge M/2$, order $S_*=\{i_1,\ldots,i_k\}$ by decreasing
$u^\top\nabla\ell_{i_j}(\widehat\beta)$, and delete the records in that order,
so $S_j=\{i_1,\ldots,i_j\}$ for $1\le j\le k$.
The first $j$ projections have average at least the average $s_*/k$
over all $k$ records. Hence
\[
u^\top g_{S_j}
=\sum_{\ell=1}^j u^\top\nabla\ell_{i_\ell}(\widehat\beta)
\ge\frac jk s_*\ge\frac{s_*}{2},
\qquad \lceil k/2\rceil\le j\le k.
\]
There are $\lfloor k/2\rfloor+1\ge k/2\ge M/4$ such states, all satisfying
$\|g_{S_j}\|_2\ge u^\top g_{S_j}\ge s_*/2$.
The path respects the cap because $|S_j|\le k\le M$.

If $k<M/2$, first delete the records in $S_*$ one at a time in any order,
so $S_k=S_*$. We then append $L=\lfloor M/2\rfloor$ remaining records.
Order the remaining records as $i_1,\ldots,i_{n-k}$ so that
$b_j:=u^\top\nabla\ell_{i_j}(\widehat\beta)$ satisfies
$b_1\ge\cdots\ge b_{n-k}$, and append the first $L$ of them.
Thus $S_{k+\ell}=S_*\cup\{i_1,\ldots,i_\ell\}$ for $1\le\ell\le L$,
and all these sets have size at most $k+L\le M$.
If $b_L<-s_*/M$, then every $b_j$ with $j\ge L$ is also below $-s_*/M$.
Since $n\ge2M$, there are at least
$n-k-L+1\ge n-M+1\ge M$ such records.
Choose $M$ of them as $C$. Then $u^\top g_C<-s_*$ and
$\|g_C\|_2\ge|u^\top g_C|>s_*$, contradicting the definition of $s_*$.
Therefore $b_j\ge b_L\ge-s_*/M$ for $1\le j\le L$, and
\[
u^\top g_{S_{k+\ell}}
=s_*+\sum_{j=1}^{\ell}b_j
\ge s_*\left(1-\frac{\ell}{M}\right)
\ge\frac{s_*}{2},
\qquad 0\le\ell\le L.
\]
These $L+1\ge M/4$ states satisfy
$\|g_{S_{k+\ell}}\|_2\ge u^\top g_{S_{k+\ell}}\ge s_*/2$.

For either constructed path $P=(S_0,\ldots,S_T)$, each state with
$\|g_{S_t}\|_2\ge s_*/2$ satisfies
\[
r^{\mathrm{set}}(S_t)
\ge\frac{\alpha_n}{\omega_n}\|g_{S_t}\|_2^2
\ge\frac{\alpha_ns_*^2}{4\omega_n}.
\]
There are at least $M/4$ such states, so
\[
R^{\mathrm{set}}
\ge\sum_{t=1}^T\{r^{\mathrm{set}}(S_t)\}^2
\ge\frac M4\left(\frac{\alpha_ns_*^2}{4\omega_n}\right)^2
=\frac{\alpha_n^2Ms_*^4}{64\omega_n^2}.
\]
Using $s_*\ge\gamma_pW_M/2$ then gives
\[
R^{\mathrm{set}}
\ge\frac{\alpha_n^2M}{64\omega_n^2}
\left(\frac{\gamma_pW_M}{2}\right)^4
=\frac{\alpha_n^2\gamma_p^4MW_M^4}{1024\omega_n^2}.
\]
For the upper bound on $D^{\mathrm{set}}/R^{\mathrm{set}}$, combine this inequality with
$\sqrt{D^{\mathrm{set}}}\le\alpha_n W_M^2(1+\eta_n)$ to obtain
\[
\frac{D^{\mathrm{set}}}{R^{\mathrm{set}}}
\le\frac{\alpha_n^2W_M^4(1+\eta_n)^2}
{\alpha_n^2\gamma_p^4MW_M^4/(1024\omega_n^2)}
=\frac{1024\omega_n^2(1+\eta_n)^2}{\gamma_p^4M}.
\]

For the lower bound on $D^{\mathrm{set}}/R^{\mathrm{set}}$, the identity
$d^{\mathrm{set}}(\varnothing,S)=r^{\mathrm{set}}(S)$ gives
$\sqrt{D^{\mathrm{set}}}\ge r_*^{\mathrm{set}}$. Also,
$R^{\mathrm{set}}\le M(r_*^{\mathrm{set}})^2$, so
\[
\frac{D^{\mathrm{set}}}{R^{\mathrm{set}}}
\ge\frac{(r_*^{\mathrm{set}})^2}{M(r_*^{\mathrm{set}})^2}
=\frac1M.
\]
Thus
\[
\frac1M\le\frac{D^{\mathrm{set}}}{R^{\mathrm{set}}}
\le\frac{1024\omega_n^2(1+\eta_n)^2}{\gamma_p^4M}.
\]

\textit{Uniformity on $\cH_n$.}
To use the finite-sample comparison in the theorem, it remains to show
$\omega_n=O(1)$ and $\eta_n=o(1)$. We first bound $\|H_\varnothing\|_2$.
Assumptions~\ref{ass:loss}--\ref{ass:identification} give
\[
\sup_{\beta\in\cN}\|\E\nabla^2\ell_i(\beta)\|_2<\infty,
\qquad
\sup_{\beta\in\cN}\|\nabla^2\mathcal P(\beta)\|_2<\infty.
\]
On $\cH_n$, the sample Hessian differs from its expectation by at most
$c_0n/4$, as established in the proof of Lemma~\ref{lem:uniform-event}.
Since $\widehat\beta\in\cN$ and $\lambda=\Theta(1)$,
\[
\|H_\varnothing\|_2
\le n\left(\sup_{\beta\in\cN}\|\E\nabla^2\ell_i(\beta)\|_2+\frac{c_0}{4}\right)
+\lambda\sup_{\beta\in\cN}\|\nabla^2\mathcal P(\beta)\|_2
\le C_Hn,
\]
where $C_H$ is independent of $n$.
Since $\kappa_n\ge c_0n/4$,
\[
0<\omega_n
\le\frac{\|H_\varnothing\|_2}{\kappa_n}
\le\frac{C_Hn}{c_0n/4}
=\frac{4C_H}{c_0}
=O(1).
\]
Recall
\[
\eta_n=\frac{MC_{\ell,n}}{C_{L,n}}+\frac{3MB_n}{\kappa_n}.
\]
Lemma~\ref{lem:uniform-event} gives $MC_{\ell,n}/C_{L,n}\to0$ and
$B_n=O(\polylog(n))$. Together with $M=o\{n/\polylog(n)\}$ and
$\kappa_n\ge c_0n/4$, these imply
\[
0\le\eta_n
\le\frac{MC_{\ell,n}}{C_{L,n}}
+\frac{12MB_n}{c_0n}
\longrightarrow0.
\]
Substituting into the finite-sample ratio bounds gives
\[
\frac{V_{\rw}^{\mathrm{set}}}{V_{\ind}^{\mathrm{set}}}
=\frac{D^{\mathrm{set}}}{R^{\mathrm{set}}}
=\Theta(M^{-1})
\]
uniformly on $\cH_n$.

\par\medskip\noindent\textbf{Computable RW.}
If $D^{\mathrm{set}}=0$, both set-based and computable RW require zero
noise. We therefore assume $D^{\mathrm{set}}>0$ when comparing their
ratio.
The accumulated variance uses the realized increment bounds, so maximizing
over deletion paths gives
\[
V_{\rw}^{\mathrm{com}}
=\frac{\sqrt{D^{\mathrm{com}}D^{\mathrm{set}}}}{\mu^2}
\le\frac{D^{\mathrm{com}}}{\mu^2}.
\]
When $D^{\mathrm{set}}>0$, the bound proved in
{\color{black}Section~\ref{app:set-variance-proof}} gives
$\sqrt{D^{\mathrm{set}}}\ge\alpha_n\gamma_p^2W_M^2/(4\omega_n)$.
Together with $D^{\mathrm{set}}\le D^{\mathrm{com}}$ and
$\sqrt{D^{\mathrm{com}}}\le\alpha_n W_M^2(1+\eta_n)$, this yields
\[
1\le\frac{V_{\rw}^{\mathrm{com}}}{V_{\rw}^{\mathrm{set}}}
=\sqrt{\frac{D^{\mathrm{com}}}{D^{\mathrm{set}}}}
\le\frac{4\omega_n(1+\eta_n)}{\gamma_p^2}.
\]
Since $\omega_n=O(1)$ and $\eta_n=o(1)$ uniformly on $\cH_n$, and $p$ is fixed,
$V_{\rw}^{\mathrm{com}}/V_{\rw}^{\mathrm{set}}=\Theta(1)$.
\end{proof}

\subsection{Proof of Corollary~\ref{cor:single-release}}
\begin{proof}
If $r^{\mathrm{cnt}}(1)=0$, all compared variances vanish and the ratios
are not stated. We therefore assume $r^{\mathrm{cnt}}(1)>0$ and retain
$a_n$ from that proof.
Since $V_{\mathrm{single}}=r^{\mathrm{cnt}}(1)^2M^4/\mu^2$, the exact
path maximum and IND sum give
\[
\frac{V_{\rw}^{\mathrm{cnt}}}{V_{\mathrm{single}}}
=\left(1+\frac{a_n(M-1)(2M-1)}{6M}\right)^2,
\qquad
\frac{V_{\ind}^{\mathrm{cnt}}}{V_{\mathrm{single}}}
=\frac{\sum_{m=1}^M m^4}{M^4}.
\]
The cap condition gives $Ma_n\to0$, and $(M-1)(2M-1)/(6M)\le M/3$. Hence
\[
\frac{V_{\rw}^{\mathrm{com}}}{V_{\mathrm{single}}}
\le\frac{V_{\rw}^{\mathrm{cnt}}}{V_{\mathrm{single}}}=1+o(1).
\]
If $M\to\infty$, then
$\sum_{m=1}^M m^4=M^5\{1+O(M^{-1})\}/5$, giving
$V_{\ind}^{\mathrm{cnt}}/V_{\mathrm{single}}=(M/5)\{1+o(1)\}$.
\end{proof}

\subsection{Proof of Remark~\ref{rem:set-single-release}}
\label{app:set-single-release}
Write $V_{\mathrm{single},*}^{\mathrm{set}}:=\max_{|S|\le M}r^{\mathrm{set}}(S)^2/\mu^2$
for the single-release benchmark, where each deletion set is chosen before
release. Under the assumptions of Corollary~\ref{cor:single-release},
we show that
\[
V_{\rw}^{a}=\Theta\!\left(V_{\mathrm{single},*}^{\mathrm{set}}\right)
\quad(a\in\{\mathrm{set},\mathrm{com}\}),
\qquad
V_{\ind}^{\mathrm{set}}=\Theta\!\left(MV_{\mathrm{single},*}^{\mathrm{set}}\right),
\]
uniformly over $D\in\cH_n$.
\begin{proof}
If $r_*^{\mathrm{set}}=0$, all compared variances vanish. For the
nonzero case with $M\ge2$, the proof in
{\color{black}Section~\ref{app:set-variance-proof}} gives
$\sqrt{D^{\mathrm{set}}}\le\alpha_n W_M^2(1+\eta_n)$ and
$r_*^{\mathrm{set}}\ge\alpha_n\gamma_p^2W_M^2/(4\omega_n)$, hence
\[
(r_*^{\mathrm{set}})^2
\le D^{\mathrm{set}}
\le\left\{\frac{4\omega_n(1+\eta_n)}{\gamma_p^2}\right\}^{\!2}
(r_*^{\mathrm{set}})^2.
\]
The uniformity argument in {\color{black}Section~\ref{app:ratio-proof}}
already established $\omega_n=O(1)$ and $\eta_n=o(1)$ uniformly on $\cH_n$.
Since $p$ is fixed, $\gamma_p$ is a positive constant independent of $n$.
Thus the coefficient in the preceding comparison is uniformly bounded
on $\cH_n$. Dividing by $\mu^2$ proves the claim for set-based RW. The result in
Theorem~\ref{thm:variance} gives the same conclusion for computable RW.
Finally, $V_{\ind}^{\mathrm{set}}=\Theta(MV_{\rw}^{\mathrm{set}})$
proves the IND comparison. For $M=1$, both set-based path maxima equal
$(r_*^{\mathrm{set}})^2$, and the computable comparison follows from
Theorem~\ref{thm:variance}.
\end{proof}

\subsection{Proof of Theorem~\ref{thm:covariance-adaptive-singleton}}\label{app:covariance-proof}
\begin{proof}

\textit{Sufficiency.}
Condition on $D$ and suppress its subscript on $\Sigma$. Gaussian conditioning uses
$A_t=\Sigma_{t,<t}\Sigma_{<t,<t}^{-1}$ and
$C_t=\Sigma_{tt}-\Sigma_{t,<t}\Sigma_{<t,<t}^{-1}\Sigma_{<t,t}$.
Fix $\zeta\in\mathcal U$. Extend a realized prefix to $P\in\mathcal C_1$ and
write $u=u_\zeta(P)$, with blocks $u_t=\zeta(S_t)$. Let $u_{1:K}=(\zeta(S_1)^\top,\ldots,\zeta(S_K)^\top)^\top$ denote its first
$K$ blocks, and let $u_{<t}$ denote its first $t-1$ blocks.
At a common history, the conditional mean difference is
$u_t-A_tu_{<t}$ and the covariance is $C_t$.
For $t=1$, there are no preceding releases, so $A_tu_{<t}=0$
and the conditional mean difference is $u_1$.
Block Cholesky factorization gives, for every realized prefix,
\[
\sum_{t\le K}\{u_t-A_tu_{<t}\}^\top
C_t^{-1}\{u_t-A_tu_{<t}\}
=u_{1:K}^\top\Sigma_{1:K,1:K}^{-1}u_{1:K}.
\]
The Schur complement gives
\[
u_{1:K}^\top\Sigma_{1:K,1:K}^{-1}u_{1:K}
\le u^\top\Sigma^{-1}u\le\mu^2.
\]
This proves the pathwise Gaussian budget. Adaptive GDP composition, as used in
{\color{black}Section~\ref{app:allocation-proof}}, establishes sufficiency, including
stopping.

\textit{Necessity.}
Fix $\zeta\in\mathcal U$ and a path
$P\in\mathcal C_1$ before releases, and write $u=u_\zeta(P)$. The two transcript laws are Gaussians
with mean gap $u$, so certification requires
$u^\top\Sigma^{-1}u\le\mu^2$.
Taking all $\zeta\in\mathcal U$ and $P\in\mathcal C_1$ gives the condition.
\end{proof}

\subsection{Proof of Proposition~\ref{prop:covariance-adaptive-singleton-minimax}}
We prove the claims in the order stated. For each upper bound, we give a
covariance that certifies all allowed residual assignments and paths. For
each lower bound, we choose allowed assignments that every feasible
covariance must certify.

\subsubsection{Count-based bounds}
Here $D^{\mathrm{cnt}}=\{\sum_{m=1}^M d^{\mathrm{cnt}}(m-1,m)\}^2$.
If $R^{\mathrm{cnt}}=0$, all residual and increment bounds vanish,
so no noise is needed for certification. We henceforth assume $R^{\mathrm{cnt}}>0$.

We first prove that $\mathfrak V_D(\mathcal U_\Delta)=D^{\mathrm{cnt}}/\mu^2$. For the upper bound, RW is feasible with noise increment variances
$d^{\mathrm{cnt}}(m-1,m)\sqrt{D^{\mathrm{cnt}}}/\mu^2$. So the largest marginal variance of this covariance is
$D^{\mathrm{cnt}}/\mu^2$, and hence
$\mathfrak V_D(\mathcal U_\Delta)\le D^{\mathrm{cnt}}/\mu^2$.

For the lower bound, fix a unit vector $w\in\R^p$ and define
$\zeta(S)=\{\sum_{j=1}^{|S|}d^{\mathrm{cnt}}(j-1,j)\}w$.
On a singleton edge from size $m-1$ to $m$, its increment has norm
$d^{\mathrm{cnt}}(m-1,m)$, so $\zeta\in\mathcal U_\Delta$.
Choose any $P=(S_0,\ldots,S_M)\in\mathcal C_1$. Its residual vector is
$u_\zeta(P)=(\zeta(S_1)^\top,\ldots,\zeta(S_M)^\top)^\top$, and
$\|\zeta(S_M)\|_2^2=D^{\mathrm{cnt}}$ because all increments point in
the same direction. For any feasible $\Sigma$, let
$V=\max_{t\le M}\lambda_{\max}(\Sigma_{tt})$. Certification of this
assignment and path, followed by projection onto the last release,
gives
\[
\mu^2\ge u_\zeta(P)^\top\Sigma^{-1}u_\zeta(P)
\ge \zeta(S_M)^\top\Sigma_{MM}^{-1}\zeta(S_M)
\ge D^{\mathrm{cnt}}/V.
\]
Hence every feasible covariance has $V\ge D^{\mathrm{cnt}}/\mu^2$,
which implies
$\mathfrak V_D(\mathcal U_\Delta)\ge D^{\mathrm{cnt}}/\mu^2$.
Together, the two inequalities give
$\mathfrak V_D(\mathcal U_\Delta)=D^{\mathrm{cnt}}/\mu^2$.

We next prove that $\mathfrak V_D(\mathcal U_0)=R^{\mathrm{cnt}}/\mu^2$. For the upper bound, independent noise with covariance $(R^{\mathrm{cnt}}/\mu^2)I_{pM}$
certifies every assignment in $\mathcal U_0$, since
$\sum_m\|\zeta(S_m)\|_2^2\le R^{\mathrm{cnt}}$ on every path. Thus
$\mathfrak V_D(\mathcal U_0)\le R^{\mathrm{cnt}}/\mu^2$.

For the lower bound, fix a unit vector $w$ and any
$P\in\mathcal C_1$. For each choice of signs $s_m\in\{-1,1\}$, define
$\zeta(\varnothing)=0$ and
$\zeta(S)=s_{|S|}r^{\mathrm{cnt}}(|S|)w$ for $S\ne\varnothing$.
Every such assignment belongs to $\mathcal U_0$ and must be certified
by any feasible $\Sigma$. Let
$V=\max_{t\le M}\lambda_{\max}(\Sigma_{tt})$. Averaging the certificate
inequalities over all $2^M$ sign choices cancels the cross terms and gives
\[
\mu^2\ge\sum_{m=1}^M r^{\mathrm{cnt}}(m)^2
w^\top(\Sigma^{-1})_{mm}w.
\]
The Schur complement gives
$(\Sigma^{-1})_{mm}\succeq\Sigma_{mm}^{-1}$, while the definition of
$V$ gives $\Sigma_{mm}^{-1}\succeq V^{-1}I_p$. Since $\|w\|_2=1$,
$\mu^2\ge R^{\mathrm{cnt}}/V$.
Thus $\mathfrak V_D(\mathcal U_0)\ge R^{\mathrm{cnt}}/\mu^2$.
The two bounds give the claimed equality.

We finally prove that $V_{\rw}^{\mathrm{cnt}}=\{1+o(1)\}\mathfrak V_D(\mathcal U_0\cap\mathcal U_\Delta)$. For the upper bound, the inclusion
$
\mathcal U_0\cap\mathcal U_\Delta\subseteq\mathcal U_\Delta
$
gives
$
\mathfrak V_D(\mathcal U_0\cap\mathcal U_\Delta)
\le\mathfrak V_D(\mathcal U_\Delta)
=D^{\mathrm{cnt}}/\mu^2.
$

For the lower bound, we need an assignment satisfying both constraints
whose final residual is close to $\sqrt{D^{\mathrm{cnt}}}$. We first
compare the sum of the singleton increment bounds with the residual
bound at each count. The definition of $d^{\mathrm{cnt}}$ gives
\[
d^{\mathrm{cnt}}(j-1,j)
=r^{\mathrm{cnt}}(1)\left[(2j-1)
+\left(\frac{C_{\ell,n}}{C_{L,n}}+\frac{B_n}{\kappa_n+B_n}\right)(j-1)^2\right].
\]
The terms $r^{\mathrm{cnt}}(1)(2j-1)$ sum to
$r^{\mathrm{cnt}}(m)=m^2r^{\mathrm{cnt}}(1)$. For the remaining terms,
$\sum_{j=1}^m(j-1)^2\le Mm^2$ and
$B_n/(\kappa_n+B_n)\le B_n/\kappa_n$. Recall that
$\eta_n=MC_{\ell,n}/C_{L,n}+3MB_n/\kappa_n=o(1)$ uniformly on $\cH_n$.
In particular, $M(C_{\ell,n}/C_{L,n}+B_n/\kappa_n)\le\eta_n$.
Thus, for $m\le M$,
\[
\begin{aligned}
\sum_{j=1}^m d^{\mathrm{cnt}}(j-1,j)
&=r^{\mathrm{cnt}}(m)
+r^{\mathrm{cnt}}(1)
\left(\frac{C_{\ell,n}}{C_{L,n}}+\frac{B_n}{\kappa_n+B_n}\right)
\sum_{j=1}^m(j-1)^2\\
&\le r^{\mathrm{cnt}}(m)
\left[1+M\left(\frac{C_{\ell,n}}{C_{L,n}}+\frac{B_n}{\kappa_n}\right)\right]\\
&\le(1+\eta_n)r^{\mathrm{cnt}}(m).
\end{aligned}
\]
Fix a unit vector $w$ and define
$\zeta(S)=(1+\eta_n)^{-1}\{\sum_{j=1}^{|S|}d^{\mathrm{cnt}}(j-1,j)\}w$.
For every $S$ with $|S|=m$, the displayed bound gives
$\|\zeta(S)\|_2\le r^{\mathrm{cnt}}(m)$, so $\zeta\in\mathcal U_0$.
For any singleton edge $A\subset B$ with $|B|=m$,
$\zeta(B)-\zeta(A)=d^{\mathrm{cnt}}(m-1,m)w/(1+\eta_n)$.
Its norm is at most $d^{\mathrm{cnt}}(m-1,m)$ because $\eta_n\ge0$,
so $\zeta\in\mathcal U_\Delta$ as well.
Choose any $P=(S_0,\ldots,S_M)\in\mathcal C_1$. Since
$\sum_{j=1}^M d^{\mathrm{cnt}}(j-1,j)=\sqrt{D^{\mathrm{cnt}}}$,
the final residual satisfies
$\|\zeta(S_M)\|_2^2=D^{\mathrm{cnt}}/(1+\eta_n)^2$.
For any feasible $\Sigma$, let
$V=\max_{t\le M}\lambda_{\max}(\Sigma_{tt})$. Certification of this
assignment and path, followed by projection onto the final release,
gives
\[
\mu^2\ge u_\zeta(P)^\top\Sigma^{-1}u_\zeta(P)
\ge\zeta(S_M)^\top\Sigma_{MM}^{-1}\zeta(S_M)
\ge\frac{D^{\mathrm{cnt}}}{(1+\eta_n)^2V}.
\]
Thus every feasible covariance has
$V\ge D^{\mathrm{cnt}}/\{(1+\eta_n)^2\mu^2\}$. Taking the infimum over
feasible covariances and using the upper bound gives
\[
\frac{V_{\rw}^{\mathrm{cnt}}}{(1+\eta_n)^2}
\le\mathfrak V_D(\mathcal U_0\cap\mathcal U_\Delta)
\le\frac{D^{\mathrm{cnt}}}{\mu^2}
=V_{\rw}^{\mathrm{cnt}}.
\]
Since $\eta_n\to0$ uniformly on $\cH_n$,
\[
V_{\rw}^{\mathrm{cnt}}
=\{1+o(1)\}\mathfrak V_D(\mathcal U_0\cap\mathcal U_\Delta).
\]

\subsubsection{Set-based bounds}\label{app:tree-separation}
Write $r_*^{\mathrm{set}}=\max_{|S|\le M}r^{\mathrm{set}}(S)$ for the
largest residual bound over allowed deletion sets.
If $R^{\mathrm{set}}=0$, all residual and increment bounds vanish,
so no noise is needed for certification. We henceforth
assume $R^{\mathrm{set}}>0$. Then
$D^{\mathrm{set}}>0$ by $r_*^{\mathrm{set}}\le\sqrt{D^{\mathrm{set}}}$.

We first prove that $\mathfrak V_D(\mathcal U_0)=R^{\mathrm{set}}/\mu^2$. For the upper bound, independent noise with covariance $(R^{\mathrm{set}}/\mu^2)I_{pM}$
certifies every assignment in $\mathcal U_0$, since
$\sum_t\|\zeta(S_t)\|_2^2\le R^{\mathrm{set}}$ on every path. Hence
$\mathfrak V_D(\mathcal U_0)\le R^{\mathrm{set}}/\mu^2$.

For the lower bound, choose a path $P\in\mathcal C_1$ attaining
$R^{\mathrm{set}}$. Such a path exists because splitting batches into single deletions and extending
the path to $M$ deletions only adds nonnegative terms to the sum. Fix a unit vector $w$.
For every choice of signs
$s_t\in\{-1,1\}$, define
$\zeta(S_t)=s_t r^{\mathrm{set}}(S_t)w$ on $P$ and set $\zeta=0$ on
all other deletion sets. Each assignment belongs to $\mathcal U_0$
and must be certified by any feasible $\Sigma$. Let
$V=\max_{t\le M}\lambda_{\max}(\Sigma_{tt})$. Averaging the certificate
inequalities over all $2^M$ sign choices cancels the cross terms and gives
\[
\mu^2\ge\sum_{t=1}^M r^{\mathrm{set}}(S_t)^2
w^\top(\Sigma^{-1})_{tt}w.
\]
The Schur complement gives
$(\Sigma^{-1})_{tt}\succeq\Sigma_{tt}^{-1}$, while the definition of
$V$ gives $\Sigma_{tt}^{-1}\succeq V^{-1}I_p$. Since $\|w\|_2=1$,
$\mu^2\ge R^{\mathrm{set}}/V$.
Thus $\mathfrak V_D(\mathcal U_0)\ge R^{\mathrm{set}}/\mu^2$.
The two bounds give the claimed equality.

\noindent For the remaining bounds, fix
$\mathcal U\in\{\mathcal U_\Delta,\mathcal U_0\cap\mathcal U_\Delta\}$.

For the upper bound $\mathfrak V_D(\mathcal U)\le C_*\log^2(M+1)V_{\rw}^{\mathrm{set}}$, we construct a covariance that certifies
$\mathcal U_\Delta$ and has small marginal variances. Let
$K\in\mathbb R^{M\times M}$ be positive definite and set
$\Sigma=(D^{\mathrm{set}}/\mu^2)(K\otimes I_p)$. Here $K_{tt}$ is a
scalar, while the covariance block for release $t$ is
$\Sigma_{tt}=(D^{\mathrm{set}}/\mu^2)K_{tt}I_p$. Thus
$\lambda_{\max}(\Sigma_{tt})=(D^{\mathrm{set}}/\mu^2)K_{tt}$, and
controlling the largest marginal variance amounts to controlling
$\max_tK_{tt}$.

We next determine what $K$ must satisfy for certification. For any
$\zeta\in\mathcal U_\Delta$ and
$P=(S_0,\ldots,S_M)\in\mathcal C_1$, let
$\Delta_j=\zeta(S_j)-\zeta(S_{j-1})$. Since
$\zeta(S_t)=\sum_{j=1}^t\Delta_j$, the stacked residual vector is
$u_\zeta(P)=\sum_{j=1}^M l_j\otimes\Delta_j$, where
$l_j=(\mathbf1_{t\ge j})_{t=1}^M$. As
$\Sigma^{-1}=(\mu^2/D^{\mathrm{set}})(K^{-1}\otimes I_p)$, the
certificate holds if
$\|u_\zeta(P)\|_{K^{-1}\otimes I_p}\le\sqrt{D^{\mathrm{set}}}$.
Here $\|x\|_{K^{-1}\otimes I_p}^2
=x^\top(K^{-1}\otimes I_p)x$.
If $l_j^\top K^{-1}l_j\le1$ for every $j$, then
\[
\begin{aligned}
\|u_\zeta(P)\|_{K^{-1}\otimes I_p}
&=\left\|\sum_{j=1}^M l_j\otimes\Delta_j\right\|_{K^{-1}\otimes I_p}
\le\sum_{j=1}^M\|l_j\otimes\Delta_j\|_{K^{-1}\otimes I_p}\\
&=\sum_{j=1}^M\sqrt{l_j^\top K^{-1}l_j}\,\|\Delta_j\|_2
\le\sum_{j=1}^M d^{\mathrm{set}}(S_{j-1},S_j)
\le\sqrt{D^{\mathrm{set}}}.
\end{aligned}
\]
It remains to construct $K$ with this property and
$\max_tK_{tt}=O(\log^2(M+1))$.

Set $h=1+\lceil\log_2M\rceil$ and $N=2^{h-1}\ge M$. For
$0\le k\le h-1$ and $0\le a<N/2^k$, define
\[
I_{k,a}=\{a2^k+1,\ldots,(a+1)2^k\}.
\]
Use each $I_{k,a}$ as a column index and set
$A_{t,I}=\mathbf1\{t\in I\}$ for $1\le t\le M$. Distinct intervals
remain distinct columns even if they agree on $\{1,\ldots,M\}$. Each
row belongs to $h$ intervals. Starting at $j$, successively choose the
largest $I_{k,a}$ beginning at the next uncovered index. Their lengths
strictly increase, so they partition $\{j,\ldots,N\}$ into at most $h$
intervals. Restricting to the first $M$ coordinates expresses $l_j$ as
the sum of their columns of $A$. The intervals of length one give
$A$ full row rank, so $K=hAA^\top$ is positive definite. The row
count gives $K_{tt}=h^2$. The column representation gives a vector
$c$ with entries in $\{0,1\}$, $Ac=l_j$, and $\|c\|_2^2\le h$. Since $A$ has full row
rank, its least norm solution to $Ac=l_j$ is
$A^\top(AA^\top)^{-1}l_j$. Hence
\[
l_j^\top K^{-1}l_j
=\frac1h\min_{Ac=l_j}\|c\|_2^2\le1.
\]
Thus $\Sigma$ certifies $\mathcal U_\Delta$, and hence also
$\mathcal U$. Its largest marginal variance is at most
$h^2V_{\rw}^{\mathrm{set}}$, so
\[
\mathfrak V_D(\mathcal U)
\le h^2V_{\rw}^{\mathrm{set}}
\le C\log^2(M+1)V_{\rw}^{\mathrm{set}}
\]
for a constant $C$.

For the lower bound $\mathfrak V_D(\mathcal U)\ge V_{\rw}^{\mathrm{set}}/C_*$, fix a unit vector $w$ and define
$\zeta(S)=r^{\mathrm{set}}(S)w$. Then $\zeta(\varnothing)=0$ and
$\|\zeta(S)\|_2=r^{\mathrm{set}}(S)$ for every $S$. For every singleton
edge $A\subset B$, \eqref{eq:bound-relations} gives
$\|\zeta(B)-\zeta(A)\|_2
=|r^{\mathrm{set}}(B)-r^{\mathrm{set}}(A)|
\le d^{\mathrm{set}}(A,B)$. Hence
$\zeta\in\mathcal U_0\cap\mathcal U_\Delta\subseteq\mathcal U$.
Choose a set $S_*$ attaining $r_*^{\mathrm{set}}$ and a full
singleton path $P$ through it at time $t_*=|S_*|\ge1$. For any feasible
$\Sigma$, let $V=\max_{t\le M}\lambda_{\max}(\Sigma_{tt})$.
The certificate, the Schur complement, and
$\Sigma_{t_*t_*}\preceq VI_p$ give, respectively,
\[
\mu^2\ge u_\zeta(P)^\top\Sigma^{-1}u_\zeta(P)
\ge \zeta(S_*)^\top\Sigma_{t_*t_*}^{-1}\zeta(S_*)
\ge (r_*^{\mathrm{set}})^2/V.
\]
Thus every feasible covariance has $V\ge(r_*^{\mathrm{set}})^2/\mu^2$,
so $\mathfrak V_D(\mathcal U)\ge(r_*^{\mathrm{set}})^2/\mu^2$.

If $M=1$, then $D^{\mathrm{set}}=(r_*^{\mathrm{set}})^2$, so
$\mathfrak V_D(\mathcal U)\ge V_{\rw}^{\mathrm{set}}$.
For $M\ge2$, the cap condition gives $n\ge2M$ for all sufficiently
large $n$, and $R^{\mathrm{set}}>0$ implies $C_{L,n}>0$ and $W_M>0$.
Under these conditions, \eqref{eq:set-path-bound} gives
$\sqrt{D^{\mathrm{set}}}\le\alpha_n W_M^2(1+\eta_n)$, and
{\color{black}Section~\ref{app:ratio-proof}} gives
$r_*^{\mathrm{set}}\ge\alpha_n\gamma_p^2W_M^2/(4\omega_n)$.
Since $V_{\rw}^{\mathrm{set}}=D^{\mathrm{set}}/\mu^2$, these inequalities give
\[
\begin{aligned}
\sqrt{D^{\mathrm{set}}}
&\le\alpha_n W_M^2(1+\eta_n)
\le\frac{4\omega_n(1+\eta_n)}{\gamma_p^2}\,r_*^{\mathrm{set}},\\
\mathfrak V_D(\mathcal U)
&\ge\frac{(r_*^{\mathrm{set}})^2}{\mu^2}
\ge\left\{\frac{\gamma_p^2}{4\omega_n(1+\eta_n)}\right\}^{\!2}
V_{\rw}^{\mathrm{set}}.
\end{aligned}
\]
Since $\sup_{D\in\cH_n}\omega_n=O(1)$,
$\sup_{D\in\cH_n}\eta_n=o(1)$, and $\gamma_p>0$ is fixed, there is
$C_*<\infty$, independent of $D$ and $n$, such that
$\{4\omega_n(1+\eta_n)/\gamma_p^2\}^2\le C_*$ for all sufficiently large
$n$. Thus $\mathfrak V_D(\mathcal U)\ge V_{\rw}^{\mathrm{set}}/C_*$.
Increasing $C_*$ to cover the constant $C$ above gives
$V_{\rw}^{\mathrm{set}}/C_*\le\mathfrak V_D(\mathcal U)
\le C_*\log^2(M+1)V_{\rw}^{\mathrm{set}}$.

\subsubsection{A matching lower bound}\label{app:matching-separation}
We now prove that $\mathfrak V_{D_n}(\mathcal U)/V_{\rw}^{\mathrm{set}}=\Theta(\log^2 M)$ for a sequence $D_n\in\cH_n$.
The upper bound was proved above for
$\mathcal U\in\{\mathcal U_\Delta,\mathcal U_0\cap\mathcal U_\Delta\}$.
It remains to construct $D_n\in\cH_n$ for which
$\mathfrak V_{D_n}(\mathcal U)\ge c\log^2 M\,V_{\rw}^{\mathrm{set}}$.
Let $M\to\infty$ in the stated cap regime, so $M=o(n)$.

Fix $a>0$. We first construct data with $\widehat\beta=0$ and
\[
\nabla\ell_1(\widehat\beta)=a,
\qquad
\nabla\ell_i(\widehat\beta)=-\frac{a}{n-1},\quad i\ge2.
\]
Thus deleting record 1 changes the gradient sum by $a$, while at most
$M$ other deletions change it by at most $Ma/(n-1)$.

Take $p=1$, $\mathcal P(\beta)=\beta^2$, and $X$ uniform on $\{-1,1\}$.
Fix $0<\epsilon<1/2$. Let $Y$ be independent of $X$ with a
fixed symmetric distribution supported on $[-a,a]$ having positive mass at $0$, $\pm a$,
and $\pm a/(n-1)$ for every even $n$. Use
\[
\ell(y,z)=\frac{z^2}{2}+\epsilon(z-\sin z)-yz
+\frac{y^2}{2(1-\epsilon)}.
\]
The individual Hessians vary with $\beta$, but their sine terms cancel
between observations with opposite feature signs.
Since $|z-\sin z|\le z^2/2$,
$\ell(y,z)\ge\{(1-\epsilon)z-y\}^2/\{2(1-\epsilon)\}\ge0$.
Also $\partial_z^2\ell(y,z)=1+\epsilon\sin z\ge1-\epsilon$ and
$|\partial_z^3\ell(y,z)|\le\epsilon$. Thus the loss and sampling
assumptions hold. The population risk
$\mathcal R(\beta)=\E\{\ell(Y,X\beta)\}$ satisfies
$\mathcal R'(\beta)=\beta$ and $\mathcal R''(\beta)=1$.
Thus $\beta_0=0$, and we take $\cN=[-1,1]$ as the fixed compact convex
neighborhood of $\beta_0$, with $c_0=1$.

For even $n$, set $x_i=1$ for $i\le n/2$ and $x_i=-1$ for $i>n/2$,
with $y_1=-a$ and $y_i=ax_i/(n-1)$ for $i\ge2$.

To verify $D_n\in\cH_n$, recall that the event constructed in
Section~\ref{app:bounds-proof} requires two observation bounds, a
Hessian bound, and a gradient bound. For this distribution, take
$u_n=1$ and $B_y(n)=a$. The two observation bounds hold because
$\max_{i\le n}|x_i|=1=u_n$ and $\max_{i\le n}|y_i|=a=B_y(n)$.
Choose $\rho_0=1/2$, so $[-\rho_0,\rho_0]\subset\operatorname{int}(\cN)$.
Direct calculation gives the remaining two conditions:
\[
\begin{aligned}
\sup_{\beta\in\cN}
\left|\frac1n\sum_{i=1}^n\nabla^2\ell_i(\beta)
-\E\nabla^2\ell_i(\beta)\right|
&=0\le\frac{c_0}{4}=\frac14,\\
\left|\frac1n\nabla L_\varnothing(\beta_0)\right|
&=0\le\frac{c_0\rho_0}{32}=\frac1{64}.
\end{aligned}
\]
Hence $D_n\in\cH_n$. Strict convexity and
$\nabla L_\varnothing(0)=0$ give $\widehat\beta=0$, with the individual
gradients specified above.

On $\cN$, direct calculation gives
$|\nabla\ell_i(\beta)|\le G_n:=1+2\epsilon+a$ and
$|\nabla^2\ell_i(\beta)|\le B_n:=1+\epsilon$.
The third derivative bound gives $C_{\ell,n}:=\epsilon$.
Since the regularizer has constant Hessian, we take
$C_{L,n}:=nC_{\ell,n}=n\epsilon$.
Finally, $\kappa_n:=n/4$ is valid because, for every $|A|\le M$
and $\beta\in\cN$,
\[
\nabla^2L_A(\beta)\ge(n-M)(1-\epsilon)+2\lambda\ge\kappa_n
\]
for all sufficiently large $n$, since $M=o(n)$.

To compute the RW variance on these data, we first evaluate the
residual bounds. For every $|S|\le M$,
\[
H_S=n-|S|+2\lambda,
\qquad
g_S=
\begin{cases}
a\{1-(|S|-1)/(n-1)\},&1\in S,\\
-a|S|/(n-1),&1\notin S.
\end{cases}
\]
Since $r^{\mathrm{set}}(S)=C_{L,n}g_S^2/(2\kappa_n^2H_S)$ and $M=o(n)$,
\[
r^{\mathrm{set}}(S)=
\begin{cases}
r^{\mathrm{set}}(\{1\})\{1+O(M/n)\},&1\in S,\\
O\{(M/n)^2r^{\mathrm{set}}(\{1\})\},&1\notin S.
\end{cases}
\]
Here $r^{\mathrm{set}}(\{1\})=C_{L,n}a^2/(2\kappa_n^2H_{\{1\}})
=\Theta(a^2/n^2)$.

To bound $\sqrt{D^{\mathrm{set}}}$, we bound the sum of
$d^{\mathrm{set}}(A,B)$ over all edges of a deletion path.
We consider separately paths that delete record 1 and paths that do not.

\textit{Paths that delete record 1.}
Split the sum into the edges before
that deletion, the edge deleting record 1, and the edges after it.
Before record 1 is deleted, all deleted records have gradient norm
$a/(n-1)$, and their norms sum to at most $Ma/(n-1)$.
The calculation in \eqref{eq:set-path-bound} therefore bounds the sum
of $d^{\mathrm{set}}(A,B)$ over these preceding edges by
$\alpha_n\{Ma/(n-1)\}^2(1+\eta_n)
=O\{(M/n)^2r^{\mathrm{set}}(\{1\})\}$.

For the edge $A\subsetneq B$ that deletes record 1,
$d^{\mathrm{set}}(A,B)\le r^{\mathrm{set}}(A)+r^{\mathrm{set}}(B)
=r^{\mathrm{set}}(\{1\})\{1+O(M/n)\}$.

After record 1 has been deleted, an edge adding $q$ records has
$\|g_B-g_A\|_2=qa/(n-1)$, $\rho_A,\rho_B=O(a/n)$,
$\delta_{A,B}=O(qa/n^2)$, and $H_A-H_B=q$.
Each of the three terms in $U_{A,B}$ is therefore
$O(qa^2/n^3)=O\{qr^{\mathrm{set}}(\{1\})/n\}$.
Since $d^{\mathrm{set}}(A,B)\le U_{A,B}$ and $\sum q\le M$,
the sum over these subsequent edges is
$O\{(M/n)r^{\mathrm{set}}(\{1\})\}$.
Adding the three bounds gives a total increment sum of at most
$r^{\mathrm{set}}(\{1\})\{1+O(M/n)\}$.

\textit{Paths that do not delete record 1.}
All deleted records have gradient norm $a/(n-1)$, so their norms sum
to at most $Ma/(n-1)$ over the entire path.
The calculation in \eqref{eq:set-path-bound} bounds the total increment
sum by
\[
\alpha_n\{Ma/(n-1)\}^2(1+\eta_n)
=O\{(M/n)^2r^{\mathrm{set}}(\{1\})\}
=o\{r^{\mathrm{set}}(\{1\})\}.
\]

Thus both cases give a total increment sum of at most
$r^{\mathrm{set}}(\{1\})\{1+O(M/n)\}$.
Taking the maximum over paths gives the upper bound below.
A path starting with $\varnothing\subset\{1\}$ gives the lower bound,
since $d^{\mathrm{set}}(\varnothing,\{1\})=r^{\mathrm{set}}(\{1\})$:
\[
r^{\mathrm{set}}(\{1\})
\le\sqrt{D^{\mathrm{set}}}
\le r^{\mathrm{set}}(\{1\})\{1+O(M/n)\}.
\]
Since $V_{\rw}^{\mathrm{set}}=D^{\mathrm{set}}/\mu^2$, we obtain
\[
V_{\rw}^{\mathrm{set}}
=\frac{\{r^{\mathrm{set}}(\{1\})\}^2}{\mu^2}\{1+O(M/n)\}.
\]

We next prove a lower bound for the largest marginal variance of any
feasible fixed covariance.
The same covariance must certify all possible deletion times of record 1.
To obtain these constraints, define an assignment that is zero before
that deletion and constant afterward:
\[
r_n:=r^{\mathrm{set}}(\{1\})\left(1-\frac{M}{n-1}\right)^2,
\qquad
\zeta(S):=r_n\mathbf1_{\{1\in S\}}.
\]
The explicit formulas for $g_S$ and $H_S$ give
$r^{\mathrm{set}}(S)\ge r_n$ whenever $1\in S$.
On a singleton edge adding record 1,
$\rho_A+\rho_B=a/\kappa_n$ and $\delta_{A,B}\ge a/\kappa_n$.
Thus $U_{A,B}\ge r^{\mathrm{set}}(\{1\})\ge r_n$ and
$r^{\mathrm{set}}(A)+r^{\mathrm{set}}(B)\ge r_n$, so
$d^{\mathrm{set}}(A,B)\ge r_n$. On all other edges, $\zeta$ has zero
increment. Hence $\zeta\in\mathcal U_0\cap\mathcal U_\Delta$.

If record 1 is deleted at step $j$, the residual vector along that path is
$u_\zeta(P)=r_nl_j$, where $l_j=(\mathbf1_{t\ge j})_{t=1}^M\in\R^M$ has
$j-1$ zeros followed by $M-j+1$ ones.
Every $j\in\{1,\ldots,M\}$ is possible, so the same feasible covariance
$\Sigma$ must satisfy
\[
r_n^2l_j^\top\Sigma^{-1}l_j\le\mu^2
\qquad(1\le j\le M).
\]
To turn these constraints into a lower bound on
$V=\max_{t\le M}\Sigma_{tt}$, collect the vectors as columns of
$L=(l_1,\ldots,l_M)$, the lower triangular matrix of ones.
Summing the preceding constraints gives
\[
\|\Sigma^{-1/2}L\|_F^2
=\sum_{j=1}^M\|\Sigma^{-1/2}l_j\|_2^2
=\sum_{j=1}^M l_j^\top\Sigma^{-1}l_j
\le\frac{M\mu^2}{r_n^2}.
\]
Since $\Sigma_{tt}\le V$ at every time $t$,
\[
\|\Sigma^{1/2}\|_F^2
=\operatorname{tr}(\Sigma)
=\sum_{t=1}^M\Sigma_{tt}
\le MV.
\]
Writing $\|L\|_*$ for the sum of its singular values and using
$L=\Sigma^{1/2}(\Sigma^{-1/2}L)$, we obtain
\[
\|L\|_*\le\|\Sigma^{1/2}\|_F\|\Sigma^{-1/2}L\|_F
\le\sqrt{MV}\,\frac{\sqrt M\,\mu}{r_n}
=\frac{M\mu\sqrt V}{r_n}.
\]
Thus
\[
V\ge\frac{r_n^2}{\mu^2}\left(\frac{\|L\|_*}{M}\right)^2.
\]
It remains to show $\|L\|_*\ge M\log M/\pi$.
The inverse of $L$ has ones on the diagonal and minus ones on the
subdiagonal. Thus $(L^{-1})^\top L^{-1}$ is tridiagonal, with diagonal
$(2,\ldots,2,1)$ and adjacent entries $-1$. Its eigenvalues give
$s_j(L)=[2\sin\{(2j-1)\pi/(4M+2)\}]^{-1}$ for $1\le j\le M$.
Using $\sin x\le x$ and
$\sum_{j=1}^M(2j-1)^{-1}\ge\frac12\sum_{j=1}^M j^{-1}\ge\frac12\log M$,
\[
\|L\|_*=\sum_{j=1}^M s_j(L)
\ge\frac{2M+1}{\pi}\sum_{j=1}^M\frac1{2j-1}
\ge\frac{M\log M}{\pi}.
\]
Since $r_n=r^{\mathrm{set}}(\{1\})\{1+O(M/n)\}$, every feasible
fixed covariance satisfies, for all sufficiently large $n$,
\[
V\ge\frac{r_n^2}{\mu^2}\left(\frac{\|L\|_*}{M}\right)^2
\ge\frac{\{r^{\mathrm{set}}(\{1\})\}^2\log^2 M}{4\pi^2\mu^2}.
\]
Taking the infimum over $\Sigma$ and using
$V_{\rw}^{\mathrm{set}}\le2\{r^{\mathrm{set}}(\{1\})\}^2/\mu^2$
for all sufficiently large $n$ gives
$\mathfrak V_{D_n}(\mathcal U)/V_{\rw}^{\mathrm{set}}\ge\log^2 M/(8\pi^2)$.
The upper bound proved above gives
$\mathfrak V_{D_n}(\mathcal U)/V_{\rw}^{\mathrm{set}}=O(\log^2 M)$.
Together with the lower bound, this proves
\[
\frac{\mathfrak V_{D_n}(\mathcal U)}{V_{\rw}^{\mathrm{set}}}
=\Theta(\log^2 M),
\qquad
\mathcal U\in\{\mathcal U_\Delta,\mathcal U_0\cap\mathcal U_\Delta\}.
\]

\subsection{Proof of Theorem~\ref{thm:utility-consistency}}\label{app:utility-proof}
We first prove the conditional prediction identities used both in the
experiments and in the consistency argument.

\begin{lemma}
\label{lem:adaptive-identities}
Conditional on $\mathcal F_t$, suppose the test feature is independent
of the fresh release noise and has second moment matrix $\Gamma$.
For IND and RW,
\begin{equation}
\label{eq:prediction-identities}
\begin{aligned}
\Pred_t^{\ind}
&=e_{S_t}^\top\Gamma e_{S_t}+\tr(\Gamma)\tau_t^2,\\
\Pred_t^{\rw}
&=(e_{S_t}+Z_{t-1})^\top\Gamma(e_{S_t}+Z_{t-1})
+\tr(\Gamma)\sigma_t^2.
\end{aligned}
\end{equation}
\end{lemma}

\begin{proof}
Conditional on $\mathcal F_t$, write the parameter error as
$b+\zeta$, where $b$ is fixed and $\zeta\sim N(0,vI_p)$ is independent
of the test feature. Its conditional mean is zero and its conditional
second moment is $vI_p$. Thus
\[
\begin{aligned}
\E\!\left[\{x_0^\top(b+\zeta)\}^2\mid\mathcal F_t,x_0\right]
&=(x_0^\top b)^2
+2(x_0^\top b)x_0^\top\E(\zeta\mid\mathcal F_t,x_0)\\
&\quad+x_0^\top\E(\zeta\zeta^\top\mid\mathcal F_t,x_0)x_0\\
&=(x_0^\top b)^2+v\|x_0\|_2^2.
\end{aligned}
\]
Taking expectation with respect to the test distribution
gives
\[
\begin{aligned}
\E\!\left[\{x_0^\top(b+\zeta)\}^2\mid\mathcal F_t\right]
&=b^\top\E(x_0x_0^\top\mid\mathcal F_t)b
+v\,\E(\|x_0\|_2^2\mid\mathcal F_t)\\
&=b^\top\Gamma b+v\tr(\Gamma).
\end{aligned}
\]
For IND, take $b=e_{S_t}$, $\zeta=\varepsilon_t$, and $v=\tau_t^2$.
For RW, take $b=e_{S_t}+Z_{t-1}$, $\zeta=\xi_t$, and $v=\sigma_t^2$.
Here $Z_0=0$, and, for $t\ge2$,
$Z_{t-1}=Y_{t-1}^{\rw}-\widetilde\beta_{S_{t-1}}$ is
$\mathcal F_t$-measurable.
\end{proof}

\begin{proof}[Proof of Theorem~\ref{thm:utility-consistency}]
We first control the maximum parameter and prediction errors by the
residual bounds and allocated variances. We then apply the conditions in
the theorem and transfer prediction consistency to GED consistency.
Throughout, condition on $D\in\cH_n$. Strict nesting gives $T\le M$.
After stopping, extend the sequence to $M$ slots with zero RW increments.
All bounds below hold for every allowed deletion policy.

\par\medskip\noindent\textbf{Parameter and prediction errors.}
The released parameter errors satisfy
\[
Y_t^{\ind}-\widehat\beta_{S_t}=e_{S_t}+\varepsilon_t,
\qquad
Y_t^{\rw}-\widehat\beta_{S_t}=e_{S_t}+Z_t,
\quad Z_t=\sum_{s\le t}\xi_s.
\]
We use \eqref{eq:prediction-identities} for the prediction errors.

\textit{RW.} To control the accumulated noise, use the fact that $Z_t$
is a martingale. The accumulated variance is bounded by $V_{\rw}$
on every path. Doob's inequality and orthogonality of the increments give
\[
\E_\pi\!\left[\max_{t\le T}\|Z_t\|_2^2\mid D\right]
\le4\E_\pi[\|Z_M\|_2^2\mid D]
=4p\E_\pi\!\left[\sum_{t\le T}\sigma_t^2\mid D\right]
\le4pV_{\rw}.
\]
For count-based RW, $\|e_{S_t}\|_2\le r^{\mathrm{cnt}}(M)$ and
$\sigma_t^2\le V_{\rw}^{\mathrm{cnt}}$.
For each release, the prediction identity and
$\|a+b\|_2^2\le2\|a\|_2^2+2\|b\|_2^2$ give
\[
\begin{aligned}
\Pred_t^{\rw}
&=\|\Gamma^{1/2}(e_{S_t}+Z_{t-1})\|_2^2
+\tr(\Gamma)\sigma_t^2\\
&\le2\|\Gamma^{1/2}e_{S_t}\|_2^2
+2\|\Gamma^{1/2}Z_{t-1}\|_2^2
+\tr(\Gamma)\sigma_t^2\\
&\le2\|\Gamma\|_2r^{\mathrm{cnt}}(M)^2
+2\|\Gamma^{1/2}Z_{t-1}\|_2^2
+\tr(\Gamma)V_{\rw}^{\mathrm{cnt}}.
\end{aligned}
\]
Taking the maximum over releases and enlarging the noise maximum
from $0\le t\le T-1$ to $0\le t\le T$ yields
\[
\max_{t\le T}\Pred_t^{\rw}
\le2\|\Gamma\|_2r^{\mathrm{cnt}}(M)^2
+2\max_{0\le t\le T}\|\Gamma^{1/2}Z_t\|_2^2
+\tr(\Gamma)V_{\rw}^{\mathrm{cnt}}.
\]
The transformed process $\Gamma^{1/2}Z_t$ is also a martingale.
Its increments remain orthogonal in conditional expectation given $D$,
and
$\E_\pi(\|\Gamma^{1/2}\xi_t\|_2^2\mid\mathcal F_t)
=\tr(\Gamma)\sigma_t^2$.
Thus Doob's inequality gives
\[
\begin{aligned}
\E_\pi\!\left[\max_{t\le T}\|\Gamma^{1/2}Z_t\|_2^2\mid D\right]
&\le4\E_\pi[\|\Gamma^{1/2}Z_M\|_2^2\mid D]\\
&=4\tr(\Gamma)\E_\pi\!\left[\sum_{t\le T}\sigma_t^2\mid D\right]\\
&\le4\tr(\Gamma)V_{\rw}^{\mathrm{cnt}}.
\end{aligned}
\]
Taking conditional expectation in the prediction bound therefore gives
\[
\E_\pi\!\left[\max_{t\le T}\Pred_t^{\rw}\mid D\right]
\le2\|\Gamma\|_2r^{\mathrm{cnt}}(M)^2
+9\tr(\Gamma)V_{\rw}^{\mathrm{cnt}}.
\]

For parameter error, use
$Y_t^{\rw}-\widehat\beta_{S_t}=e_{S_t}+Z_t$ and the same squared
norm inequality to obtain
\[
\max_{t\le T}\|Y_t^{\rw}-\widehat\beta_{S_t}\|_2^2
\le2r^{\mathrm{cnt}}(M)^2+2\max_{t\le T}\|Z_t\|_2^2.
\]
Taking conditional expectation and using the bound on
$\E_\pi[\max_{t\le T}\|Z_t\|_2^2\mid D]$ gives
\[
\E_\pi\!\left[
\max_{t\le T}\|Y_t^{\rw}-\widehat\beta_{S_t}\|_2^2\mid D
\right]
\le2r^{\mathrm{cnt}}(M)^2+8pV_{\rw}^{\mathrm{cnt}}.
\]
For set-based and computable RW, the same bounds hold
with $r_*^{\mathrm{set}}$ in place of $r^{\mathrm{cnt}}(M)$ and with
$V_{\rw}^{\mathrm{set}}$ or $V_{\rw}^{\mathrm{com}}$, respectively.

\textit{IND.} Its noise variance $V_{\ind}$ is constant across releases
conditional on $D$. Generate $M$ independent noise vectors
$\varepsilon_1,\ldots,\varepsilon_M\sim N(0,V_{\ind}I_p)$, and use
$\varepsilon_t$ at release $t$. Since $T\le M$,
$\max_{1\le t\le T}\|\varepsilon_t\|_2\le\max_{1\le t\le M}\|\varepsilon_t\|_2$,
even when the stopping time $T$ depends on earlier releases.
For $G_1,\ldots,G_M\sim N(0,I_p)$, Jensen's inequality and
$\E\exp(\|G_j\|_2^2/4)=2^{p/2}$ give
\[
\E\max_{j\le M}\|G_j\|_2^2
\le4\E\log\sum_{j=1}^M\exp(\|G_j\|_2^2/4)
\le4\log M+2p\log2.
\]
For count-based IND, it follows that
\[
\begin{aligned}
\E_\pi\!\left[
\max_{t\le T}\|Y_t^{\ind}-\widehat\beta_{S_t}\|_2^2\mid D
\right]
&\le2r^{\mathrm{cnt}}(M)^2
+(8\log M+4p\log2)V_{\ind}^{\mathrm{cnt}},\\
\max_{t\le T}\Pred_t^{\ind}
&\le\|\Gamma\|_2r^{\mathrm{cnt}}(M)^2
+\tr(\Gamma)V_{\ind}^{\mathrm{cnt}}.
\end{aligned}
\]
The set-based bounds follow by replacing $r^{\mathrm{cnt}}(M)$ with
$r_*^{\mathrm{set}}$ and $V_{\ind}^{\mathrm{cnt}}$ with
$V_{\ind}^{\mathrm{set}}$. Both error bounds remain valid when the noise variance is zero.

\par\medskip\noindent\textbf{Consistency conditions.}
For each of RW and IND, let $U$ denote either the maximum squared
parameter error or the maximum prediction error, and let $Q_n(D)$
be its corresponding bound derived above. In each case, we have
\[
\E_\pi[U\mid D]\le Q_n(D)
\qquad(D\in\cH_n)
\]
for every allowed deletion policy $\pi$.
To obtain consistency from these bounds, we first prove that
\[
Q_n(D)\xrightarrow{p}0
\quad\Longrightarrow\quad
\sup_\pi\Pp_\pi(U>\epsilon)\longrightarrow0
\qquad\text{for every }\epsilon>0.
\]
For $\eta,\epsilon>0$, conditional Markov's inequality gives
\[
\begin{aligned}
\sup_\pi\Pp_\pi(U>\epsilon)
&\le\delta_n+\Pp\{D\in\cH_n,\ Q_n(D)>\eta\}\\
&\quad+\sup_\pi\E\!\left[
\mathbf1_{\{D\in\cH_n,\ Q_n(D)\le\eta\}}
\Pp_\pi(U>\epsilon\mid D)\right]\\
&\le\delta_n+\Pp\{D\in\cH_n,\ Q_n(D)>\eta\}
+\frac{\eta}{\epsilon}.
\end{aligned}
\]
If $Q_n(D)\xrightarrow{p}0$ on $\cH_n$, then
$\Pp\{D\in\cH_n,\ Q_n(D)>\eta\}\to0$ for every fixed $\eta>0$.
Together with $\delta_n\to0$, the preceding bound gives
\[
\limsup_{n\to\infty}\sup_\pi\Pp_\pi(U>\epsilon)
\le\frac{\eta}{\epsilon}.
\]
Letting $\eta\downarrow0$ yields
$\sup_\pi\Pp_\pi(U>\epsilon)\to0$ for every $\epsilon>0$.

It remains to verify $Q_n(D)\xrightarrow{p}0$ for each of the four
error bounds under the stated consistency conditions.
For count-based allocations,
{\color{black}Section~\ref{app:count-variance-proof}} gives
uniformly on $\cH_n$,
\[
V_{\rw}^{\mathrm{cnt}}
=\frac{r^{\mathrm{cnt}}(M)^2\{1+o(1)\}}{\mu^2},
\qquad
V_{\ind}^{\mathrm{cnt}}
=\frac{R^{\mathrm{cnt}}}{\mu^2}
\le\frac{M r^{\mathrm{cnt}}(M)^2}{\mu^2}.
\]
Substituting these variance bounds into the error bounds above shows
that $r^{\mathrm{cnt}}(M)\to0$ suffices for RW parameter and prediction
consistency. For IND, $r^{\mathrm{cnt}}(M)\sqrt M\to0$ suffices for
prediction consistency, and
$r^{\mathrm{cnt}}(M)\sqrt{M\log(M+1)}\to0$ suffices for parameter
consistency.

For set-based allocations, \eqref{eq:bound-relations} and the path maxima
give
\[
\begin{aligned}
(r_*^{\mathrm{set}})^2&\le D^{\mathrm{set}},
&V_{\rw}^{\mathrm{set}}&=D^{\mathrm{set}}/\mu^2,\\
R^{\mathrm{set}}&\le M(r_*^{\mathrm{set}})^2\le M D^{\mathrm{set}},
&V_{\ind}^{\mathrm{set}}&=R^{\mathrm{set}}/\mu^2.
\end{aligned}
\]
Substituting these relations into the same error bounds shows that
$D^{\mathrm{set}}\xrightarrow{p}0$ suffices for both RW
errors. For IND, $M D^{\mathrm{set}}\xrightarrow{p}0$ suffices for prediction
consistency, and $M\log(M+1)D^{\mathrm{set}}\xrightarrow{p}0$ suffices for
IND parameter consistency. Computable RW has
$(r_*^{\mathrm{set}})^2\le D^{\mathrm{com}}$ and accumulated variance
at most $D^{\mathrm{com}}/\mu^2$, so the same RW argument applies when
$D^{\mathrm{com}}\xrightarrow{p}0$. The count-based conditions also
cover the other allocations because their residual bounds, IND variances,
and accumulated RW variances are no larger.

\par\medskip\noindent\textbf{GED consistency.}
We prove
\[
\max_{t\le T}\GED_t
\le C_G\sqrt{\max_{t\le T}\Pred_t}
\left\{1+(\max_{t\le T}\Pred_t)^{s/2}\right\},
\]
where $C_G$ is independent of $n,M,T$ and the deletion policy.
The prediction consistency results then imply GED consistency.

Write
$\Delta_t=Y_t-\widehat\beta_{S_t}$. Expectations in this paragraph are
conditional on $\mathcal F_t$. The mean value theorem and the derivative growth
condition give
\[
\begin{aligned}
&\left|
\ell(y_0,x_0^\top\widehat\beta_{S_t}+x_0^\top\Delta_t)
-\ell(y_0,x_0^\top\widehat\beta_{S_t})
\right|\\
&\quad\le C|x_0^\top\Delta_t|
\{1+|y_0|^s+|x_0^\top\widehat\beta_{S_t}|^s
+|x_0^\top\Delta_t|^s\}.
\end{aligned}
\]
On $\cH_n$, all retained estimators lie in the fixed compact set $\cN$.
Cauchy--Schwarz bounds the expectations of the first three terms by
\[
\begin{aligned}
&\E\!\left[|x_0^\top\Delta_t|
\{1+|y_0|^s+|x_0^\top\widehat\beta_{S_t}|^s\}
\mid\mathcal F_t\right]\\
&\quad\le\sqrt{\Pred_t}\,
\left\{1+(\E|y_0|^{2s})^{1/2}
+\big(\E|x_0^\top\widehat\beta_{S_t}|^{2s}
\mid\mathcal F_t\big)^{1/2}\right\}
\le C\sqrt{\Pred_t}.
\end{aligned}
\]
The last inequality uses the response moment bound and the uniform
feature moments on the fixed compact set $\cN$.
For the last term, independence of the test observation from $\Delta_t$
and the sub-Gaussian feature condition imply
\[
\E|x_0^\top\Delta_t|^{s+1}
\le C\E(\Delta_t^\top\Gamma\Delta_t)^{(s+1)/2}
\le C\Pred_t^{(s+1)/2}.
\]
The second inequality follows from Jensen's inequality when $0<s<1$.
For $s\ge1$, write $\Delta_t=b+\zeta$ with
$\zeta\sim N(0,vI_p)$ conditional on $\mathcal F_t$.
Here $(b,v)=(e_{S_t},\tau_t^2)$ for IND and
$(b,v)=(e_{S_t}+Z_{t-1},\sigma_t^2)$ for RW.
To bound the centered Gaussian part, write $\zeta=\sqrt v\,G$ with
$G\sim N(0,I_p)$. Since
$\|\Gamma^{1/2}G\|_2^2\le\tr(\Gamma)\|G\|_2^2$ and $p$ is fixed,
\[
\E(\|\Gamma^{1/2}\zeta\|_2^{s+1}\mid\mathcal F_t)
\le(v\tr\Gamma)^{(s+1)/2}\E\|G\|_2^{s+1}
\le C(v\tr\Gamma)^{(s+1)/2}.
\]
The inequality $\|a+c\|_2^{s+1}\le2^s(\|a\|_2^{s+1}+\|c\|_2^{s+1})$
then gives
\[
\begin{aligned}
\E(\Delta_t^\top\Gamma\Delta_t)^{(s+1)/2}
&\le C\left\{(b^\top\Gamma b)^{(s+1)/2}
+(v\tr\Gamma)^{(s+1)/2}\right\}\\
&\le C\{b^\top\Gamma b+v\tr(\Gamma)\}^{(s+1)/2}
=C\Pred_t^{(s+1)/2}.
\end{aligned}
\]
Combining the bounds and taking the maximum over $t\le T$ proves
the claimed inequality.

For $M=\Theta(n^\gamma)$, the residual rate yields
\[
\begin{aligned}
r^{\mathrm{cnt}}(M)&=O\{n^{2\gamma-2}\polylog(n)\},\\
r^{\mathrm{cnt}}(M)\sqrt{M\log(M+1)}
&=O\{n^{5\gamma/2-2}\polylog(n)\}.
\end{aligned}
\]
These bounds give the deletion ranges stated after the theorem:
$\gamma<1$ for RW and $\gamma<4/5$ for IND.
\end{proof}

\subsection{Proof of Proposition~\ref{prop:logistic-certificate}}
\begin{proof}
For every $|A|\le M$, we prove
\[
\|\widehat\beta_A-\widehat\beta\|_2\le\rho,
\qquad\text{and}\qquad
\nabla^2L_A(\beta)\succeq\underline\kappa I_p
\quad\text{for all }\|\beta-\widehat\beta\|_2\le\rho.
\]

\textit{Bounds at $\widehat\beta$.}
Write $w(z)=\{1+\exp(-z)\}^{-1}\{1+\exp(z)\}^{-1}$ and
recall the retained Hessian $H_A=H_0-\sum_{i\in A}H_i$ at the original fit.
The ridge term gives $H_A\succeq2\lambda I_p$.  Weyl's inequality also
gives
\[
\lambda_{\min}(H_A)
\ge
\lambda_{\min}(H_0)-\sum_{i\in A}\|H_i\|_2
\ge
\lambda_{\min}(H_0)-\sum_{j=1}^M\|H\|_{(j)}.
\]
The definition $h_i=\sup_{v\ne0}v^\top H_iv/(v^\top H_0v)$ gives
$H_i\preceq h_iH_0$. Thus
\[
\sum_{i\in A}H_i
\preceq\Bigl(\sum_{i\in A}h_i\Bigr)H_0
\preceq\Bigl(\sum_{j=1}^M h_{(j)}\Bigr)H_0.
\]
If $\sum_{j=1}^M h_{(j)}<1$, it follows that
\[
H_A\succeq\Bigl(1-\sum_{j=1}^M h_{(j)}\Bigr)H_0
\succeq\Bigl(1-\sum_{j=1}^M h_{(j)}\Bigr)
\lambda_{\min}(H_0)I_p.
\]
Taking the largest of the available lower bounds proves
$H_A\succeq\underline\alpha I_p$ simultaneously for every
$|A|\le M$.

The full-data first-order condition gives
$\nabla L_A(\widehat\beta) =-\sum_{i\in A}g_i$.
Since $W_M$ is the sum of the $M$ largest norms $\|g_i\|_2$,
the triangle inequality gives
\[
{
\left\|\sum_{i\in A}g_i\right\|_2
\le\sum_{i\in A}\|g_i\|_2
\le W_M.
}
\]

\textit{Localization and curvature on the ball.}
The normalization $\|x_i\|_2\le1$ lets us control the logistic weights
by the parameter displacement. Since
$\left|{(d/dz)}\log w(z)\right| =|1-2\{1+\exp(-z)\}^{-1}| \le1$,
we have $w(z+u)\ge e^{-|u|}w(z)$.
Since $|x_i^\top h|\le\|h\|_2$,
$w(x_i^\top\widehat\beta+x_i^\top h)
\ge e^{-\|h\|_2}w(x_i^\top\widehat\beta)$, and hence
\[
\nabla^2L_A(\widehat\beta+h)
\succeq2\lambda I_p+e^{-\|h\|_2}(H_A-2\lambda I_p).
\]
To apply this inequality along the segment to $\widehat\beta_A$, let
$v_A=\widehat\beta_A-\widehat\beta$ and $r_A=\|v_A\|_2$.  When
$r_A>0$, set $u_A=v_A/r_A$.  The first order condition and the preceding
weight inequality imply
\begin{align*}
-u_A^\top\nabla L_A(\widehat\beta)
&=
\int_0^{r_A}
u_A^\top\nabla^2L_A(\widehat\beta+su_A)u_A\,ds\\
&\ge
2\lambda r_A
+u_A^\top(H_A-2\lambda I_p)u_A
\int_0^{r_A}e^{-s}\,ds\\
&\ge
2\lambda r_A
+(\underline\alpha-2\lambda)(1-e^{-r_A}).
\end{align*}
As a function of the radius, the last expression is zero at zero,
diverges to infinity, and has derivative at least $2\lambda>0$.
Thus the equation defining $\rho$ has a unique nonnegative solution.
Since $-u_A^\top\nabla L_A(\widehat\beta)\le W_M$,
\[
2\lambda r_A+(\underline\alpha-2\lambda)(1-e^{-r_A})
\le W_M
=2\lambda\rho+(\underline\alpha-2\lambda)(1-e^{-\rho}).
\]
Strict monotonicity gives $r_A\le\rho$. The same bound holds when $r_A=0$.
For every
$\|h\|_2\le\rho$, the same weight inequality gives
\[
\nabla^2L_A(\widehat\beta+h)
\succeq
\{2\lambda+e^{-\rho}(\underline\alpha-2\lambda)\}I_p
=
\underline\kappa I_p.
\]

\end{proof}

\subsection{Proof of Proposition~\ref{prop:logistic-constants}}
\begin{proof}
The individual bounds follow from scalar derivatives on the predictor
intervals.
The gradient and Hessian are
$\nabla\ell_i(\beta)=\{\pi(x_i^\top\beta)-y_i\}x_i$ and
$\nabla^2\ell_i(\beta)=w(x_i^\top\beta)x_ix_i^\top$.
Monotonicity of $\pi$ gives the gradient bound, and the Hessian bound
follows directly. Along the segment between $\beta$ and $\beta'$, the
predictor stays in $[a_i,b_i]$, so integration gives
\[
|w(x_i^\top\beta)-w(x_i^\top\beta')|
\le m_i\|x_i\|_2\|\beta-\beta'\|_2.
\]
Multiplication by $\|x_ix_i^\top\|_2=\|x_i\|_2^2$ proves the
individual Hessian Lipschitz bound.

For the retained objective, write
$\nabla^2L_A(\beta)-\nabla^2L_A(\beta')
=\sum_{i\notin A}c_ix_ix_i^\top$, where
$|c_i|\le m_i\|x_i\|_2\|\beta-\beta'\|_2$.
The ridge Hessian $2\lambda I_p$ is constant and cancels.

Let $S_A=\sum_{i\notin A}m_i\|x_i\|_2x_ix_i^\top$. The coefficient
bounds imply
\[
-\|\beta-\beta'\|_2S_A
\preceq\nabla^2L_A(\beta)-\nabla^2L_A(\beta')
\preceq\|\beta-\beta'\|_2S_A.
\]
Since $0\preceq S_A\preceq S_\varnothing$, this proves the uniform
Lipschitz bound $\|S_\varnothing\|_2=C_L^{\mathrm{mat}}$.
Each summand of $S_\varnothing$ has norm $m_i\|x_i\|_2^3$.
Positive semidefinite ordering and the triangle inequality therefore give
$C_\ell\le C_L^{\mathrm{mat}}\le nC_\ell$.
Both terms in the minimum defining $C_L$ are at least
$C_L^{\mathrm{mat}}$, so $C_L$ is also valid and has the stated bounds.
If $C_\ell=0$, then $S_\varnothing=0$ and $C_L=0$, and all the
Hessians are constant on the ball.
Otherwise, $C_L\ge\min\{nC_\ell,\log(n+1)MC_\ell\}$ gives
\[
\frac{MC_\ell}{C_L}
\le\max\left\{\frac Mn,\frac1{\log(n+1)}\right\}\to0
\qquad\text{when }M=o(n).
\]

\end{proof}

\end{document}